\documentclass[11pt]{article}

\usepackage{amssymb,amsmath,amsthm}
\usepackage{natbib}
\usepackage[margin=0.9in]{geometry}
\usepackage{makecell}
\usepackage{enumerate}
\usepackage{bm}
\usepackage{graphicx}
\usepackage{xcolor}
\usepackage[ruled,vlined]{algorithm2e}

\newcommand{\E}{\mathbb{E}}
\newcommand{\G}{\mathbb{G}}
\newcommand{\Var}{\mathrm{Var}}
\newcommand{\Cov}{\mathrm{Cov}}

\newcommand{\R}{\mathbb{R}}

\renewcommand{\P}{\mathbb{P}}
\newcommand{\td}{\tilde}

\newcommand{\lb}{\left(}
\newcommand{\rb}{\right)}
\newcommand{\eps}{\epsilon}

\newcommand{\N}{\mathcal{N}}

\renewcommand{\td}{\tilde}
\usepackage{url}
\newcommand{\tr}{\mathrm{Tr}}

\newcommand{\cls}{{{\scriptscriptstyle\mathrm{XY-only}}}}
\newcommand{\ppi}{{{\scriptscriptstyle\mathrm{PPI}}}}
\newcommand{\ppiplus}{{{\scriptscriptstyle\mathrm{PPI++}}}}
\newcommand{\REPPI}{{{\scriptscriptstyle\mathrm{RePPI}}}}

\newcommand{\cc}{{{\scriptscriptstyle\mathrm{CC}}}}
\newcommand{\all}{{{\scriptscriptstyle\mathrm{All}}}}

\newcommand{\cf}{{{\scriptscriptstyle\mathrm{CrossFit}}}}
\newcommand{\initial}{{{\scriptscriptstyle\mathrm{Init}}}}

\newcommand{\rectifier}{imputed loss~}
\usepackage{hyperref}

\DeclareMathOperator*{\argmin}{argmin}

\newtheorem{theorem}{Theorem}
\newtheorem{lemma}{Lemma}
\newtheorem{proposition}{Proposition}

\newtheorem{remark}{Remark}
\newtheorem{assumption}{Assumption}

\newtheorem{example}{Example}

\usepackage{authblk}

\makeatletter
\newcommand\blfootnote[1]{%
  \begingroup
  \renewcommand\thefootnote{}
\footnotetext{\Hy@raisedlink{\relax}#1}
  \addtocounter{footnote}{-1}
  \endgroup
}
\makeatother
\begin{document}
\title{Predictions as Surrogates: \\Revisiting Surrogate Outcomes in the Age of AI}
\author{Wenlong Ji, Lihua Lei, Tijana Zrnic}

\affil{Stanford University}
\maketitle
\blfootnote{Authors are ordered alphabetically. We thank Bin Nan and Jiwei Zhao for their helpful comments. L.L. and W.J. are grateful for the support of National Science Foundation grant DMS-2338464. T.Z. gratefully acknowledges support from Stanford Data Science.}
\begin{abstract}
We establish a formal connection between the decades-old surrogate outcome model in biostatistics and economics and the emerging field of prediction-powered inference (PPI). 
The connection treats predictions from pre-trained models, prevalent in the age of AI, as cost-effective surrogates for expensive outcomes that are not fully observed.
Building on the surrogate outcomes literature, we develop \emph{recalibrated prediction-powered inference}, a more efficient approach to statistical inference than existing PPI proposals. 
Our method departs from the existing proposals by using flexible machine learning techniques to learn the optimal ``imputed loss'' through a step we call \emph{recalibration}.
Importantly, the method always improves upon the estimator that relies solely on the data with available true outcomes,
even when the optimal imputed loss is estimated imperfectly, and it achieves the smallest asymptotic variance among PPI estimators if the estimate is consistent.
Computationally, our optimization objective is convex whenever the loss function that defines the target parameter is convex. We further analyze the benefits of recalibration, both theoretically and numerically, in several common scenarios where machine learning predictions systematically deviate from the outcome of interest. We demonstrate significant gains in effective sample size over existing PPI proposals via three applications leveraging state-of-the-art AI models.
\end{abstract}

\section{Introduction}

In many scientific applications, the outcome of interest is costly or time-consuming to acquire. Surrogate outcomes, sometimes referred to as auxiliary or proxy variables, are often collected to accelerate data analysis and improve statistical efficiency. Intuitively, surrogates are outcomes that are easy to obtain and highly correlated with the outcome of interest, which are usually not fully observed. Early developments of surrogate outcome models were motivated by clinical trials where measuring the primary endpoint may be ``unduly invasive, uncomfortable or expensive'' \citep{prentice1989surrogate} and sometimes ``confounded by secondary treatments or competing risks'' \citep{wittes1989surrogate}.
For example, \cite{wittes1989surrogate} consider the peak cardiac enzyme level in the bloodstream as a surrogate for the damage to the heart muscle caused by a myocardial infarction. Other examples include the CD4 counts as a surrogate for HIV infection \citep{fleming1994surrogate}, the acute disease status as a surrogate for the chronic disease \citep{pepe1992inference}, and responses shortly after the trial as a surrogate for long-term outcomes in longitudinal studies with drop-outs \citep{post2010analysis}. 

Later on, surrogates have found applications in economics and marketing. For example, \cite{chen2005measurement} present a case study on estimating returns to schooling using the Current Population Survey (CPS). Ideally, the study would be based on employer-reported social security earnings, but these outcomes are only available for the subset of individuals who reported their social security numbers in the survey. As a result, Chen et al. use the household-reported income as a surrogate for the outcome of interest. \cite{athey2019surrogate} study the effect of the Greater Avenues
to Independence (GAIN) job training program on long-term labor market outcomes in California, surrogated by short-term employment and earnings. \cite{kallus2024role} discuss an example of using digital ad clicks, which are available for all users, to surrogate visitations to brick-and-mortar stores, which are only observable for those who agree to share cellphone geolocation data. Recently, tech firms have leveraged online experiments run for two weeks or less as surrogates for the purpose of understanding the long-term effects of a newly launched feature  \citep{athey2019surrogate, gupta2019top, zhang2023evaluating, tran2023inferring}.

In all aforementioned applications, surrogates are domain-specific variables that still need to be collected by the researcher, albeit at a lower cost than the outcome of interest. The incurred cost and time required for surrogate measurement are not always negligible. Note also that the surrogates may themselves be missing due to survey non-response, attrition, or unexpected failure of the measurement system. This would lead to a violation of surrogacy assumptions \citep{prentice1989surrogate,frangakis2002principal,lauritzen2004discussion,chen2007criteria,vanderweele2013surrogate}, even if the outcome of interest is missing at random.

\subsection{Predictions as Surrogates in the Age of AI}
With the rise of machine learning and AI, a nascent literature on prediction-powered inference (PPI) introduces predictions by pre-trained models as another class of surrogates \citep{angelopoulos2023prediction, angelopoulos2023ppi++, wang2020methods, zrnic2024cross, zrnic2024active, fischstratified, xu2025unified}. 
For example, a large language model (LLM) can quickly generate text annotations, such as tones and sentiments, which can serve as surrogates for ``gold-standard'' human annotations \citep{egami2024using,gligoric2024can}. Despite the similarity to the earlier uses of surrogates, prediction surrogates present an important conceptual departure from the surrogate framework. For example, the cost of generating a prediction, even from large-scale commercial AI models, is often orders of magnitude smaller than measuring a domain-specific surrogate variable. In addition, prediction surrogates are never missing: the prediction is fully decided by the covariates or contextual information associated with each unit (up to algorithmic randomness that can be controlled by the researcher; e.g., random seeds).  

This last point raises an important question. Strictly speaking, unlike domain-specific surrogates, predictions carry no additional information since they are solely a function of the observables. What is then the value of using pre-trained models? The answer is: they accelerate learning.
In  settings where covariates are low-dimensional and tabular, the researcher can learn the relationship between the outcome of interest and the covariates precisely simply using the data at hand, without resorting to external prediction models, provided that the data is of a decent size. However, this strategy may be ineffective or even infeasible in modern scientific applications. Indeed, given the vast computational resources and large amounts of high-dimensional and unstructured data that go into their training, pre-trained models can ``pick up'' complex patterns that can significantly accelerate learning the relationship between outcomes and covariates. This makes pre-trained models often more reliable than ``hand-cooked'' domain-specific models.

For example, election forecasting models rely on multimodal data such as polls, prediction markets, past election results, and economic indicators \citep[e.g.][]{hummel2014fundamental, donnini2024election}. The researcher does not have access to all data used to train such models but can query the models for any demographic group of interest through the API. The performance gap between a hand-cooked model and a pre-trained model is even more pronounced in the age of AI, when no single researcher or lab has sufficient resources to train a generative model from scratch that is competitive with commercial models. Furthermore, for unstructured covariates such as images and texts, fitting good models of the outcome of interest is increasingly challenging for most researchers. With massive training data, pre-trained AI models can learn from a wide variety of contexts, capturing nuances, relationships, and patterns that smaller datasets might not cover. This diversity leads to more generalizable representations, and hence more reliable predictions. Empirical evidence supports this, even in areas where general-purpose AI models might not be expected to excel. For example, \cite{vafa2022career} build a foundation model for job sequences, by treating each occupation as a word, using millions of passively-collected resumes from labor market participants. Yet, \cite{du2024labor} later show that the model underperforms off-the-shelf LLMs with careful fine-tuning. 

For all these reasons, pre-trained machine learning and AI models carry substantive value, as the relationship between outcomes and covariates is often difficult to learn solely from the available data. This contrasts with the perspective of semiparametric statistics, which assumes a low-dimensional setting where all conditional distributions can be learned accurately without relying on external data.

\subsection{Recalibrated PPI: A Lesson from Surrogate Outcomes}
{
Although predictions are natural surrogates, the formal connection between PPI and the literature on surrogate outcome models is still an active area of investigation. Existing literature \cite{gronsbell2024another, miao2025assumption, xu2025unified} has explored using methods from surrogate outcome models to improve PPI, however a practical computational strategy to achieve efficiency is not yet established.}
Our first contribution is to explain the two frameworks in a unified language, clarifying the connections and distinctions. Drawing on results from the surrogate outcomes literature \citep{robins1994estimation, chen2000unified,
chen2003information, chen2005measurement, chen2008semiparametric, tang2012efficient}, we identify inefficiencies in existing PPI proposals \citep{angelopoulos2023prediction, angelopoulos2023ppi++, gronsbell2024another}. 

Building off the established connections, we then introduce a more efficient method, termed \emph{recalibrated PPI}, or RePPI for short. The method is applicable for estimands defined through estimating equations. When the estimating equation is given by a generalized linear model (GLM),
the procedure essentially applies PPI with ``recalibrated'' predictions, obtained by approximating the conditional expectation of the true outcome given the predicted outcome and the covariates that define the estimand. For general estimating equations, the method applies the same recalibration to the ``imputed loss''~\citep{angelopoulos2023prediction}.
The recalibration can be regarded as model fine-tuning. Recalibrated PPI is always more efficient than inference that does not utilize predictions as surrogates, even when the conditional expectation of the true outcome is estimated inconsistently. If the conditional expectation is estimated consistently, the method achieves the minimum asymptotic variance among all ``prediction-powered estimators,'' defined in Section \ref{sec: surrogate and PPI}. While the core idea is inspired by existing literature \citep{robins1994estimation, chen2000unified}, we differ from the previous proposals in two key aspects: (1) our method applies cross-fitting to enable flexible recalibration of the \rectifier via machine learning algorithms \citep{chernozhukov2018double}; (2) we use a linearized empirical risk minimization approach to simplify the computation. { We elaborate on our methodological novelty in Section \ref{sec:method_contribution}.}

Recognizing that even the best general-purpose model can generate biased predictions, another main contribution of this paper is to study the role of recalibration in the context of common use-cases for pre-trained models. We provide theoretical insights into the benefits of recalibration in three common scenarios:
\begin{itemize}
    \item \textbf{Modality mismatch.} The predictions may be based on only a subset of the available covariates. For example, for each individual in the dataset we may have both demographic information and a medical scan (e.g., MRI). We may use a computer vision model to predict the individual's diagnosis based on the medical scan, but there is no way to input the demographic information into the prediction. In such cases, recalibration can improve efficiency by fine-tuning the generic predictions for each demographic group.
    
    \item \textbf{Distribution shift.} The prediction model may have been trained on a general-purpose dataset following a different distribution than the population under study. For example, most LLMs are trained on texts from the whole internet and may not be well-calibrated to specific subgroups. Recalibration improves efficiency by adjusting the predictions to better reflect the outcome in the population of interest.
    \item \textbf{Discrete predictions.} Discrete predictions, commonly used for binary or categorical outcomes, may not be sufficiently informative. Moreover, while there are sometimes ways of obtaining a probabilistic output (e.g., by prompting an LLM to produce probabilistic predictions, or by looking at its token probabilities), these are widely acknowledged to be miscalibrated and unreliable \citep{xiong2023can,wei2024measuring}. Recalibration improves efficiency by turning the discrete predictions into calibrated numerical ones. We note that \cite{hofer2024bayesian} similarly acknowledge a need for recalibrating discrete predictions in the PPI context.
\end{itemize}

\section{A Review of the Surrogate Outcome Model and PPI}

\subsection{Surrogate Outcome Model}

The surrogate outcome model addresses the standard problem of inferring the relationship between an outcome variable $Y$ and covariates $X$. The parameter of interest $\theta^\star\in \R^d$ is defined as the solution to an estimating equation: 
\begin{equation}
\label{eqn:estimand}
\E[U_\theta(X, Y)] = 0,
\end{equation}
for some function $U_\theta$.  For example, $U_\theta(X, Y)$ can be chosen to be a score function, i.e., the gradient of a log-likelihood.

When data collection is costly or time-consuming, it may be impossible to measure the true outcome for all subjects. In the hope of increasing the sample size, researchers often collect a surrogate outcome $\hat Y$ that is correlated with the outcome of interest $Y$ and can be measured at a much lower cost.  We assume that we observe \(n\) labeled observations, for which \(Y\) is available, and \(N\) unlabeled observations, for which only \((X,\hat Y)\)
 is available. Equivalently, if \(D\in\{0,1\}\) denotes the indicator of having a label, \(D=1\) meaning that \(Y\) is observed, then we have an incomplete dataset $\{(Y_i, \hat Y_i, X_i,  D_i)\}_{i=1}^{n+N}$.

Unless $Y$ and $\hat Y$ are perfectly correlated, replacing $Y$ by $\hat Y$ does not yield a valid estimating equation, i.e., $\E[U_{\theta^\star}(X, \hat{Y})] \neq 0$.
\cite{pepe1992inference} provides an early semiparametric solution to correct for this bias; here we take the perspective of \cite{robins1994estimation}, who consider a similar semiparametric model. We note that, although they focus on missing covariates rather than missing outcomes, their theory directly applies to the latter. They assume that the missingness occurs at random, i.e., $\P(D= 1|Y,\hat Y, X) = p$ for some $p\in (0, 1)$. Their proposed estimator $\hat\theta$ is defined as the solution to the modified estimating equation
\begin{equation}
\label{eqn: surrogate estimator}
\sum_{i=1}^{n+N} \left(\frac{D_i}{p}U_\theta(X_i, Y_i) - \frac{D_i-p}{p}s_\theta( X_i, \hat{Y}_i)\right) = 0,    
\end{equation}
where $s_\theta$ is a user-specified function. {Throughout the paper, we will call $s_\theta$ the \emph{imputed score function} or simply \emph{score function}, as it usually appears in the form of the gradient of a likelihood function in most applications discussed in the paper, but it can be more general in principle}. They prove that the optimal choice is given by 
\begin{equation}
\label{eqn:optimal score}
    s^\star_\theta(X, \hat{Y}) = \E[U_\theta(X, Y)\mid X, \hat{Y}].
\end{equation}
The resulting estimator $\hat\theta$ with this choice of $s_\theta$ is \emph{semiparametrically efficient}. The authors suggest estimating the optimal $s^\star_\theta$ parametrically.

{

Equation \eqref{eqn: surrogate estimator} is an augmented inverse probability weighted (AIPW) estimating equation for a
missing-outcome problem \citep{robins1994estimation, tsiatis2006semiparametric}. The first term, $\frac{D_i}{p} U_\theta(X_i,Y_i),$ is the inverse-probability weighted complete-data estimating function, while the second term, $-\frac{D_i-p}{p}s_\theta(X_i,\hat Y_i),$
is a mean-zero augmentation based on the surrogate information \((X_i,\hat Y_i)\). Equivalently,
the estimating equation can be written as
\begin{equation}
    \label{eqn: AIPW}
    \sum_{i=1}^{n+N}
    \left[
        \frac{D_i}{p}\{U_\theta(X_i,Y_i)-s_\theta(X_i,\hat Y_i)\}
        +
        s_\theta(X_i,\hat Y_i)
    \right]
    =0 .
\end{equation}
This is the usual AIPW decomposition in causal inference: the labeled observations ($D_i=1$) estimate the residual
\(U_\theta(X,Y)-s_\theta(X,\hat Y)\), while all observations contribute the surrogate-based
augmentation \(s_\theta(X,\hat Y)\). In our setting the observation probability \(p\) is known
and constant, so the main issue is not the estimation of $p$ but the choice of the augmentation
\(s_\theta\).

Recognizing the practical challenge of choosing a well-specified parametric model for fitting $s_\theta^\star$, \cite{chen2000unified} improve the surrogate outcome estimator with a safeguard against poorly specified $s_\theta$. They condition on the missingness, i.e., they assume access to a dataset with observed outcomes $\{(Y_i, \hat Y_i, X_i)\}_{i=1}^{n}$ and a dataset with missing outcomes $\{(\hat Y_i, X_i)\}_{i=n+1}^{n+N}$ independently. For $U_\theta (X,Y)\in \R^d$, they propose the following ``safe'' estimator, which we will refer as the ``CC-estimator'' for short:
\begin{equation}
\label{eqn: CC estimator}
    \hat{\theta}^\cc \triangleq \hat{\theta}^\cls - \hat{M}_s(\hat{\theta}_s^\cls - \hat{\theta}_s^\all).
\end{equation}
Here, $\hat{\theta}^\cls$, $\hat{\theta}_s^\cls$, and $\hat{\theta}_s^\all$ respectively solve the following estimating equations:
\[
\frac{1}{n}\sum_{i=1}^{n}U_\theta(X_i, {Y}_i) = 0, \qquad \frac{1}{n}\sum_{i=1}^{n}s_\theta(X_i, \hat{Y}_i) = 0, \qquad \frac{1}{n+N}\sum_{i=1}^{n+N}s_\theta(X_i, \hat{Y}_i) = 0,
\]
and the matrix $\hat{M}_s$ is an empirical estimate of 
$$\E[\nabla U_{\hat{\theta}^\cls}]^{-1}\Cov(U_{\hat{\theta}^\cls},  s_{\hat{\theta}^\cls })\Cov( s_{\hat{\theta}^\cls})^{-1}\E[\nabla s_{\hat{\theta}^\cls}]^{-1}$$

based on the labeled data. The matrix $\hat{M}_s$ can be viewed as a tuning parameter that decides how much to rely on the predictions $\hat Y_i$; if it equals zero, then $\hat{\theta}^\cc$ reduces to the ``XY-only'' estimator $\hat{\theta}^\cls$ that ignores the predictions. The specific choice of $\hat{M}_s$ above optimizes estimator efficiency, and thus the CC-estimator is never worse by incorporating predictions, no matter the choice of $s_\theta$. If, however, we happen to be able to learn and use the optimal score \eqref{eqn:optimal score}, then the estimator is semiparametrically efficient.

Inspired by the above works, we develop a method for the context of prediction-powered inference that is both \emph{efficient} when the optimal score function is estimated consistently, and \emph{safe} against bad predictions regardless of how well the score function is chosen.

}
\subsection{Connection to Prediction-Powered Inference}
\label{sec: surrogate and PPI}
Prediction-powered inference (PPI) was proposed by \cite{angelopoulos2023prediction} as a way to incorporate machine learning predictions from any black-box model into statistical inference.
In their setting, the researcher has access to two datasets, one labeled and one unlabeled. For notational convenience, we split the covariates into two sets: by $X$ we denote the covariates that define the inference problem, as in \eqref{eqn:estimand}, and by $W$ we denote possibly unstructured, high-dimensional additional covariates, such as text or images, that can be used for prediction. Thus, the researcher has access to $n$ i.i.d. labeled data points, $\{(Y_i,  X_i, W_i)\}_{i=1}^n$, and $N$ i.i.d. unlabeled data points, $\{(X_i, W_i)\}_{i=n+1}^{n+N}$. The working assumption in PPI is that the distribution of $(X_i, W_i)$ is the same in the two datasets.
The target parameter is defined as 
\begin{equation}\label{eq:thetastar}
\theta^\star = \argmin_{\theta\in \Theta}\E[\ell_{\theta}(X, Y)],
\end{equation}
for a loss $\ell_\theta$ that is convex in $\theta \in \R^d$. This target is equivalent to the estimating-equation target \eqref{eqn:estimand} if we take $U_\theta = \nabla \ell_{\theta}$.

The researcher additionally has access to a black-box machine learning model $f$, which outputs predictions $\hat{Y} = f(X, W)$. The model is not required to take both $X$ and $W$ as inputs; often it only takes $W$, e.g. in the case of text annotation or image classification.
With the assistance of the machine learning model, one can expand the labeled and the unlabeled datasets with predictions $\hat{Y}$. Clearly, this recovers the same problem structure as the surrogate outcome model---the researcher has access to an incomplete dataset $\{(Y_i, \hat{Y}_i, X_i, W_i, D_i)\}_{i=1}^{n+N}$, where $Y_i$ is observed if and only if $D_i = 1$, and $D_i=1$ for $i\in\{1,\dots,n\}$ and $D_i=0$ for $i\in\{n+1,\dots,n+N\}$. The resemblance suggests that predictions can be thought of as surrogates. 
Throughout the paper, we will study a setting where $n/N \rightarrow r$ with a limit $r\in (0,1)$ as $n\rightarrow\infty$ to develop asymptotic theory. 

The estimators from most of the recent PPI literature~\citep{angelopoulos2023prediction,angelopoulos2023ppi++,miao2025assumption,gan2024prediction,gronsbell2024another, xu2025unified} can be written in a unified form. We call the following unifying formula the \emph{PPI estimator}:
\begin{equation}
\label{eqn: PPI}
    \begin{aligned}
\hat{\theta}_g^\ppi = \argmin_{\theta} \frac{1}{n}\sum_{i=1}^{n}\ell_{\theta}(X_i, Y_i) -\lb\frac{1}{n}\sum_{i=1}^{n}g_{\theta}(X_i, \hat Y_i) - \frac{1}{N}\sum_{i=n+1}^{n+N}g_{\theta}(X_i, \hat Y_i)\rb.  
\end{aligned}
\end{equation}
Here, $g_\theta$ is a method-specific function that we call the \emph{imputed loss}; the different estimators from the literature differ in their choice of $g_\theta$. If $\ell_\theta$ is convex, \eqref{eqn: PPI} is essentially equivalent to \eqref{eqn: surrogate estimator} with $s_\theta(X, \hat{Y}) = \nabla g_\theta(X, \hat{Y})$.
Similar to \eqref{eqn: surrogate estimator}, the PPI estimator can be viewed as optimizing the loss function of the AIPW form with known propensity score $p=\frac{n}{n+N}$.

We summarize below the existing choices of $g_\theta$ in the PPI literature and their counterparts from surrogate outcomes and related literature. 
\begin{itemize}
\item The XY-only estimator $\hat{\theta}^\cls$ is a special case of \eqref{eqn: PPI} with $g_\theta(X, \hat{Y}) = 0$. Recall, this is the estimator that ignores the predictions and simply uses the subset of the data where $Y$ is observed.
\item The standard PPI estimator \citep{angelopoulos2023prediction, angelopoulos2023ppi++}, which we denote by $\hat{\theta}^\ppi$, chooses $g_\theta = \ell_\theta$. 
This estimator is very similar to the estimator from~ \cite{tang2012efficient} with a single imputation and a constant propensity score $p(X_i)= n/(n+N)$.
\item Followup works choose $g_\theta$ such that $\nabla g_{\theta} = \hat{M} \nabla \ell_\theta$ for some matrix $\hat{M}$ that minimizes the asymptotic variance of the resulting estimator within a class $\mathcal{M}$ . Specifically, $\mathcal{M}$ is chosen as the set of scaled identity matrices in  \cite{angelopoulos2023ppi++}, the set of diagonal matrices in \cite{miao2025assumption}, and the set of all matrices in \cite{xu2025unified}. For one-dimensional targets $\theta^\star$, these proposals are equivalent. 
\item \cite{gan2024prediction} consider a more general class with $\nabla g_\theta = \gamma \hat{M}(\theta)  f_\theta$ where $\gamma$ is a scalar, $f_\theta$ is a given function that takes values in $\R^q$, and  $\hat{M}(\theta)$ is the $d\times q$ matrix that minimizes the asymptotic variance of the resulting estimator. In their numerical studies, Gan et al. choose $f_\theta = \nabla \ell_\theta$, as in PPI++. Their method generalizes the method of~\cite{song2024general}, who leverage the unlabeled data through simple summary statistics such as polynomials, rather than pre-trained models.
\end{itemize}

All aforementioned estimators but $\hat{\theta}^\ppi$ are guaranteed to be more efficient than the XY-only estimator $\hat{\theta}^\cls$. However, 
none of them achieves the lowest asymptotic variance among the class of PPI estimators defined by \eqref{eqn: PPI}, except in a few special cases.

Returning to \cite{robins1994estimation}, this earlier work implies that the optimal imputed loss is given~by 
\begin{equation}\label{eq:optimal_gtheta}
\nabla g_{\theta}^\star(X, \hat{Y}) = \frac{N}{n+N}\E[\nabla \ell_\theta(X, Y)\mid X, \hat{Y}].
\end{equation}
In general, this optimal choice is more complex than a linear transformation applied to $\nabla \ell_\theta$. 
Recognizing this difficulty, in this paper, we roughly choose $g_\theta$ to be of the form 
\begin{equation}\label{eq:our_gtheta}
\nabla g_\theta(X, \hat{Y}) = \hat{M}  \nabla \hat g^\star_\theta(X,\hat Y),
\end{equation}
where $\nabla \hat g^\star_\theta$ is an estimate of the optimal $\nabla g_\theta^\star$ by \cite{robins1994estimation}, and $\hat{M}$ is a $d\times d$ matrix that serves a similar purpose as in \cite{angelopoulos2023ppi++}, \cite{miao2025assumption}, \cite{xu2025unified}, and, as discussed earlier, \cite{chen2000unified}: to protect against poor estimates $\nabla \hat g^\star_\theta$, ensuring no loss in efficiency compared to the XY-only baseline. We do not restrict $\hat{M}$ to be within a simple class.\footnote{Although the right-hand side of \eqref{eq:our_gtheta} is not necessarily the gradient of some function, we directly work with $\nabla g_\theta$ in our theory, without having to invoke $g_\theta$, so this is not an issue.}

Below, we provide a concrete example to show the explicit form of the PPI estimator. 
{

\begin{example}[Mean estimation]
\label{ex:mean}
Consider the mean estimation problem where the goal is to estimate $\theta^\star = \E[Y]$. Equivalently, we can write this estimand as an instantiation of \eqref{eq:thetastar} with $\ell_\theta(X,Y) = (Y-\theta)^2$.
Let \(m^\star(X,\hat Y)=\E[Y\mid X,\hat Y]\). Since
\(\nabla \ell_\theta(X,Y)=2\theta-2Y\), the optimal imputed loss \eqref{eq:optimal_gtheta} can be chosen as
$
    g_\theta^\star(X,\hat Y)
    =\frac{N}{n+N}(m^\star(X,\hat Y)-\theta)^2.
$
The PPI objective \eqref{eqn: PPI} then
becomes
\[
    \frac{1}{n}\sum_{i=1}^n (Y_i-\theta)^2
    -
    \frac{N}{n+N}
    \left\{
        \frac{1}{n}\sum_{i=1}^n (m^\star(X_i,\hat Y_i)-\theta)^2
        -
        \frac{1}{N}\sum_{i=n+1}^{n+N} (m^\star(X_i,\hat Y_i)-\theta)^2
    \right\}.
\]
Therefore, the resulting estimator has a closed form
$
    \hat\theta_{g^\star}^\ppi
    =
    \frac{1}{n+N}\sum_{i=1}^{n+N} m^\star(X_i,\hat Y_i)
    +
    \frac{1}{n}\sum_{i=1}^n
    \{Y_i-m^\star(X_i,\hat Y_i)\}.
$
\end{example}
}
\begin{remark}\label{rmk: W}
We take a moment to discuss the relevance of the additional covariates $W$, which do not appear in the definition of the target \eqref{eq:thetastar}. Technically, we could allow $g_{\theta}$ to depend on $W$. \cite{robins1994estimation} again implies that the optimal $\nabla g_\theta(X, W, \hat{Y})$ would be given by $\E[\nabla \ell_\theta(X, Y)\mid X, W, \hat{Y}]$. However, this would make $\hat{Y}$ redundant since the prediction is simply a function of $(X, W)$, so $\E[\nabla \ell_\theta(X, Y)\mid X, W, \hat{Y}] = \E[\nabla \ell_\theta(X, Y)\mid X, W]$. An implicit assumption in existing PPI works, which we adopt here as well, is that $W$ is unstructured, high-dimensional data that is difficult to access other than through $\hat Y$. For example, $W$ can be text or image data, which is difficult for modeling and learning from scratch since it often requires a large sample size and high computational budget to train state-of-the-art deep learning models. In this scenario, directly estimating $\E[\nabla \ell_\theta(X, Y)\mid X, W]$ is extremely challenging, if not impossible. In contrast, pretrained foundation models, which acquire general knowledge from massive data sources, can process this task-specific information efficiently and produce high-quality predictions $\hat Y$ with low cost -- they essentially provide an informative summary or featurization of $W$. For this reason, we focus on the class of estimators that depend on $W$ only through $\hat{Y}$. We leave other, more flexible uses of $W$ for future research. 
\end{remark}
{

\begin{remark}
In addition to the ``loss-level'' correction \eqref{eqn: PPI}, a few other PPI works follow the formulation of \cite{chen2000unified} and explore an ``estimator-level'' correction: 
\begin{equation}
\label{eqn: PPI estimator correction}
    \begin{aligned}
\hat{\theta}_g^\ppi = \hat{\theta}^\textrm{lab} - M(\hat{\theta}_g^\textrm{lab} - \hat{\theta}_g^\textrm{unlab}),
\end{aligned}
\end{equation}
where $M\in \R^{d \times d}$ is a tuning matrix, and
\begin{equation*}
    \hat{\theta}^\textrm{lab} =  \argmin_{\theta} \frac{1}{n}\sum_{i=1}^{n}\ell_{\theta}(X_i, Y_i), \ \hat{\theta}_g^\textrm{lab} = \argmin_{\theta} \frac{1}{n}\sum_{i=1}^{n}g_{\theta}(X_i, \hat Y_i), \ 
    \hat{\theta}_g^\textrm{unlab} = \argmin_{\theta} \frac{1}{N}\sum_{i=n+1}^{n+N}g_{\theta}(X_i, \hat Y_i). 
\end{equation*}

\cite{miao2024valid, zrnic2024note} choose the original loss function $g=\ell$ and a scalar matrix $M =rI$; \cite{miao2024task, kluger2025prediction} choose the original loss function $g=\ell$ and $M$ as an estimate of $(\Var(\hat{\theta}_g^\textrm{lab})+\Var(\hat{\theta}_g^\textrm{unlab}))^{-1}\Cov(\hat{\theta}_g^\textrm{lab},\hat{\theta}^\textrm{lab})$; \cite{gronsbell2024another} further replace $\hat{\theta}_g^\textrm{unlab}$ with $\hat{\theta}_g^\textrm{all}$, which uses all labeled data and unlabeled data for the estimaton. In particular, \cite{gronsbell2024another} point out the efficient choice of $g$, but they still suggest using $g=\ell$ and optimizing $M$ among all matrices, admitting the practical challenge of finding the optimal $g$. In terms of asymptotic theory, the estimator-level correction \eqref{eqn: PPI estimator correction} is similar to the loss-level correction \eqref{eqn: PPI}, and we will focus on the formulation \eqref{eqn: PPI}.
\end{remark}
}
\subsection{Methodological Contribution}
\label{sec:method_contribution}
{

In this section, we highlight the methodological contribution of our method compared to the existing literatures on surrogate outcomes and prediction-powered inference. 
\begin{enumerate}
    \item \textbf{Cross-fitting}: Our RePPI estimator uses cross-fitting to isolate the score learning and inference. Classical surrogate outcome estimators \cite{chen2000unified, chen2008improving} typically use low-dimensional parametric models for modeling $\E[U_\theta(X, Y)\mid X, \hat{Y}]$, which can achieve efficiency when the model is correctly specified, even if we use the same sample to learn the model and do inference. However, if one replaces the working model with a highly adaptive learner, such as random forests, boosted trees, or neural networks, same-sample fitting can introduce first-order overfitting bias. Cross-fitting makes the evaluation of recalibrated scores for all labeled observations out-of-sample, so that the final estimator can achieve efficiency under an $L^2$-convergence condition for the learned score. Therefore, cross-fitting enables us to use highly flexible machine learning models to handle complex datasets, where the optimal score cannot be captured within a parametric model, to still achieve efficiency. We elaborate on the role of cross-fitting and compare it with fitting the score on the entire labeled sample in Section \ref{sec: cross-fitting}.
    \item \textbf{Linearized implementation}: Our RePPI estimator can be implemented by optimizing the PPI objective \eqref{eqn: PPI} with a \emph{linear} choice of $g_\theta$. This ensures the final estimation procedure is convex as long as the original loss function is convex, regardless of the property of the fitted score. This resolves the potential non-convexity issue of previous PPI estimators \citep{angelopoulos2023prediction,angelopoulos2023ppi++}. Moreover, our $g_\theta$ only estimates the optimal score function at $\theta^\star$, which is sufficient for asymptotic efficiency. In contrast, \cite{chen2000unified, gronsbell2024another} require estimating the optimal score function
    for all values $\theta$, as well as strong assumptions on the model class to enable efficient minimization with the fitted score function.
    \cite{xu2025unified} use a basis-expansion approach to learn the optimal score, and explicitly require the eigenvalues of the empirical covariance matrix for the basis function to be lower bounded. In Appendix \ref{sec:other-comparison}, we show that the basis-expansion approach can be numerically unstable and suffer from approximation error.
\end{enumerate}
To summarize, although semiparametric efficiency is widely established in the literature, our work provides a practical methodology to achieve this efficiency without requiring strong parametric assumptions on the optimal score or fitting the optimal score for every value of $\theta$. This is particularly important in applications arising from complicated datasets with prediction surrogates generated from AI and machine learning models, where classical parametric models fail to capture the relationship between true labels, covariates, and surrogates. We provide a concrete computational guide (Algorithm \ref{alg:crossfit}) to enable the use of any flexible learner, together with theoretical guarantees on safety against poor predictions and efficiency under consistent estimation of the optimal score (Theorem \ref{thm:main}), as outlined above. 

}

\section{Our Method: Recalibrated PPI}
\label{sec: REPPI}

\subsection{Optimal Imputed Loss}
\label{sec: optimal rectifier}
Based on the connections established in Section \ref{sec: surrogate and PPI}, we know that the optimal imputed loss, i.e., the one that yields the smallest asymptotic variance of $\hat{\theta}_g^\ppi$, must be given by \eqref{eq:optimal_gtheta}. For the sake of completeness, we state this result formally below and present a self-contained proof in Section \ref{sec: proof REPPI} of the Supplementary Material. Up to minor technical differences in the setup, the proof is almost the same as in \cite{robins1994estimation}.

\begin{theorem}
\label{thm: efficient PPI}
    Let the target $\theta^\star$ defined in \eqref{eq:thetastar} be unique. Assume that $n/N \rightarrow r$ and the objective function \eqref{eqn: PPI} is convex. Let $H_{\theta^\star} = \E[\nabla^2\ell_{\theta^\star}(X, Y)]$. Under regularity conditions (Assumption \ref{asm: regularity} in Section~\ref{sec: proof REPPI} of the Supplementary Material), $\sqrt{n}(\hat{\theta}_g^\ppi - \theta^{\star})\xrightarrow{d} \N(0, \Sigma^\ppi_g)$, where 
    $$
    \Sigma_g^\ppi  = H_{\theta^\star}^{-1}\lb r \Cov(\nabla g_{\theta^\star}(X, \hat Y)) + \Cov(\nabla \ell_{\theta^\star}(X, Y) - \nabla g_{\theta^\star}(X, \hat Y))\rb H_{\theta^\star}^{-1}.
    $$
    Furthermore, if $g_\theta$ satisfies \eqref{eq:optimal_gtheta} at $\theta^\star$, i.e., 
\begin{equation}\label{eq:optimal_gtheta_thetastar}
    \nabla g_{\theta^\star}(X, \hat{Y})= \frac{1}{1+r} s^\star(X, \hat{Y}), \quad \text{where }s^\star(X, \hat{Y}) = \E[\nabla \ell_{\theta^\star}(X, Y)\mid X, \hat{Y}],
    \end{equation}
    then $\Sigma_g^\ppi = H_{\theta^\star}^{-1}\lb \Cov(\nabla \ell_{\theta^\star}(X, Y)) - \frac{1}{1+r}\Cov(\E[\nabla \ell_{\theta^\star}(X, Y)|X, \hat Y])\rb H_{\theta^\star}^{-1}$, and $\Sigma_g^{\ppi}\preceq \Sigma_{g'}^\ppi$ for any~$g'_\theta$.
\end{theorem}

Here, the asymptotic variance of $\hat{\theta}_g^\ppi$ with optimal $g_\theta$ defined by \eqref{eq:optimal_gtheta_thetastar} matches the efficiency bound of \cite{xu2025unified} in the prediction-restricted PPI model.
{ In this model, the black-box prediction rule \(f\) is treated as fixed, and the auxiliary information available for unlabeled observations is $(X,\hat Y)$, where $\hat Y=f(X,W)$; we cannot use the full high-dimensional or unstructured covariates \(W\) directly. Equivalently, one observes \(n\) labeled points \((Y_i,\hat Y_i, X_i)\) and \(N\) unlabeled points $(X_i, \hat Y_i)$, with both samples drawn from the same marginal distribution. Under this model, the resulting efficiency bound coincides with the corresponding semi-supervised bound of \citet{xu2025unified} when their auxiliary covariate is taken to be \((X,\hat Y)\). The efficient augmentation in the full nonparametric model in which \(W\) is directly available and tractable for inference would condition on \((X,W)\) and result in a smaller bound.}

Theorem \ref{thm: efficient PPI} implies that the asymptotic variance of $\hat{\theta}_g^\ppi$ depends on the \rectifier only through its gradient at $\theta^\star$, suggesting that the optimal \rectifier is non-unique. Although any $g_\theta$ satisfying \eqref{eq:optimal_gtheta_thetastar} is statistically optimal, the computational efficiency of different choices may vary. For example, the \rectifier defined in \eqref{eq:optimal_gtheta} might be a complicated function of $\theta$ that results in a non-convex objective function in \eqref{eqn: PPI}. Moreover, computing $\hat\theta^\ppi_g$ would require estimating the conditional expectation $\E[\nabla \ell_\theta(X, Y)\mid X, \hat Y]$ for \emph{every} value of $\theta$, which is challenging for general losses.

A more convenient choice of the imputed loss is the linear function
\begin{equation}\label{eq:linear_rectifier}
g_\theta(X, \hat{Y}) = \frac{1}{1+r}\theta^\top s^\star(X, \hat{Y}).
\end{equation}
With this choice, the objective function in \eqref{eqn: PPI} simply adds a linear shift to the standard empirical loss; therefore, it remains convex as long as $\ell_\theta$ is convex. 

\begin{example}[Generalized linear models]
\label{ex:GLM} Suppose $\ell_{\theta}$ is given by a generalized linear model (GLM): $\nabla \ell_\theta(X, Y) = X (\mu(X^\top \theta) - Y)$, for some $\mu$. Then, \eqref{eq:optimal_gtheta_thetastar} gives
\[s^\star(X, \hat{Y}) = X \lb \mu(X^\top \theta^\star) - \E[Y\mid X, \hat{Y}]\rb = \nabla \ell_{\theta^\star}(X, \E[Y\mid X, \hat{Y}]).\]
When predictions are calibrated in the sense that $\hat{Y} = \E[Y \mid X]$, then $\E[Y\mid X, \hat{Y}] = \hat{Y}$, and standard PPI is optimal. In general, $\E[Y\mid X, \hat{Y}]$ can be viewed as a recalibrated prediction.
\end{example}

{

\begin{example}[Smooth Huber-type robust regression]
Consider a smooth Huber-type robust regression loss \citep{huber1992robust,charbonnier1997deterministic},
$
    \ell_\theta(X,Y)
    =
    \delta^2
    \left[
        \left\{1+\left(\frac{Y-X^\top\theta}{\delta}\right)^2\right\}^{1/2}
        -1
    \right].
$
Then
$
    \nabla \ell_\theta(X,Y)
    =
    -X\,
    \frac{Y-X^\top\theta}
    {
        \left\{1+\left(\frac{Y-X^\top\theta}{\delta}\right)^2\right\}^{1/2}
    },
$
and \eqref{eq:optimal_gtheta_thetastar} gives
$
    s^\star(X,\hat Y)
    =
    -X\,
    \E\left[
        \left.
        \frac{Y-X^\top\theta^\star}
        {
            \left\{1+\left(\frac{Y-X^\top\theta^\star}{\delta}\right)^2\right\}^{1/2}
        }
        \right|X,\hat Y
    \right].
$
Unlike in the case of GLMs, this is generally not equal to
\(\nabla\ell_{\theta^\star}(X,\E[Y\mid X,\hat Y])\), so recalibration must target the score
rather than only the outcome.
\end{example}}

\subsection{Recalibrated PPI: An Efficient Implementation}
\label{sec: computation}
Despite the simple form of \eqref{eq:linear_rectifier}, $g_\theta$ is challenging to approximate because (a) $\theta^\star$ is unknown, and (b) $s^\star$ may be complex. Our main idea is to replace $\theta^\star$ by an initial estimator $\hat{\theta}_0$, such as the XY-only estimator $\hat{\theta}^\cls$, and $s^\star$ by $\hat s$, an estimate of $\E[\nabla \ell_{\hat{\theta}_0}(X, Y)\mid X, \hat{Y}]$ produced via a flexible machine learning method. For example, we can apply random forests or gradient boosting, treating $(\nabla \ell_{\hat{\theta}_0}(X_i, Y_i))_{i=1}^{n}$ as the outcomes and $(X_i, \hat{Y}_i)_{i=1}^{n}$ as the covariates. 

When $\hat s$ consistently estimates $s^\star$, we can show that the resulting estimator $\hat{\theta}_g^\ppi$ is asymptotically equivalent to the estimator given by the optimal \rectifier \eqref{eq:linear_rectifier}. However, if the estimate is asymptotically biased---for example, due to the computational complexity of approximating $s^\star$---the asymptotic variance of the resulting estimator would be inflated and could even be worse than the variance of the XY-only estimator. To guarantee an efficiency gain over the XY-only estimator, we apply the idea of optimal control variates \citep{chen2000unified, gan2024prediction, gronsbell2024another} or power tuning \citep{angelopoulos2023ppi++, miao2025assumption}. Specifically, we pre-multiply $\hat{s}(X, \hat{Y})$ with a matrix $\hat{M}$ that captures the correlation between the true gradient $\nabla \ell_{\theta^\star}(X, Y)$ and the estimated score $\hat s(X,\hat Y)$, and we minimize the asymptotic variance of the resulting PPI estimator over $\hat M$. Again, to operationalize the approach, we replace $\theta^\star$ by an initial estimator $\hat{\theta}_0$ and set 
\begin{equation}\label{eq:Mhat}
\hat M = \widehat \Cov(\nabla \ell_{\hat \theta_0}(X,Y), \hat s(X,\hat Y)) \widehat \Cov(\hat s(X,\hat Y))^{-1},
\end{equation}
where $\widehat \Cov$ denotes the sample covariance matrix. This choice of $\hat M$, as we shall soon see, will guarantee an improvement upon the XY-only estimator.

To mitigate the dependencies in these nested estimation steps, we apply a three-fold cross-fitting procedure. The complete recalibrated PPI procedure is given in Algorithm \ref{alg:crossfit}. We justify the procedure theoretically in the following theorem. 

\begin{algorithm}[t]
\vspace{0.5em}
\textbf{Step 1:} Randomly split the labeled dataset into three folds $\mathcal{D}_1$, $\mathcal{D}_2$, and $\mathcal{D}_3$ evenly. 

\textbf{Step 2:} On $\mathcal{D}_3$, compute the initial estimator $\hat \theta_0^1 = \argmin_{\theta} \frac{1}{|\mathcal{D}_3|}\sum_{i\in \mathcal{D}_3}\ell_{\theta}(X_i, Y_i)$

\textbf{Step 3:} On $\mathcal{D}_2$, use flexible method to estimate $\hat{s}(X, \hat{Y})\approx \mathbb{E}[\nabla \ell_{\hat{\theta}_0^1}(X, Y) \mid X, \hat{Y}]$

\textbf{Step 4:} On $\mathcal{D}_1$, compute $\hat M$ as in \eqref{eq:Mhat}

\textbf{Step 5:} On $\mathcal{D}_1$ and the unlabeled data, compute $\hat{\theta}_{\hat g}^\ppi$ with $\hat g_\theta(X, \hat{Y}) = \frac{1}{1+n/N} \theta^\top \hat{M}\hat{s}(X, \hat{Y})$;
denote the obtained estimate by $\hat{\theta}^1$

\textbf{Step 6:} Repeat Steps 2--5 with fold rotations: $(\mathcal{D}_1, \mathcal{D}_3, \mathcal{D}_2)$ and $(\mathcal{D}_2, \mathcal{D}_1, \mathcal{D}_3)$; obtain estimates $\hat{\theta}^2$ and $\hat{\theta}^3$

\textbf{Step 7:} Compute the final estimate: $
\hat \theta^\cf = \frac{|\mathcal{D}_1|}{n} \hat{\theta}^1 + \frac{|\mathcal{D}_2|}{n} \hat{\theta}^2 + \frac{|\mathcal{D}_{3}|}{n} \hat{\theta}^3$
\vspace{0.5em}
\caption{Recalibrated Prediction-Powered Inference (RePPI)}
\label{alg:crossfit}
\end{algorithm}

\begin{theorem}
\label{thm:main}
Let the target $\theta^\star$ defined in \eqref{eq:thetastar} be unique and assume $n/N \rightarrow r$. Let $H_{\theta^\star} = \E[\nabla^2\ell_{\theta^\star}(X, Y)]$. If $\E[\|\hat s(X, \hat Y) - s(X, \hat Y)\|^2] \rightarrow 0$ holds for some $s$, then under regularity conditions (Assumptions \ref{asm: regularity} and \ref{asm: smooth gradient} in Section \ref{sec: proof REPPI} of the Supplementary Material), $\sqrt{n}(\hat{\theta}^\cf - \theta^{\star})\xrightarrow{d} \N(0, \Sigma_s^\REPPI)$, where 
\begin{align}\label{eq:Sigma_REPPI}
    \Sigma_s^\REPPI &= H_{\theta^\star}^{-1}\bigg( \Cov(\nabla \ell_{\theta^\star}(X, Y))  -\Delta\bigg) H_{\theta^\star}^{-1}\\
    \text{and } \Delta &= \frac{1}{1+r} \Cov(\nabla \ell_{\theta^\star}(X, Y), s(X, \hat{Y}))\Cov(s(X, \hat{Y}))^{-1}\Cov(s(X, \hat{Y}), \nabla \ell_{\theta^\star}(X, Y)).\nonumber
    \end{align}
\end{theorem}
{

\begin{remark}
\label{rmk: consistency}
    In Theorem \ref{thm:main}, we require the L2 consistency of fitted score $\E[\|\hat s(X, \hat Y) - s(X, \hat Y)\|^2] \rightarrow 0$ to a fixed target $s$. In practice, for each round $k\in [1,2,3]$ in the fold rotation in Algorithm \ref{alg:crossfit}, the score fitting step would have a different target $\E[\nabla\ell_{\hat \theta_0^k}(X,Y)|X,\hat Y]$, where $\hat \theta_0^k$ is the initial estimator at this round. Denote $s^k(X,\hat Y)$ as the fitted score in round $k$, the L2 consistency assumption is satisfied as long as we have $\E[\|\hat s^k(X,\hat Y)-s_{\hat \theta_0^k}(X,\hat Y)\|^2]\rightarrow0$ and $\E[\|s_{\hat \theta_0^k}(X,\hat Y)-s_{\theta^\star}(X,\hat Y)\|^2]\rightarrow0$ for a fold specific limit $s_{\hat \theta_0^k}$ and a global limit $s_{\theta^\star}$. The first requirement $\E[\|\hat s^k(X,\hat Y)-s_{\hat \theta_0^k}(X,\hat Y)\|^2]\rightarrow0$ is the fixed target consistency of score learning, and the second requirement $\E[\|s_{\hat \theta_0^k}(X,\hat Y)-s_{\theta^\star}(X,\hat Y)\|^2]\rightarrow0$ can be achieved through the continuity for the class of function $s_\theta$ on parameter $\theta$. We impose the assumption $\E[\|\hat s(X, \hat Y) - s(X, \hat Y)\|^2] \rightarrow 0$ in the statement of Theorem \ref{thm:main} for simplicity. 
\end{remark}
}
Theorem \ref{thm:main} demonstrates two key properties of recalibrated PPI. First, if the optimal \rectifier is consistently estimated, meaning $s(X, \hat Y) = s^\star(X,\hat Y) =  \E[\nabla\ell_{\theta^\star}(X, Y)|X, \hat Y]$, then
\begin{equation}
\label{eqn: optimal RePPI var}
\Sigma_s^\REPPI = H_{\theta^\star}^{-1}\lb \Cov(\nabla \ell_{\theta^\star}(X, Y))  - \frac{1}{1+r} \Cov(\E[\nabla\ell_{\theta^\star}(X, Y)|X, \hat Y])\rb H_{\theta^\star}^{-1},    
\end{equation}
matching the efficiency of the optimal PPI estimator from Theorem \ref{thm: efficient PPI}. The second key property is that, even if the optimal \rectifier is inconsistently estimated, i.e. $s(X, \hat Y) \neq s^\star(X, \hat Y)$, recalibrated PPI remains more efficient than the XY-only approach.

Notably, in the extreme case where the prediction model is perfect, i.e., $Y = \hat Y$, then the unlabeled data $(X, \hat Y)$ effectively provides extra labeled data of size $N$. In this scenario, the imputed loss $s^\star (X, \hat Y) = \nabla\ell_{\theta}(X,\hat Y)$ satisfies the optimality condition \eqref{eq:optimal_gtheta}, and the RePPI estimator is equivalent to the empirical risk minimization over $n+N$ samples. As a result, the asymptotic variance of the RePPI estimator \eqref{eq:Sigma_REPPI} becomes $\frac{n}{n+N}H_{\theta^\star}^{-1}\Cov(\nabla \ell_{\theta^\star}(X, Y))H_{\theta^\star}^{-1}$, which equals to the asymptotic variance of the XY-only estimator with sample size $n+N$. Another notable scenario is when the prediction is solely based on the covariates $X$, i.e., $\hat Y = f(X)$ for some function $f$. Then the optimal imputed loss is $s^\star(X,\hat Y) =  \E[\nabla\ell_{\theta^\star}(X, Y)|X, \hat Y] = \E[\nabla\ell_{\theta^\star}(X, Y)|X]$, which implies that the prediction cannot help with inference. This is because $\hat Y$ is simply a measurable function of features $X$, it brings no extra information and it suffices to only use $X$ for recalibration. However, in this case, the RePPI estimator can still improve the asymptotic variance upon the XY-only estimator by $\frac{1}{1+r} H_{\theta^\star}^{-1}\Cov(\E[\nabla\ell_{\theta^\star}(X, Y)|X]) H_{\theta^\star}^{-1}$, due to the recalibration of the score function using covariates $X$.

{

\begin{remark}
The equal three-way split in Algorithm \ref{alg:crossfit} is not essential for the validity of the method; it is used only as a simple and symmetric implementation. The proof of Theorem 2 only requires that each fold has a non-vanishing asymptotic fraction of the labeled sample, i.e.,
\[
    |\mathcal D_k|/n \to \pi_k \in (0,1), \qquad k=1,2,3,
\]
and that the fold-specific estimators are combined with weights \(|\mathcal D_k|/n\), as in Algorithm 1. In practice, the score fitting may require more samples than other steps when a complicated model is used. In this scenario, it is more reasonable to evenly split the data into $K$ folds ($K\geq 3$), use more than one fold to estimate the score in Step 3 (and potentially Step 2, 4, 5), rotate the folds as in Step 6, and aggregate the estimates as in Step 7. It is straightforward to prove the same asymptotic theory for this splitting scheme, as long as the sample size of each fold is proportional to the total sample size. 
\end{remark}
\begin{remark}
For the cross-fitting scheme in Algorithm \ref{alg:crossfit}, it is important to isolate the initial estimation, score fitting, and inference on different folds. The primary reason is to make them independent, so that in each step, one can treat quantities from the other steps as fixed. For example, if the initial estimation and score fitting are on the same fold, then we cannot treat the initial estimator as independent, and one may need to impose a uniform convergence assumption on the algorithm, such as $\sup_{\theta:\|\theta-\theta^\star\|\leq \delta}\E[\|\hat s_{\theta} - s_\theta\|]\rightarrow 0$, which is much stronger than the assumption discussed in Remark \ref{rmk: consistency} and hard to satisfy for many flexible models. Moreover, if the score fitting and the inference step are merged on the same fold, then the distribution of the fitted score on the labeled data and the unlabeled data can be substantially different due to overfitting. 
However, the estimation of $\hat M$ can be merged with the inference fold, because the final estimator depends on the $\hat M$ in a simple algebraic form, and it is straightforward to control the estimation error of $\hat M$ directly. Therefore, we don't split out another fold specifically for $\hat M$ to improve the estimation accuracy on other tasks. 
\end{remark}

Theorem 2 immediately yields a plug-in construction of confidence intervals. After computing the
cross-fitted RePPI estimator \(\hat\theta^\cf\), let
\(\tilde s_i=\hat s^{(-k)}(X_i,\hat Y_i)\) denote the out-of-fold fitted score for each labeled observation
\(i\in\mathcal D_k\), and set \(\hat r=n/N\). Define the empirical covariance matrices on the labeled
sample as
$
    \widehat V_\ell
    =
    \widehat{\operatorname{Cov}}
    \{\nabla\ell_{\hat\theta^\cf}(X_i,Y_i)\},
    \widehat V_s
    =
    \widehat{\operatorname{Cov}}
    \{\tilde s_i\},
$
and
$
    \widehat C_{\ell s}
    =
    \widehat{\operatorname{Cov}}
    \{\nabla\ell_{\hat\theta^\cf}(X_i,Y_i),\tilde s_i\}.
$
We then estimate the asymptotic covariance in Theorem 2 by
\[
    \widehat\Sigma_{\mathrm{RePPI}}
    =
    \widehat H^{-1}
    \left\{
        \widehat V_\ell
        -
        \frac{N}{n+N}
        \widehat C_{\ell s}
        \widehat V_s^{-1}
        \widehat C_{\ell s}^{\top}
    \right\}
    \widehat H^{-1},
    \qquad
    \widehat H
    =
    \frac1n\sum_{i=1}^n
    \nabla^2\ell_{\hat\theta^\cf}(X_i,Y_i).
\]
If \(\widehat V_s\) is nearly singular, we replace \(\widehat V_s^{-1}\) by a ridge-regularized inverse
\((\widehat V_s+\lambda I)^{-1}\) with a small \(\lambda\ge 0\). For the \(j\)th coordinate of
\(\theta^\star\), an asymptotic \(1-\alpha\) confidence interval is
$
    \hat\theta_{j}^\cf
    \pm
    z_{1-\alpha/2}
    \sqrt{
        \frac{
            \widehat\Sigma_{\mathrm{RePPI},jj}
        }{n}
    }.
$

}
\section{Why is Recalibration Important? Stylized model analysis and applications}
\label{sec: linear model}
We showed that recalibrated PPI achieves the lowest asymptotic variance among all PPI estimators \eqref{eqn: PPI} when $s^\star$ is consistently estimated. This raises a natural question: under what conditions, and to what extent, does the recalibration step that estimates $s^\star$ provide a benefit? In this section, we quantify the efficiency gains through three case studies, each representing a typical scenario where machine learning predictions systematically differ from the outcome of interest. For each case, we compare our method against the XY-only, PPI, and PPI++ estimators, described in Section \ref{sec: surrogate and PPI} under our unified framework. We denote by $\Sigma^\cls, \Sigma^\ppi, \Sigma^\ppiplus, \Sigma^\REPPI$ the respective asymptotic covariance matrices.

For each scenario, we provide an application on a real-world dataset with predictions from pretrained AI models to demonstrate the importance of recalibration. In each application, we have a fully labeled dataset, which we randomly split into a labeled portion and an unlabeled portion. For the purpose of evaluating coverage, we take the value of the estimand on the full labeled dataset as a proxy for $\theta^\star$. We estimate the average width of $95\%$ confidence intervals and coverage over 100 trials. All code is available at \url{https://github.com/Wenlong2000/RePPI/}.

\subsection{Modality Mismatch}\label{sec:modality_mismatch}
In modern applications, researchers can easily access data from various modalities beyond tabular data, including audios, images, and texts. Seeing that such data is often complex and high-dimensional, it is usually not involved in the definition of the target parameter. Nonetheless, the age of AI produces a plethora of prediction models that can turn unstructured data into numerical representations or predictions. In this scenario, $\hat{Y}$ takes as input the unstructured data $W$ and leaves the inferentially relevant covariates $X$ on the table. We call this scenario modality mismatch.
\subsubsection{Stylized Model}
To study the efficiency gain quantitatively, we consider a stylized setting where $Y$ is generated as $
    Y = W^\top \gamma + X^\top \theta + \epsilon,
$ where $X\sim \mathcal{N}(0, \Sigma_X), W\sim \mathcal{N}(0, \Sigma_W)$, and $\epsilon\sim \mathcal{N}(0, \sigma^2)$ are independent Gaussian variables. The predictions $\hat{Y}$ are generated from a misspecified linear model, by regressing $Y$ on $W$ solely. The target $\theta^\star$ is given by the squared loss $\ell_\theta(X, Y) = (Y - X^\top \theta)^2$, i.e. $\theta^\star$ is the best linear approximation of $Y$ from $X$.
The independence assumption implies that $\theta^\star = \theta$ and the predictions converge to $W^\top \gamma$. For simplicity, we ignore the sampling uncertainty of the prediction model and assume $\hat{Y} = W^\top \gamma$. We derive the asymptotic variances of the baselines in the following proposition.  We denote by $\Sigma^\cls, \Sigma^\ppi, \Sigma^\ppiplus, \Sigma^\REPPI$ the respective asymptotic covariance matrices.
\begin{proposition}
\label{prop: lm independent}
Assume that $n/N \rightarrow r$. Then, in the setting described above,
    \begin{equation*}
    \begin{aligned}
        \tr(\Sigma^\cls) = &\lb\sigma^2+\gamma^\top \Sigma_W\gamma\rb\tr(\Sigma_X^{-1}),\\
        \tr(\Sigma^{\ppi}) = &(\sigma^2+ (1+r) \theta^\top \Sigma_X\theta + r\gamma^\top \Sigma_W\gamma) \tr(\Sigma_X^{-1})+(1+r) \|\theta\|^2,\\
        \tr(\Sigma^{\ppiplus}) =&\left(\sigma^2 + \left(\frac{r}{1+r} + \frac{1}{1+r}\frac{1}{1+\frac{\gamma^\top \Sigma_W\gamma}{\|\theta\|^2/\tr(\Sigma_X^{-1}) +\theta^\top \Sigma_X\theta }}\right)\gamma^\top \Sigma_W \gamma\right)\tr(\Sigma_X^{-1}) ,\\
        \tr(\Sigma^{\REPPI}) =& \lb\sigma^2+\frac{r}{1+r}\gamma^\top \Sigma_W\gamma\rb\tr(\Sigma_X^{-1}).
    \end{aligned}
    \end{equation*}
    In particular, $\tr(\Sigma^{\REPPI}) \leq \tr(\Sigma^{\ppiplus}) \leq \min\{\tr(\Sigma^\cls), \tr(\Sigma^\ppi)\}$.
\end{proposition} 
Comparing $\tr(\Sigma^\cls)$ with $\tr(\Sigma^\ppi)$, we can see that PPI can be worse than the XY-only estimator when $\|\theta\|$ is large. Comparing $\tr(\Sigma^\cls)$, $\tr(\Sigma^\ppiplus)$, and $\tr(\Sigma^\REPPI)$, we observe that the contribution of the predictions is to reduce the variance caused by $W$. The maximal reduction is $1/(1+r) \Var(W^\top \gamma)$, achieved by RePPI. The reduction of PPI++ ranges from $0$ to the maximal amount, depending on the relative contribution of $X$ and $W$ in the model of $Y$. When $\theta = 0$, PPI++ is as efficient as RePPI; when $\|\theta\|\rightarrow \infty$, PPI++ approaches the XY-only estimator. The efficiency gain of RePPI compared to PPI++ is given by 
$\frac{1}{1+r}\lb \frac{1}{\gamma^\top \Sigma_W \gamma} + \frac{1}{\|\theta\|^2/\tr(\Sigma_X^{-1}) + \theta^\top \Sigma_X\theta}\rb^{-1}\cdot \tr(\Sigma_X^{-1}).$
Thus, RePPI gains substantially over PPI++ when both $W$ and $X$ are informative.

{

\subsubsection{Application to X-ray image diagnosis}
We next consider a medical-imaging application to demonstrate the gain of RePPI in a modality-mismatch setting. We use the NIH ChestX-ray14 dataset
\citep{wang2017chestx}. The outcome is whether a patient has pleural effusion.
The target parameter is the coefficient of standardized age in a logistic regression of the effusion label on age, sex, and X-ray view position. The surrogate prediction is produced by a DenseNet-121 chest X-ray classifier pretrained
on CheXpert \citep{irvin2019chexpert} and implemented through TorchXRayVision
\citep{Cohen2022xrv}. Importantly, this model only takes
the radiograph as input. It does not use the structured covariates that define the
inferential target. This creates a concrete modality mismatch: the image model provides
useful diagnostic information, but its prediction is not directly calibrated for
age-adjusted inference.

We restrict the analysis to adult patients with frontal-view scans and use one image per patient. After preprocessing, the dataset contains \(29{,}366\) observations. We randomly split the data into labeled and unlabeled portions, and vary the labeled fraction from 10\% to 30\%, and compare XY-only, PPI, PPI++, and RePPI. We use a logistic regression model to fit the recalibration score in RePPI. The full-data estimate is used as a proxy for the target parameter when evaluating coverage. The experimental results are reported in Figure~\ref{fig:nih-cxr} and Table \ref{tab:nih-cxr}. Standard PPI performs poorly because the raw image-only
prediction is not on the right scale for the downstream regression target. PPI++ slightly improves upon the XY-only estimator through optimal power tuning, while RePPI further improves efficiency by recalibrating the image-based surrogate using the structured covariates. Across the labeled fractions considered, RePPI yields the shortest intervals and reduces the sample size requirement by approximately 5\% over PPI++ to reach a targeted confidence interval length.

\begin{figure}[t]
    \centering
    \includegraphics[width=\textwidth]{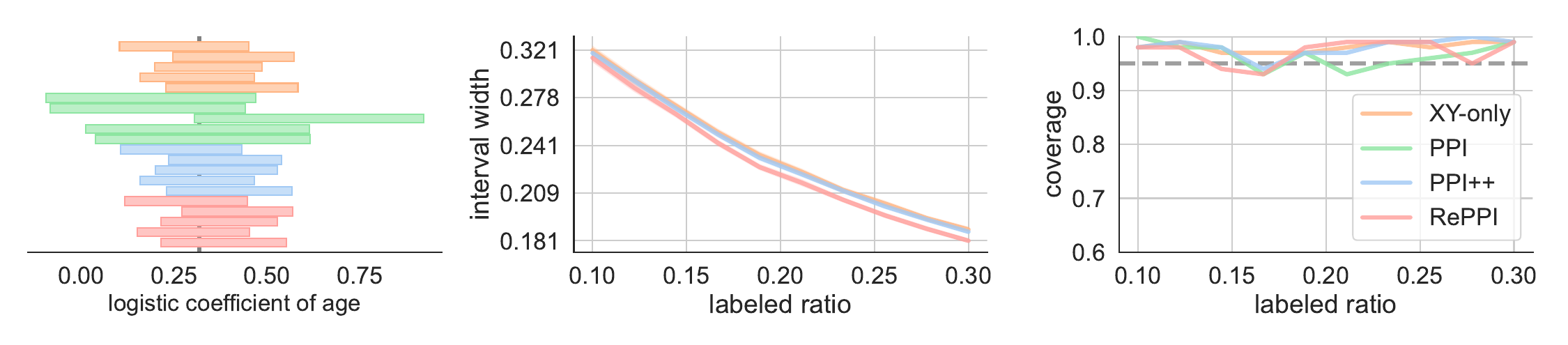}
    \caption{
    NIH ChestX-ray14 experiment.  The left panel shows a representative 95\% confidence interval, the middle panel shows average interval width, and the right panel shows empirical coverage. The horizontal axis represents the ratio of labeled data, $\frac{n}{N+n}$. We omit the interval length curve for PPI because it is significantly wider than the other methods.
    }
    \label{fig:nih-cxr}
\end{figure}
\begin{table}[t]
\centering
\caption{
Required number of labeled observations to achieve a given interval length in the
NIH ChestX-ray14 modality-mismatch experiment. The last column reports the
label reduction of RePPI relative to PPI++. Here ``NA'' means that the target
interval length is not reached within the labeled-sample-size range considered.
}
\label{tab:nih-cxr}
\begin{tabular}{c|cccc|c}
\hline
Interval length
& XY-only & PPI & PPI++ & RePPI
& Reduced samples (\%) \\
\hline
0.20 & 7704 & NA & 7602 & 7164 & 5.74\% \\
0.22 & 6285 & NA & 6178 & 5934 & 3.93\% \\
0.24 & 5374 & NA & 5273 & 4963 & 5.86\% \\
0.26 & 3955 & NA & 4493 & 4315 & 3.95\% \\
\hline
\end{tabular}
\end{table}
\subsection{Distribution Shift}
Pre-trained models like LLMs are often trained on general-purpose data that come from generic sources. However, researchers may target only a specific subpopulation relevant to their study. This distinction between the target subpopulation and the population used for model training surfaces as distribution shift, leading to inaccurate predictions.
}
\subsubsection{Stylized Model}
To study the effect of distribution shift, we consider the same setting as in Section \ref{sec:modality_mismatch}, except that the prediction model takes both $X$ and $W$ as inputs. Again we consider a linear prediction model $
    \hat{Y} = X^\top \td{\theta} + W^\top \tilde\gamma
$, but its coefficients are misspecified:

\begin{proposition}
\label{prop: lm shift}
Assume that $n/N \rightarrow r$. Then, in the setting described above,
\begin{align*}
        \tr(\Sigma^\cls) =& \lb\sigma^2+\gamma^\top \Sigma_W\gamma\rb\tr(\Sigma_X^{-1}),\\
        \tr(\Sigma^{\ppi}) = & (1+r) \|\theta-\tilde \theta\|^2  + \Big(\sigma^2+ (1+r) (\theta -\tilde\theta)^\top \Sigma_X(\theta -\tilde\theta) + r\tilde \gamma^\top \Sigma_W\tilde \gamma \\
        & + (\gamma- \td{\gamma})^\top \Sigma_W(\gamma - \td{\gamma})\Big) \tr(\Sigma_X^{-1}),\\
\tr(\Sigma^{\ppiplus}) = &\left(\sigma^2 + \gamma^\top \Sigma_W \gamma - \frac{1}{1+r}\frac{(\gamma^\top \Sigma_W \td{\gamma})^2}{\td{\gamma}^\top \Sigma_W \td{\gamma} + (\theta -\tilde\theta)^\top \Sigma_X(\theta -\tilde\theta) + \frac{\|\theta -\tilde\theta\|^2}{ \tr(\Sigma_X^{-1})}}\right)\tr(\Sigma_X^{-1}),\\
\tr(\Sigma^{\REPPI}) =& \left(\sigma^2 + \gamma^\top \Sigma_W \gamma - \frac{1}{1+r}\frac{(\gamma^\top \Sigma_W \td{\gamma})^2}{\td{\gamma}^\top \Sigma_W \td{\gamma}}\right)\tr(\Sigma_X^{-1}).
\end{align*}
In particular, $\tr(\Sigma^{\REPPI}) \leq \tr(\Sigma^{\ppiplus}) \leq \min\{\tr(\Sigma^\cls), \tr(\Sigma^\ppi)\}$.
\end{proposition}
Comparing $\Sigma^\ppi$ with $\Sigma^\cls$, we can see that PPI can be worse than the XY-only estimator when $\td{\gamma}$ is very different from $\gamma$. RePPI always improves upon the XY-only estimator unless $W^\top\td{\gamma}$ is orthogonal to $W^\top\gamma$ in the sense that $\gamma^\top \Sigma_W \td{\gamma} = 0$.
Comparing $\Sigma^\ppiplus$ with $\Sigma^\REPPI$, we conclude that the gain of RePPI over PPI++ is increasing in the bias of $\td{\theta}$. This illustrates the importance of recalibration under distribution shifts. 

{

\subsubsection{Application to Toxicity Measurement in Online Comments}
We next consider a text-classification application to demonstrate the gain of RePPI in a distribution shift setting. We use the CivilComments-WILDS dataset
\citep{borkan2019nuanced,wilds2021}, which consists of more than 400,000 public comments from 50 English-language news sites. Each observation consists of an online
comment, a human toxicity score, and identity annotations. The outcome is the human
toxicity score, scaled between 0 and 1. The target parameter is the regression coefficient on whether the comment mentions an identity group, and we use the comment length as a confounder.

As the prediction surrogate, we use a pretrained Distil-Bert model \citep{sanh2019distilbert} downloaded from Hugging Face\footnote{\url{https://huggingface.co/citizenlab/distilbert-base-multilingual-cased-toxicity}}, which outputs a probability that a comment is toxic. This model is trained on the JIGSAW toxic comment classification challenge dataset \citep{jigsaw-toxic-comment-classification-challenge} that collects the comments from Wikipedia and is labeled by a human rater. Therefore, although its prediction is informative, it is trained to detect toxicity in a knowledge-sharing environment instead of news discussion. This creates a natural distribution-shift setting: the surrogate comes from a prediction model trained in a different environment from the population used for inference.

We subsample \(10{,}000\) comments for the experiment. We randomly split the data into
labeled and unlabeled portions, and vary the labeled fraction from 10\% to 30\%, and compare XY-only, PPI, PPI++, and RePPI. We use a random forest model to fit the recalibration score in RePPI. The full-data estimate is used as a proxy for the target parameter when evaluating coverage. The experimental results are reported in Figure~\ref{fig:civilcomments} and Table \ref{tab:civil comments}. Standard PPI produces wider intervals than the XY-only estimator, reflecting the mis-calibration of the raw toxicity-model prediction under distribution shift. PPI++ improves over the XY-only estimator with the use of optimal power tuning. RePPI further improves over PPI++ by recalibrating the surrogate prediction trained on Wikipedia comments to the news comments population. Across the labeled fractions considered, RePPI gives the shortest intervals, with a roughly 5\% reduction in required labels relative to PPI++ for the interval lengths examined.

\begin{figure}[t]
    \centering
    \includegraphics[width=\textwidth]{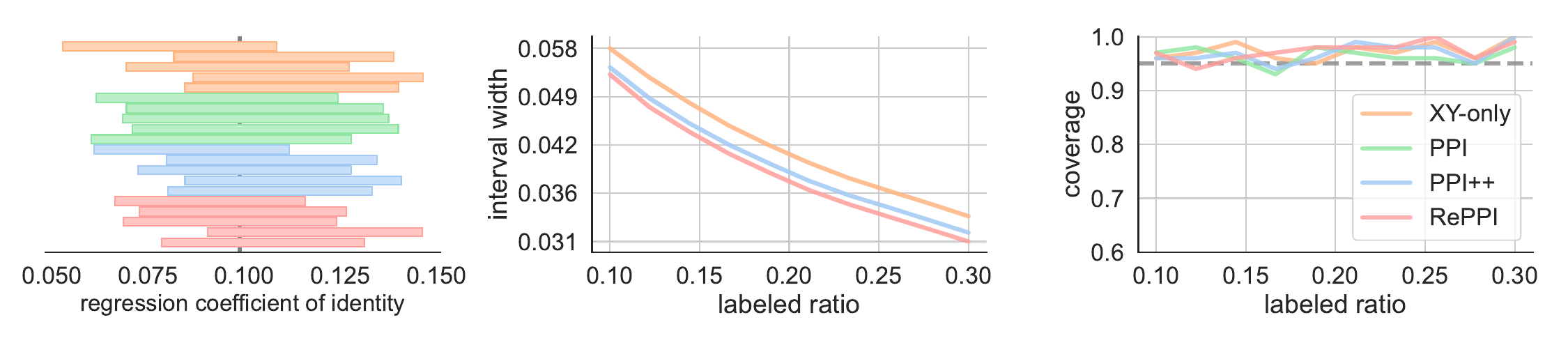}
    \caption{
    CivilComments distribution-shift experiment. The figure is produced similarly to Figure \ref{fig:nih-cxr}.
    }
    \label{fig:civilcomments}
\end{figure}
\begin{table}[t]
\centering
\caption{
CivilComments-WILDS experiment.  The table is produced similarly to Table \ref{tab:nih-cxr}.
}
\label{tab:civil comments}
\begin{tabular}{c|cccc|c}
\hline
Interval length
& XY-only & PPI & PPI++ & RePPI
& Reduced samples (\%) \\
\hline
0.040 & 2075 & NA   & 1847 & 1739 & 5.85\% \\
0.042 & 1885 & NA   & 1666 & 1578 & 5.32\% \\
0.044 & 1723 & 2638 & 1521 & 1433 & 5.77\% \\
0.046 & 1583 & 2314 & 1389 & 1313 & 5.53\% \\
\hline
\end{tabular}
\end{table}
}
\subsection{Discrete Predictions}

Machine learning models often produce only coarse discrete predictions. For example, \cite{hofer2024bayesian} discuss examples where LLMs used as autoraters produce discrete, uncalibrated responses.
In such cases, substituting $\hat Y$ for ${Y}$ is suboptimal, even if $\hat{Y}$ is monotonic in $Y$. Recalibration essentially brings the predictions back to the scale of the true outcomes. 
\subsubsection{Stylized Model}
To illustrate the efficiency gain quantitatively, we consider mean estimation, $\theta^\star = \E[Y]$, obtained by taking $\ell_\theta(Y) = (Y - \theta)^2/2$. Consider the mixture model where $Z\sim \mathrm{Unif}(\{1,2,3\})$ and $Y\mid Z \sim \N(\mu_Z, \sigma^2)$.
For simplicity, we assume the prediction equals the label of the mixture component that $Y$ is generated from: $\hat Y = Z$. We derive the asymptotic variances in the following proposition.
\begin{proposition}
\label{prop: binary}
Assume that $n/N \rightarrow r$. 
    In the setting described above, 
    \begin{equation*}
\begin{aligned}
    &\Sigma^\cls = \sigma^2 + \frac{(\mu_1 - \mu_2)^2 + (\mu_2 - \mu_3)^2 + (\mu_3 - \mu_1)^2}{9}, \Sigma^\ppi = \Sigma^\cls + \frac{2(1+r)}{3} - \frac{2(\mu_3-\mu_1)}{3}, \\
    &\Sigma^\ppiplus = \Sigma^\cls - \frac{1}{1+r}\frac{(\mu_3-\mu_1)^2}{6}, \Sigma^\REPPI = \Sigma^\cls - \frac{1}{1+r}\frac{(\mu_1 - \mu_2)^2 + (\mu_2 - \mu_3)^2 + (\mu_3 - \mu_1)^2}{9}. \\
\end{aligned} 
\end{equation*}
In particular, $\Sigma^\ppiplus - \Sigma^\REPPI =  \frac{1}{1+r}\frac{(2\mu_2 - \mu_3- \mu_1)^2}{18} \geq 0$.
\end{proposition}
From Proposition \ref{prop: binary}, we find that the improvement of RePPI over PPI++ is proportional to $((\mu_3-\mu_2) - (\mu_2-\mu_1))^2$, which represents the linearity of the mixture distribution. When $\mu_3 - \mu_2 = \mu_2 - \mu_1$ holds, then PPI++ has sufficient calibration. Otherwise, we need a nonlinear mapping to calibrate the predictions. Therefore, the linear calibration of PPI++ is insufficient and the nonparametric calibration of RePPI is necessary. 

{

\subsubsection{Applications to Wine Reviews}

We next consider a wine-review application that serves as a real-data counterpart
to the discrete-prediction setting. We use the WineEnthusiast wine
review dataset\footnote{\url{https://www.kaggle.com/datasets/mysarahmadbhat/wine-tasting}}, which contains wine characteristics such as country, price, and region,
together with a taster's written review and an integer rating between 80 and 100. The
target parameter is the coefficient of price in a regression of the rating on price,
and we include the wine region as a confounder.

As the surrogate prediction, we use OpenAI's GPT-5.4 nano model to predict the wine
rating from the written review. The model is prompted to output a rating from five categories: (1) very bad, (2) bad, (3) OK, (4) good, (5) very good. And we then use the category index 1-5 as the surrogate prediction. This creates a discrete-outcome setting: the surrogate predictions are coarse ordinal scores rather than continuous measurements. Even when the language model captures the overall sentiment or quality described in the review, its discrete rating may not be calibrated to the human taster's scoring scale. For example, most of the reviews can be considered as good or very good, making the distribution of predictions heavily skewed.

We subsample \(10{,}000\) wines from the United States for the experiment. We randomly split the data into labeled and unlabeled portions, and vary the labeled fraction from 10\% to 30\%, and compare XY-only, PPI, PPI++, and RePPI. We use a two-layer neural network model to fit the recalibration score in RePPI. The full-data estimate is used as a proxy for the target parameter when evaluating coverage. The results are shown in Figure~\ref{fig: wine} and Table~\ref{tab: wine}. RePPI improves uniformly over the other methods. This is consistent with the discrete
prediction mechanism studied above: a raw discrete prediction is informative,
but it need not be on the correct calibrated scale for inference. PPI++ can only apply
a linear correction to the raw rating, while RePPI learns a more flexible recalibration
of the discrete surrogate. In this experiment, RePPI reduces the number of labeled
examples needed to reach the same interval length by about 18--19\% relative to PPI++.

\begin{figure}[h]
    \centering
    \includegraphics[width=0.98\linewidth]{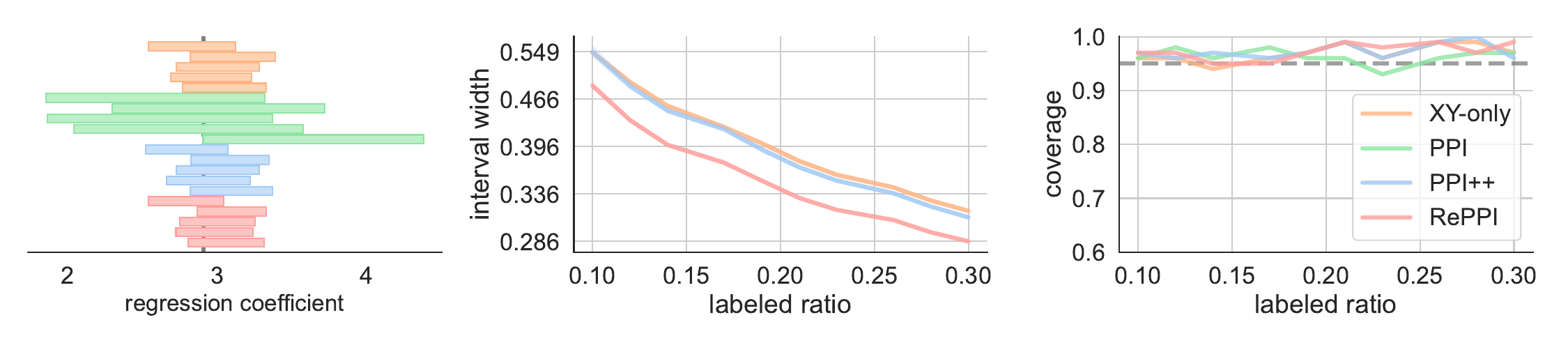}
    \caption{Wine review experiments. The figure is produced similarly to Figure \ref{fig:nih-cxr}.}
    \label{fig: wine}
\end{figure}
\begin{table}[tbh!]
\centering
\caption{Wine review experiment. The table is produced similarly to Table \ref{tab:nih-cxr}.
}
\begin{tabular}{c||ccccc}
         \hline
 Interval Length & XY-only & PPI & PPI++ & RePPI & Reduced Samples (\%) \\
\hline
0.33 & 2744 & NA & 2625 & 2163 & 17.57\% \\
0.35 & 2457 & NA & 2389 & 1875 & 21.50\% \\
0.37 & 2151 & NA & 2110 & 1725 & 18.22\% \\
0.39 & 1902 & NA & 1856 & 1529 & 17.61\% \\
\hline
    \end{tabular}
\label{tab: wine}
\end{table}
}
\section{Conclusion}
\label{sec:discussion}
In this paper, we establish the connection between the classical surrogate outcome model and the recent advancements in prediction-powered inference (PPI). From this unified perspective, we develop recalibrated prediction-powered inference (RePPI), a more efficient approach for leveraging machine learning predictions to improve statistical inference. The RePPI estimator offers both stronger theoretical guarantees in terms of asymptotic efficiency and notable computational benefits. We demonstrate the benefits of the RePPI estimator across various scenarios, including modality mismatch, distribution shift, and discrete predictions, through both theoretical analysis and empirical simulations. Additionally, we demonstrate our method on several real-world datasets, showing that it consistently yields tighter confidence intervals than existing PPI methods. 

This work has several limitations that suggest directions for future research. First, our main formulation uses the high-dimensional or unstructured covariates $W$ only through the prediction $\hat Y=f(X,W)$, which is natural when $W$ consists of text, images, or other difficult-to-model modalities, but may discard information contained in richer model outputs such as embeddings, logits, uncertainty scores, or multiple predictions. Second, our theoretical guarantees are asymptotic, and finite-sample performance depends on the quality of the learned imputed score, the labeled sample size, and the strength of the surrogate signal. Finally, our settings mostly focus on i.i.d. labeled and unlabeled samples. It would be practically important to explore how to integrate unlabeled data from different sources in a practical setting.

\bibliography{Better_PPI}

@article{angelopoulos2023ppi++,
  title =	 {{PPI}++: Efficient prediction-powered inference},
  author =	 {Angelopoulos, Anastasios N and Duchi, John C and
                  Zrnic, Tijana},
  journal =	 {arXiv preprint arXiv:2311.01453},
  year =	 {2023}
}

@article{mccaw2024synthetic,
  title={Synthetic surrogates improve power for genome-wide association studies of partially missing phenotypes in population biobanks},
  author={McCaw, Zachary R and Gao, Jianhui and Lin, Xihong and Gronsbell, Jessica},
  journal={Nature genetics},
  volume={56},
  number={7},
  pages={1527--1536},
  year={2024},
  publisher={Nature Publishing Group US New York}
}

@inproceedings{wang2017chestx,
  title={Chestx-ray8: Hospital-scale chest x-ray database and benchmarks on weakly-supervised classification and localization of common thorax diseases},
  author={Wang, Xiaosong and Peng, Yifan and Lu, Le and Lu, Zhiyong and Bagheri, Mohammadhadi and Summers, Ronald M},
  booktitle={Proceedings of the IEEE conference on computer vision and pattern recognition},
  pages={2097--2106},
  year={2017}
}

@inproceedings{Cohen2022xrv,
title = {{TorchXRayVision: A library of chest X-ray datasets and models}},
author = {Cohen, Joseph Paul and Viviano, Joseph D. and Bertin, Paul and Morrison, Paul and Torabian, Parsa and Guarrera, Matteo and Lungren, Matthew P and Chaudhari, Akshay and Brooks, Rupert and Hashir, Mohammad and Bertrand, Hadrien},
booktitle = {Medical Imaging with Deep Learning},
url = {https://github.com/mlmed/torchxrayvision},
arxivId = {2111.00595},
year = {2022}
}

@inproceedings{irvin2019chexpert,
  title={Chexpert: A large chest radiograph dataset with uncertainty labels and expert comparison},
  author={Irvin, Jeremy and Rajpurkar, Pranav and Ko, Michael and Yu, Yifan and Ciurea-Ilcus, Silviana and Chute, Chris and Marklund, Henrik and Haghgoo, Behzad and Ball, Robyn and Shpanskaya, Katie and others},
  booktitle={Proceedings of the AAAI conference on artificial intelligence},
  volume={33},
  pages={590--597},
  year={2019}
}

@inproceedings{borkan2019nuanced, title={Nuanced metrics for measuring unintended bias with real data for text classification}, author={Borkan, Daniel and Dixon, Lucas and Sorensen, Jeffrey and Thain, Nithum and Vasserman, Lucy}, booktitle={Companion Proceedings of The 2019 World Wide Web Conference}, pages={491--500}, year={2019} }

@inproceedings{wilds2021, title = {{WILDS}: A Benchmark of in-the-Wild Distribution Shifts}, author = {Pang Wei Koh and Shiori Sagawa and Henrik Marklund and Sang Michael Xie and Marvin Zhang and Akshay Balsubramani and Weihua Hu and Michihiro Yasunaga and Richard Lanas Phillips and Irena Gao and Tony Lee and Etienne David and Ian Stavness and Wei Guo and Berton A. Earnshaw and Imran S. Haque and Sara Beery and Jure Leskovec and Anshul Kundaje and Emma Pierson and Sergey Levine and Chelsea Finn and Percy Liang}, booktitle = {International Conference on Machine Learning (ICML)}, year = {2021} }

@article{sanh2019distilbert,
  title={DistilBERT, a distilled version of BERT: smaller, faster, cheaper and lighter},
  author={Sanh, Victor and Debut, Lysandre and Chaumond, Julien and Wolf, Thomas},
  journal={arXiv preprint arXiv:1910.01108},
  year={2019}
}

@incollection{huber1992robust,
  title={Robust estimation of a location parameter},
  author={Huber, Peter J},
  booktitle={Breakthroughs in statistics: Methodology and distribution},
  pages={492--518},
  year={1992},
  publisher={Springer}
}

@article{charbonnier1997deterministic,
  title={Deterministic edge-preserving regularization in computed imaging},
  author={Charbonnier, Pierre and Blanc-F{\'e}raud, Laure and Aubert, Gilles and Barlaud, Michel},
  journal={IEEE Transactions on image processing},
  volume={6},
  number={2},
  pages={298--311},
  year={1997},
  publisher={IEEE}
}

@misc{jigsaw-toxic-comment-classification-challenge,
    author = {cjadams and Jeffrey Sorensen and Julia Elliott and Lucas Dixon and Mark McDonald and nithum and Will Cukierski},
    title = {Toxic Comment Classification Challenge},
    year = {2017},
    howpublished = {\url{https://kaggle.com/competitions/jigsaw-toxic-comment-classification-challenge}},
    note = {Kaggle}
}

@article{miao2024task,
  title={Task-agnostic machine-learning-assisted inference},
  author={Miao, Jiacheng and Lu, Qiongshi},
  journal={Advances in Neural Information Processing Systems},
  volume={37},
  pages={106162--106189},
  year={2024}
}

@article{miao2024valid,
  title={Valid inference for machine learning-assisted genome-wide association studies},
  author={Miao, Jiacheng and Wu, Yixuan and Sun, Zhongxuan and Miao, Xinran and Lu, Tianyuan and Zhao, Jiwei and Lu, Qiongshi},
  journal={Nature genetics},
  volume={56},
  number={11},
  pages={2361--2369},
  year={2024},
  publisher={Nature Publishing Group US New York}
}

@article{miao2025assumption,
  title={Assumption-lean and data-adaptive post-prediction inference},
  author={Miao, Jiacheng and Miao, Xinran and Wu, Yixuan and Zhao, Jiwei and Lu, Qiongshi},
  journal={Journal of Machine Learning Research},
  volume={26},
  number={179},
  pages={1--31},
  year={2025}
}

@article{testa2025semiparametric,
  title={Semiparametric semi-supervised learning for general targets under distribution shift and decaying overlap},
  author={Testa, Lorenzo and Xu, Qi and Lei, Jing and Roeder, Kathryn},
  journal={arXiv preprint arXiv:2505.06452},
  year={2025}
}

@article{wang2020methods,
  title={Methods for correcting inference based on outcomes predicted by machine learning},
  author={Wang, Siruo and McCormick, Tyler H and Leek, Jeffrey T},
  journal={Proceedings of the National Academy of Sciences},
  volume={117},
  number={48},
  pages={30266--30275},
  year={2020},
  publisher={National Academy of Sciences}
}

@article{egami2023using,
  title={Using imperfect surrogates for downstream inference: Design-based supervised learning for social science applications of large language models},
  author={Egami, Naoki and Hinck, Musashi and Stewart, Brandon and Wei, Hanying},
  journal={Advances in Neural Information Processing Systems},
  volume={36},
  pages={68589--68601},
  year={2023}
}

@article{rister2025correcting,
  title={Correcting the Measurement Errors of AI-Assisted Labeling in Image Analysis Using Design-Based Supervised Learning},
  author={Rister Portinari Maranca, Alessandra and Chung, Jihoon and Hinck, Musashi and Wolsky, Adam D and Egami, Naoki and Stewart, Brandon M},
  journal={Sociological Methods \& Research},
  pages={00491241251333372},
  year={2025},
  publisher={SAGE Publications Sage CA: Los Angeles, CA}
}

@article{chakrabortty2018efficient,
  title={Efficient and adaptive linear regression in semi-supervised settings},
  author={Chakrabortty, Abhishek and Cai, Tianxi},
   journal={The Annals of Statistics},
  volume={46},
  number={4},
  pages={1541--1572},
  year={2018}
}

@article{zhang2019semi,
  title={Semi-supervised inference: General theory and estimation of means},
  author={Zhang, Anru and Brown, Lawrence D and Cai, T Tony},
journal={The Annals of Statistics},
  volume={47},
  number={5},
  pages={2538--2566},
  year={2019}
}

@article{yang2019combining,
  title =	 {Combining multiple observational data sources to
                  estimate causal effects},
  author =	 {Yang, Shu and Ding, Peng},
  journal =	 {Journal of the American Statistical Association},
  year =	 {2019},
  publisher =	 {Taylor \& Francis}
}

@article{angelopoulos2023prediction,
  title={Prediction-powered inference},
  author={Angelopoulos, Anastasios N and Bates, Stephen and Fannjiang, Clara and Jordan, Michael I and Zrnic, Tijana},
  journal={Science},
  volume={382},
  number={6671},
  pages={669--674},
  year={2023},
  publisher={American Association for the Advancement of Science}
}

@book{keener2010theoretical,
  title={Theoretical statistics: Topics for a core course},
  author={Keener, Robert W},
  year={2010},
  publisher={Springer Science \& Business Media}
}

@article{xu2025unified,
  title={A Unified Framework for Semiparametrically Efficient Semi-Supervised Learning},
  author={Xu, Zichun and Witten, Daniela and Shojaie, Ali},
  journal={arXiv preprint arXiv:2502.17741},
  year={2025}
}

@article{gligoric2024can,
  title={Can Unconfident LLM Annotations Be Used for Confident Conclusions?},
  author={Gligori{\'c}, Kristina and Zrnic, Tijana and Lee, Cinoo and Cand{\`e}s, Emmanuel J and Jurafsky, Dan},
  journal={arXiv preprint arXiv:2408.15204},
  year={2024}
}

@article{zrnic2024cross,
  title={Cross-prediction-powered inference},
  author={Zrnic, Tijana and Cand{\`e}s, Emmanuel J},
  journal={Proceedings of the National Academy of Sciences},
  volume={121},
  number={15},
  pages={e2322083121},
  year={2024},
  publisher={National Acad Sciences}
}

@article{chen2008improving,
  title={Improving semiparametric estimation by using surrogate data},
  author={Chen, Song Xi and Leung, Denis HY and Qin, Jing},
  journal={Journal of the Royal Statistical Society Series B: Statistical Methodology},
  volume={70},
  number={4},
  pages={803--823},
  year={2008},
  publisher={Oxford University Press}
}

@article{pepe1992inference,
  title={Inference using surrogate outcome data and a validation sample},
  author={Pepe, Margaret Sullivan},
  journal={Biometrika},
  volume={79},
  number={2},
  pages={355--365},
  year={1992},
  publisher={Oxford University Press}
}

@article{wittes1989surrogate,
  title={Surrogate endpoints in clinical trials: cardiovascular diseases},
  author={Wittes, Janet and Lakatos, Edward and Probstfield, Jeffrey},
  journal={Statistics in medicine},
  volume={8},
  number={4},
  pages={415--425},
  year={1989},
  publisher={Wiley Online Library}
}

@article{robins1994estimation,
  title={Estimation of regression coefficients when some regressors are not always observed},
  author={Robins, James M and Rotnitzky, Andrea and Zhao, Lue Ping},
  journal={Journal of the American statistical Association},
  volume={89},
  number={427},
  pages={846--866},
  year={1994},
  publisher={Taylor \& Francis}
}

@article{chen2000unified,
  title={A unified approach to regression analysis under double-sampling designs},
  author={Chen, Yi-Hau and Chen, Hung},
  journal={Journal of the Royal Statistical Society Series B: Statistical Methodology},
  volume={62},
  number={3},
  pages={449--460},
  year={2000},
  publisher={Oxford University Press}
}

@article{chen2008semiparametric,
  title={Semiparametric Efficiency in GMM Models with Auxiliary Data},
  author={Chen, Xiaohong and Hong, Han and Tarozzi, Alessandro},
  journal={The Annals of Statistics},
  pages={808--843},
  year={2008},
  publisher={JSTOR}
}

@article{zrnic2024active,
  title={Active Statistical Inference},
  author={Zrnic, Tijana and Cand{\`e}s, Emmanuel J},
  journal={arXiv preprint arXiv:2403.03208},
  year={2024}
}

@book{van2000asymptotic,
  title={Asymptotic statistics},
  author={Van der Vaart, Aad W},
  volume={3},
  year={2000},
  publisher={Cambridge university press}
}

@article{gan2024prediction,
  title={Prediction de-correlated inference: A safe approach for post-prediction inference},
  author={Gan, Feng and Liang, Wanfeng and Zou, Changliang},
  journal={Australian \& New Zealand Journal of Statistics},
  year={2024},
  publisher={Wiley Online Library}
}

@article{chen2003information,
  title={Information recovery in a study with surrogate endpoints},
  author={Chen, Song Xi and Leung, Denis H Y and Qin, Jing},
  journal={Journal of the American Statistical Association},
  volume={98},
  number={464},
  pages={1052--1062},
  year={2003},
  publisher={Taylor \& Francis}
}

@article{prentice1989surrogate,
  title={Surrogate endpoints in clinical trials: definition and operational criteria},
  author={Prentice, Ross L},
  journal={Statistics in medicine},
  volume={8},
  number={4},
  pages={431--440},
  year={1989},
  publisher={Wiley Online Library}
}

@article{fleming1994surrogate,
  title={Surrogate and auxiliary endpoints in clinical trials, with potential applications in cancer and AIDS research},
  author={Fleming, Thomas R and Prentice, Ross L and Pepe, Margaret S and Glidden, David},
  journal={Statistics in medicine},
  volume={13},
  number={9},
  pages={955--968},
  year={1994},
  publisher={Wiley Online Library}
}

@article{chen2005measurement,
  title={Measurement error models with auxiliary data},
  author={Chen, Xiaohong and Hong, Han and Tamer, Elie},
  journal={The Review of Economic Studies},
  volume={72},
  number={2},
  pages={343--366},
  year={2005},
  publisher={Wiley-Blackwell}
}

@article{post2010analysis,
  title={The analysis of longitudinal quality of life measures with informative drop-out: a pattern mixture approach},
  author={Post, Wendy J and Buijs, Ciska and Stolk, Ronald P and de Vries, Elisabeth GE and Le Cessie, Saskia},
  journal={Quality of Life Research},
  volume={19},
  pages={137--148},
  year={2010},
  publisher={Springer}
}

@article{kallus2024role,
  title={On the role of surrogates in the efficient estimation of treatment effects with limited outcome data},
  author={Kallus, Nathan and Mao, Xiaojie},
  journal={Journal of the Royal Statistical Society Series B: Statistical Methodology},
  pages={qkae099},
  year={2024},
  publisher={Oxford University Press UK}
}

@techreport{athey2019surrogate,
  title={The surrogate index: Combining short-term proxies to estimate long-term treatment effects more rapidly and precisely},
  author={Athey, Susan and Chetty, Raj and Imbens, Guido W and Kang, Hyunseung},
  year={2019},
  institution={National Bureau of Economic Research}
}

@article{gupta2019top,
  title={Top challenges from the first practical online controlled experiments summit},
  author={Gupta, Somit and Kohavi, Ronny and Tang, Diane and Xu, Ya and Andersen, Reid and Bakshy, Eytan and Cardin, Niall and Chandran, Sumita and Chen, Nanyu and Coey, Dominic and others},
  journal={ACM SIGKDD Explorations Newsletter},
  volume={21},
  number={1},
  pages={20--35},
  year={2019},
  publisher={ACM New York, NY, USA}
}

@article{zhang2023evaluating,
  title={Evaluating the Surrogate Index as a Decision-Making Tool Using 200 A/B Tests at Netflix},
  author={Zhang, Vickie and Zhao, Michael and Le, Anh and Kallus, Nathan},
  journal={arXiv preprint arXiv:2311.11922},
  year={2023}
}

@article{tran2023inferring,
  title={Inferring the Long-Term Causal Effects of Long-Term Treatments from Short-Term Experiments},
  author={Tran, Allen and Bibaut, Aur{\'e}lien and Kallus, Nathan},
  journal={arXiv preprint arXiv:2311.08527},
  year={2023}
}

@article{lauritzen2004discussion,
  title={Discussion on causality [with reply]},
  author={Lauritzen, Steffen L and Aalen, Odd O and Rubin, Donald B and Arjas, Elja},
  journal={Scandinavian Journal of Statistics},
  volume={31},
  number={2},
  pages={189--201},
  year={2004},
  publisher={JSTOR}
}

@article{frangakis2002principal,
  title={Principal stratification in causal inference},
  author={Frangakis, Constantine E and Rubin, Donald B},
  journal={Biometrics},
  volume={58},
  number={1},
  pages={21--29},
  year={2002},
  publisher={Oxford University Press}
}

@article{chen2007criteria,
  title={Criteria for surrogate end points},
  author={Chen, Hua and Geng, Zhi and Jia, Jinzhu},
  journal={Journal of the Royal Statistical Society Series B: Statistical Methodology},
  volume={69},
  number={5},
  pages={919--932},
  year={2007},
  publisher={Oxford University Press}
}

@article{vanderweele2013surrogate,
  title={Surrogate measures and consistent surrogates},
  author={VanderWeele, Tyler J},
  journal={Biometrics},
  volume={69},
  number={3},
  pages={561--565},
  year={2013},
  publisher={Wiley Online Library}
}

@book{tsiatis2006semiparametric,
  title={Semiparametric theory and missing data},
  author={Tsiatis, Anastasios A},
  year={2006},
  publisher={Springer}
}

@book{bickel1993efficient,
  title={Efficient and adaptive estimation for semiparametric models},
  author={Bickel, Peter J and Klaassen, Chris AJ and Bickel, Peter J and Ritov, Ya’acov and Klaassen, J and Wellner, Jon A and Ritov, YA'Acov},
  volume={4},
  year={1993},
  publisher={Springer}
}

@article{kennedy2024semiparametric,
  title={Semiparametric doubly robust targeted double machine learning: a review},
  author={Kennedy, Edward H},
  journal={Handbook of statistical methods for precision medicine},
  pages={207--236},
  year={2024},
  publisher={Chapman and Hall/CRC}
}

@book{hammersley2013monte,
  title={Monte carlo methods},
  author={Hammersley, John},
  year={2013},
  publisher={Springer Science \& Business Media}
}

@article{zellner1962efficient,
  title={An efficient method of estimating seemingly unrelated regressions and tests for aggregation bias},
  author={Zellner, Arnold},
  journal={Journal of the American statistical Association},
  volume={57},
  number={298},
  pages={348--368},
  year={1962},
  publisher={Taylor \& Francis}
}

@article{schmidt1977estimation,
  title={Estimation of seemingly unrelated regressions with unequal numbers of observations},
  author={Schmidt, Peter},
  journal={Journal of Econometrics},
  volume={5},
  number={3},
  pages={365--377},
  year={1977},
  publisher={Elsevier}
}

@article{zellner1963estimators,
  title={Estimators for seemingly unrelated regression equations: Some exact finite sample results},
  author={Zellner, Arnold},
  journal={Journal of the American statistical Association},
  volume={58},
  number={304},
  pages={977--992},
  year={1963},
  publisher={Taylor \& Francis}
}

@article{conniffe1985estimating,
  title={Estimating regression equations with common explanatory variables but unequal numbers of observations},
  author={Conniffe, Denis},
  journal={Journal of Econometrics},
  volume={27},
  number={2},
  pages={179--196},
  year={1985},
  publisher={Elsevier}
}

@book{asmussen2007stochastic,
  title={Stochastic simulation: algorithms and analysis},
  author={Asmussen, S{\o}ren and Glynn, Peter W},
  volume={57},
  year={2007},
  publisher={Springer}
}

@article{egami2024using,
  title={Using large language model annotations for the social sciences: A general framework of using predicted variables in downstream analyses},
  author={Egami, Naoki and Hinck, Musashi and Stewart, Brandon M and Wei, Hanying},
  journal={Preprint from November},
  volume={17},
  pages={2024},
  year={2024}
}

@article{hummel2014fundamental,
  title={Fundamental models for forecasting elections at the state level},
  author={Hummel, Patrick and Rothschild, David},
  journal={Electoral Studies},
  volume={35},
  pages={123--139},
  year={2014},
  publisher={Elsevier}
}

@article{donnini2024election,
  title={Election night forecasting with DDHQ: A real-time predictive framework},
  author={Donnini, Zachary and Louit, Sydney and Wilcox, Shelby and Ram, Mukul and McCaul, Patrick and Frank, Arianwyn and Rigby, Matt and Gowins, Max and Tranter, Scott},
  journal={Harvard Data Science Review},
  volume={6},
  number={4},
  year={2024},
  publisher={The MIT Press}
}

@article{du2024labor,
  title={LABOR-LLM: Language-Based Occupational Representations with Large Language Models},
  author={Du, Tianyu and Kanodia, Ayush and Brunborg, Herman and Vafa, Keyon and Athey, Susan},
  journal={arXiv preprint arXiv:2406.17972},
  year={2024}
}

@article{vafa2022career,
  title={CAREER: A Foundation Model for Labor Sequence Data},
  author={Vafa, Keyon and Palikot, Emil and Du, Tianyu and Kanodia, Ayush and Athey, Susan and Blei, David M},
  journal={arXiv preprint arXiv:2202.08370},
  year={2022}
}

@article{tang2012efficient,
  title={An efficient empirical likelihood approach for estimating equations with missing data},
  author={Tang, Cheng Yong and Qin, Yongsong},
  journal={Biometrika},
  volume={99},
  number={4},
  pages={1001--1007},
  year={2012},
  publisher={Oxford University Press}
}

@article{gronsbell2024another,
  title={Another look at inference after prediction},
  author={Gronsbell, Jessica and Gao, Jianhui and Shi, Yaqi and McCaw, Zachary R and Cheng, David},
  journal={arXiv preprint arXiv:2411.19908},
  year={2024}
}

@misc{chernozhukov2018double,
  title={Double/debiased machine learning for treatment and structural parameters},
  author={Chernozhukov, Victor and Chetverikov, Denis and Demirer, Mert and Duflo, Esther and Hansen, Christian and Newey, Whitney and Robins, James},
  year={2018},
  publisher={Oxford University Press Oxford, UK}
}

@article{zrnic2024note,
  title={A note on the prediction-powered bootstrap},
  author={Zrnic, Tijana},
  journal={arXiv preprint arXiv:2405.18379},
  year={2024}
}

@article{kluger2025prediction,
  title={Prediction-powered inference with imputed covariates and nonuniform sampling},
  author={Kluger, Dan M and Lu, Kerri and Zrnic, Tijana and Wang, Sherrie and Bates, Stephen},
  journal={arXiv preprint arXiv:2501.18577},
  year={2025}
}

@article{song2024general,
  title={A general m-estimation theory in semi-supervised framework},
  author={Song, Shanshan and Lin, Yuanyuan and Zhou, Yong},
  journal={Journal of the American Statistical Association},
  volume={119},
  number={546},
  pages={1065--1075},
  year={2024},
  publisher={Taylor \& Francis}
}

@article{wei2024measuring,
  title={Measuring short-form factuality in large language models},
  author={Wei, Jason and Karina, Nguyen and Chung, Hyung Won and Jiao, Yunxin Joy and Papay, Spencer and Glaese, Amelia and Schulman, John and Fedus, William},
  journal={arXiv preprint arXiv:2411.04368},
  year={2024}
}

@article{xiong2023can,
  title={Can llms express their uncertainty? an empirical evaluation of confidence elicitation in llms},
  author={Xiong, Miao and Hu, Zhiyuan and Lu, Xinyang and Li, Yifei and Fu, Jie and He, Junxian and Hooi, Bryan},
  journal={arXiv preprint arXiv:2306.13063},
  year={2023}
}

@inproceedings{danescu2013computational,
  title={A Computational Approach to Politeness with Application to Social Factors},
  author={Danescu-Niculescu-Mizil, Cristian and Sudhof, Moritz and Jurafsky, Dan and Leskovec, Jure and Potts, Christopher},
  booktitle={51st Annual Meeting of the Association for Computational Linguistics},
  pages={250--259},
  year={2013},
  organization={ACL}
}

@inproceedings{fischstratified,
  title={Stratified Prediction-Powered Inference for Effective Hybrid Evaluation of Language Models},
  author={Fisch, Adam and Maynez, Joshua and Hofer, R Alex and Dhingra, Bhuwan and Globerson, Amir and Cohen, William W},
  booktitle={The Thirty-eighth Annual Conference on Neural Information Processing Systems},
year={2024}
}

@article{hofer2024bayesian,
  title={Bayesian Prediction-Powered Inference},
  author={Hofer, R Alex and Maynez, Joshua and Dhingra, Bhuwan and Fisch, Adam and Globerson, Amir and Cohen, William W},
  journal={arXiv preprint arXiv:2405.06034},
  year={2024}
}
\bibliographystyle{plainnat}

\appendix

\section{Relationship with broader literature}
{

Besides the line of work on surrogate outcomes, prediction-powered inference has a deep connection with broader literature. We will briefly review these connections below. 

~\\
~\noindent \textbf{Post-Prediction Inference}: The post-prediction inference method \citep{wang2020methods} is the first modern perspective on the problem of inference with prediction. In the same setting as the PPI (Section \ref{sec: surrogate and PPI}), they propose a procedure that first estimates the relationship between observed and predicted outcomes on a validation sample, and uses that relationship to correct downstream
inference through bootstrap or analytical formula. Their method is successfully applied to many scientific problems, such as gene expression and the cause of death. However, their procedure may not preserve the statistical validity of the inference in the general settings. Thus, the PPI methods \cite{angelopoulos2023prediction, angelopoulos2023ppi++} are proposed as an always-valid solution. 

~\\
~\noindent \textbf{Control Variates.}
The variance-reduction mechanism in PPI is also a direct instance of the classical control-variate principle
\citep[e.g.][]{hammersley2013monte, asmussen2007stochastic} from Monte-Carlo simulation. In its simplest form, if one wants
to estimate \(\mathbb E[U]\) and has access to another variable \(H\) with known or accurately estimable
mean, then $U-M\{H-\mathbb E(H)\}$ has the same expectation as \(U\), and the optimal linear coefficient is
\begin{equation}
\label{eqn: control variates coefficient}
M^\star=\operatorname{Cov}(U,H)\operatorname{Var}(H)^{-1}.    
\end{equation}

The loss function of PPI \eqref{eqn: PPI} applies a similar idea for M-estimation. The target here is to minimize the population loss function $\E[\ell_\theta(X, Y)]$, which can be approximated by the empirical loss function $\frac{1}{n}\sum_{i=1}^n\ell_\theta(X_i, Y_i)$. For any function $g_\theta(X,\hat Y)$, the additional unlabeled data allow us to estimate its mean much more accurately than the target loss function $\ell_\theta(X, Y)$. Therefore, $g_\theta(X,\hat Y)$ can be used as a control variates for $\ell_\theta(X, Y)$, and the true mean $\E[g_\theta(X,\hat Y)]$ can be replaced with the empirical estimate $\frac{1}{N}\sum_{i=n+1}^{n+N}g_{\theta}(X_i, \hat Y_i)$, which recovers the PPI loss function \eqref{eqn: PPI}. The optimal coefficient \ref{eqn: control variates coefficient} is exactly the same as the power tuning technique used in PPI++ \citep{angelopoulos2023ppi++} and the correction terms in many works \citep{chen2000unified, gan2024prediction, gronsbell2024another}. We will also use its matrix version in our methodology. 

~\\
~\noindent\textbf{Semi-Parametric Inference.}
The efficient score \eqref{eqn:optimal score} has a standard influence-function projection interpretation. For the
complete-data estimating equation $\mathbb E\{U_{\theta}(X,Y)\}=0,$ define $H_{\theta} = \E[\nabla_\theta U_\theta(X,Y)]$, 
then the complete-data influence function is
\[
    \varphi(X,Y)
    =
    -H_{\theta^\star}^{-1}U_{\theta^\star}(X,Y).
\]
Since only \((X,\hat Y)\) is observed for the unlabeled sample, the relevant projected
component is
\[
    \mathbb E\{\varphi(X,Y)\mid X,\hat Y\}
    =
    -H_{\theta^\star}^{-1}
    \mathbb E\{U_{\theta^\star}(X,Y)\mid X,\hat Y\}.
\]
Thus the optimal score
$
    s^\star(X,\hat Y)=\mathbb E\{U_{\theta^\star}(X,Y)\mid X,\hat Y\}
$
is the score-level projection of the complete-data influence function onto observed information. Under the equivalent missing-label formulation with labeling probability \(p=n/(n+N)\), the efficient observed-data influence function in the semi-parametric inference literature \citep{robins1994estimation,tsiatis2006semiparametric,bickel1993efficient, chen2005measurement, chen2008semiparametric, kennedy2024semiparametric} is given by 
\[
    \varphi_{\mathrm{EIF}}(X,\hat Y, Y, D)
    =
    -H_{\theta^\star}^{-1}
    \left[
        s^\star(X,\hat Y)
        +
        \frac{D}{p}
        \{U_{\theta^\star}(X,Y)-s^\star(X,\hat Y)\}
    \right].
\]
The efficient influence function
is Neyman orthogonal to first-order perturbations of the nuisance function \(s^\star\): if we define $\psi_s(X,\hat Y,Y,D) = s(X,\hat Y)
        +
        \frac{D}{p}
        \{U_{\theta^\star}(X,Y)-s(X,\hat Y)\}$ for any measurable function $s$, then it is straightforward to see that
$$
\frac{d}{dt} \E[\psi_{s^\star + tq}(X,\hat Y,Y,D)]\Big|_{t=0} = \E\left[q(X,\hat Y)-\frac{D}{p}q(X,\hat Y)\right] = 0
$$
for any measurable perturbation $q$. Therefore, in our RePPI estimator, we use the cross-fitting technique \citep{chernozhukov2018double} to enable the use of any flexible machine-learning nuisance estimators without introducing extra bias.

~\\
~\noindent \textbf{Semi-Supervised Inference.}
In semi-supervised inference problems, a
small labeled sample $\{X_i, Y_i\}_{i=1}^n$ is accompanied by a larger sample of always-observed covariates $\{X_i\}_{i=n+1}^{n+N}$. Existing methods exploit the unlabeled
covariate distribution in different ways. For mean estimation, \citet{zhang2019semi} use least-squares and nonparametric regression ideas to estimate the component of \(Y\) predictable from \(X\), and then average this fitted component using the larger covariate sample. For linear regression, \citet{chakrabortty2018efficient} propose an imputation-and-refitting procedure with cross-validation that is designed
to improve efficiency while remaining adaptive to model misspecification. For general M-estimation, \citet{song2024general} formulate the use of unlabeled data
more explicitly as a projection problem and derive optimal weights for combining labeled and unlabeled information. More recently, \citet{testa2025semiparametric}
extends a semiparametric missing-at-random framework with distribution shift and
decaying overlap. The PPI framework can be viewed as a special case of semi-supervised inference, where the covariates $X_i$ are supplemented with a prediction $\hat Y$ that can be regarded as a surrogate of the true label $Y$. 

As a concurrent work, \citet{xu2025unified} provides a broad semiparametric efficiency theory for semi-supervised
inference. When the auxiliary information is taken to be \((X,\hat Y)\), the efficient estimator of \citet{xu2025unified} targets the same optimal imputed score as RePPI. Notably, their efficient estimator uses a basis expansion approach to fit the optimal score and requires the eigenvalue of the empirical covariance matrix of basis functions to be lower bounded, which could suffer from the curse of dimensionality and computational instability in practice (detailed empirical comparison in Appendix \ref{sec:other-comparison}). Our contribution over \citet{xu2025unified} is to provide a practical computational strategy that enables learning the optimal imputed score with a general class of learners. 

~\\
~\noindent \textbf{Design-Based Supervised Learning.}
The PPI framework is closely related to design-based supervised learning (DSL)
\citep{egami2023using,egami2024using,rister2025correcting}, which studies downstream
inference when text, image, or other AI-generated annotations are available for all units and
gold-standard expert labels are collected for a designed subset. DSL constructs a bias-corrected
pseudo-outcome
\[
    \widetilde Y
    =
    \widehat s(X,\hat Y)
    +
    \frac{D}{\pi(X,\hat Y)}\{Y-\widehat s(X,\hat Y)\},
\]
where \(\pi(X,\hat Y)\) is the known labeling probability. Downstream moments are then evaluated with
\(\widetilde Y\) in place of \(Y\). In the special case of constant labeling probability, no distribution
shift, and a generalized linear model, this estimator is algebraically equivalent to the classical AIPW
surrogate-outcome estimator \eqref{eqn: AIPW} and therefore has the same population recalibration target as \eqref{eqn:optimal score}. However, the emphasis of the two approaches is different. DSL focuses on design-based validity for downstream social-science analyses under known, possibly stratified expert-labeling probabilities,
and it guarantees unbiasedness and valid uncertainty quantification even when the AI annotations
are arbitrarily biased. Our method instead focuses on achieving efficiency for general
M-estimation. Its recalibration target \eqref{eqn:optimal score} need not reduce to an outcome-level conditional mean for nonlinear losses. In addition,
RePPI combines flexible cross-fitted score learning with matrix power tuning, so that imperfect
recalibration does not lead to first-order efficiency loss relative to labeled-only inference, and its
linearized imputed-loss formulation preserves convexity of the downstream empirical risk
minimization problem.

~\\
~\noindent \textbf{Seemingly Unrelated Regression.}
There is also a useful analogy between PPI and seemingly unrelated regression (SUR) \citep{zellner1962efficient,zellner1963estimators, schmidt1977estimation}. In SUR, several regression equations are estimated jointly, and generalized least squares improves efficiency when the errors across equations are correlated. In the context of PPI, the SUR problem can be stated as 
$$
Y = X^\top \beta^\star + \epsilon, \quad \hat Y = X^\top \gamma^\star + \epsilon',
$$
where $\epsilon$ and $\epsilon'$ are correlated errors such that $\E[\eps|X] = 0,\E[\eps'|X] = 0,\Cov(\eps,\eps'|X) \neq 0 $. Therefore, it is possible to utilize the information of $\hat Y$ to reduce the uncertainty in the estimation of $\beta^\star$. For example, the estimator from \cite{conniffe1985estimating} can be written as 
$$
\hat \beta^{\textrm{SUR}} = \hat \beta^{\textrm{lab}} - \hat M(\hat \gamma^{\textrm{lab}}-\hat \gamma^{\textrm{all}}),
$$
where $\hat \beta^{\textrm{lab}}$ is the OLS coefficient of $Y$ on $X$ using labeled data, $\hat \gamma^{\textrm{lab}}$ is the OLS coefficient of $\hat Y$ on $X$ using labeled data, $\hat \gamma^{\textrm{all}}$ is the OLS coefficient of $\hat Y$ on $X$ using all labeled and unlabeled data, and $\hat M$ is a tuning matrix. The formulation of the estimator is similar to \eqref{eqn: PPI estimator correction}, while the choice of $\hat M$ can be different since SUR assumes an underlying model. Similar ideas have also been applied recently to genome-wide association studies, such as \citet{mccaw2024synthetic} and \citet{miao2024valid}, under different modeling assumptions.
We refer interested readers to \cite{gronsbell2024another} for a more detailed discussion. 

~\\
~\noindent \textbf{Causal Inference with Multiple Data Sources.} 
\cite{yang2019combining} considers the problem of average treatment effect estimation using two data sources: one that includes all confounders and another that omits a subset of confounders. Specifically, the first dataset includes $(A_i, X_i, U_i, Y_i)_{i=1}^{n}$, and the second dataset includes $(A_i, X_i, Y_i)_{i=n+1}^{n+N}$. In this setting, $A_i$ is a binary treatment, $Y_i$ is the outcome variable with $Y_i = Y_i(1)A_i + Y_i(0)(1 - A_i)$, where $(Y_i(1), Y_i(0))$ is the pair of potential outcomes, and $(X_i, U_i)$ is the set of confounders satisfying the strong ignorability assumption: $A_i \perp\!\!\!\perp (Y_i(1), Y_i(0))\mid X_i, U_i$. Similarly as in the PPI setting, they assume the distribution of $(A_i, X_i, Y_i)$ is identical in the two datasets. Under their assumptions, the average treatment effect $\theta^\star = \E[Y(1) - Y(0)]$ can be identified through an estimating equation $m_\theta(A, X, U, Y)$. For example, we can set
\begin{equation}\label{eq:IPW}
m_{(\theta, \eta)}(A, X, U, Y) = \frac{AY}{e_\eta(X, U)} - \frac{(1 - A)Y}{1 - e_\eta(X, U)} - \theta, \end{equation}
where $e_\eta(X, U)$ is a correctly specified parametric model for the propensity score $\P(A = 1\mid X, U)$. \cite{yang2019combining} essentially apply the estimator of \cite{chen2000unified} with $h_\gamma$ being another estimating equation $\td{m}(A, X, Y)$, e.g., 
\begin{equation}\label{eq:IPW_tdm}
\td{m}_{(\theta, \td{\eta})}(A, X, Y) = \frac{AY}{\td{e}_{\td{\eta}}(X)} - \frac{(1 - A)Y}{1 - \td{e}_{\td{\eta}}(X)} - \theta,
\end{equation}
where $\td{e}_{\td{\eta}}(X)$ is an error-prone propensity score model. However, they do not study the optimal choice of $\td{m}$. Similarly, the optimal $\td{m}$ in this setting is given by 
$\E[m_{(\theta, \eta)}(A, X, U, Y)\mid A, X, Y].$
With the choice \eqref{eq:IPW}, this reduces to
\begin{equation}\label{eq:IPW_tdm_optimal}
\td{m}_{(\theta, \eta)}(A, X, Y) = AY \E\left[\frac{1}{e_\eta(X, U)}\mid A,X,Y\right] - (1 - A)Y\E\left[\frac{1}{1 - e_\eta(X, U)}\mid A,X,Y\right] - \theta.
\end{equation}
This is different from \eqref{eq:IPW_tdm}, implying that the estimator of \cite{yang2019combining} is inefficient in general and can be improved using \eqref{eq:IPW_tdm_optimal} instead of \eqref{eq:IPW_tdm}.

}
{

\section{The role of cross-fitting}
\label{sec: cross-fitting}

In existing surrogate outcome literature \citep{robins1994estimation, chen2000unified, chen2008improving}, although the semi-parametric efficiency is established when using the optimal score function  $s^\star(X,\hat Y) = \E[\nabla\ell_{\theta^\star}(X, Y)|X, \hat Y]$, the computation strategy for achieving the efficiency is generally unclear. In practice, $s^\star(X,\hat Y)$ is hardly known in prior and needs to be fitted using samples. In our methodology (Algorithm \ref{alg:crossfit}), we use the cross-fitting strategy to divide the samples for estimating the optimal score and making inference, and we prove in Theorem \ref{thm:main} that this strategy can recover the semi-parametric efficiency when the optimal score function is consistently estimated.
In regular settings where the fitted score belongs to a fixed Donsker class (such as the fixed-dimensional parametric models advocated by \cite{chen2000unified, chen2008improving}), classical empirical process theory implies that using the same sample to fit the score and make inference is valid. However, in applications with modern machine learning models, the optimal imputed loss is often complicated and can not be consistently learned through the P-donsker class, such as a fixed-dimensional parametric model. This aligns with the benefits of cross-fitting advocated in \cite{chernozhukov2018double}: it relaxes the strong Donsker class requirements and enabled the usage of flexible learners.

We now make this distinction explicit through a theoretical analysis. For any function $a(X,Y,\hat Y)$, write
\[
    \P_n a = \frac1n\sum_{i=1}^n a(X_i,Y_i,\hat Y_i),
\]
and for any function $b(X,\hat Y)$, write
\[
    \P_nb = \frac1n\sum_{i=1}^n b(X_i,\hat Y_i),
    \quad
    \P_Nb = \frac1N\sum_{i=n+1}^{n+N} b(X_i,\hat Y_i),
    \quad
    \bar \P_{n,N}b
    =
    \frac{n}{n+N}\P_nb
    +
    \frac{N}{n+N}\P_Nb.
\]
For centered empirical processes, we write:
\[
    \G_n a
    =
    \frac1{\sqrt n}\sum_{i=1}^n\{a(X_i,Y_i,\hat Y_i)-\E a(X,Y,\hat Y)\},
\]
\[
    \hat\G_n b
    =
    \frac1{\sqrt n}\sum_{i=1}^n\{b(X_i,\hat Y_i)-\E b(X,\hat Y)\},
    \qquad
    \hat\G_N b
    =
    \frac1{\sqrt N}\sum_{i=n+1}^{n+N}\{b(X_i,\hat Y_i)-\E b(X,\hat Y)\}.
\]

First, consider the fixed-score version of the RePPI estimating equation. For a possibly misspecified score limit $s(X,\hat Y)$ and a fixed matrix $M$, the linearized imputed loss is
\[
    g_\theta(X,\hat Y)=\frac{N}{n+N}\theta^\top M s(X,\hat Y),
\]
so that the corresponding estimating equation is
\begin{equation}
\label{eq:same_section_fixed_score_moment}
    0
    =
    \P_n\nabla\ell_\theta
    +
    \frac{N}{n+N} M\{\P_Ns-\P_ns\}.
\end{equation}
If $s$ is fixed, then the standard Taylor expansion argument (see proof of Theorem~\ref{thm: efficient PPI} in Appendix \ref{sec: proof REPPI} for details) gives
\begin{equation}
\label{eq:same_section_fixed_score_expansion}
    \sqrt n(\hat\theta^\ppi-\theta^\star)
    =
    -H_{\theta^\star}^{-1}
    \left[
        \G_n\nabla\ell_{\theta^\star}
        +
        \frac{N}{n+N} M\left\{\sqrt{\frac nN}\hat\G_Ns-\hat\G_ns\right\}
    \right]
    +o_p(1).
\end{equation}
When $M=\Cov(\nabla\ell_{\theta^\star},s)\Cov(s)^{-1}$, the variance in \eqref{eq:same_section_fixed_score_expansion} is exactly $\Sigma_s^\REPPI$ in Theorem~\ref{thm:main}.

The issue is what happens if the same labeled observations are used both to fit the score and to evaluate the labeled correction term. Let $\hat s_\theta^\all$ denote a score learner trained on all labeled observations, and let $\hat s_{\theta^\star}^\all$ be its value at $\theta^\star$. A same-sample linearized implementation uses the moment
\begin{equation}
\label{eq:same_section_same_score_moment}
    0
    =
    \P_n\nabla\ell_\theta
    +
    \frac{N}{n+N} M\{\P_N\hat s_\theta^\all-\P_n\hat s_\theta^\all\}.
\end{equation}
This formulation also captures the local behavior of a same-sample flexible extension of the CC-estimator \eqref{eqn: CC estimator}. Indeed, if a whole fitted score family $\hat s_\theta^\all$ is used in the CC construction, define
\[
    \hat\theta_{\hat s}^\cls
    \in
    \left\{\theta:\P_n\hat s_\theta^\all=0\right\},
    \qquad
    \hat\theta_{\hat s}^\all
    \in
    \left\{\theta:\bar \P_{n,N}\hat s_\theta^\all=0\right\},
\]
and
\[
    \hat\theta_{\hat s}^\cc
    =
    \hat\theta^\cls
    -
    \hat M_s(\hat\theta_{\hat s}^\cls-\hat\theta_{\hat s}^\all).
\]
A Taylor expansion of the two-score equations shows that the same empirical-process term derived below also enters this CC implementation.

\begin{proposition}[Same-sample score learning and the extra term]
\label{prop:same_sample_extra_term}
Let $\hat s_{\theta^\star}^\all$ be trained on the labeled sample and suppose that,
\[
    \E\left[
        \|\hat s_{\theta^\star}^\all(X,\hat Y)-s(X,\hat Y)\|^2
    \right]
    =o_p(1).
\]
Suppose further that the same-sample empirical process has the expansion
\begin{equation}
\label{eq:same_section_leakage_expansion}
    R_{n,s}^\all
    :=
    \G_n
    \{\hat s_{\theta^\star}^\all-s\}
    =
    b_n+\frac1{\sqrt n}\sum_{i=1}^n\zeta(X_i,Y_i,\hat Y_i)+o_p(1),
\end{equation}
where $b_n\to b$ and $\zeta$ has mean zero and finite second moment. Assume also that $\hat M\to_pM$ and Assumption \ref{asm: regularity} and \ref{asm: smooth gradient} hold. Then the solution $\hat\theta^\mathrm{Same}$ to \eqref{eq:same_section_same_score_moment} satisfies
\begin{align}
\label{eq:same_section_same_expansion}
    \sqrt n(\hat \theta^\mathrm{Same}-\theta^\star)
    =
    -H_{\theta^\star}^{-1}
    \Bigg[
        \G_n\nabla\ell_{\theta^\star}
        +
        \frac{N}{n+N} M\lb \left\{\sqrt{\frac nN}\hat\G_Ns-\hat\G_ns\right\}
        -
        b
        -
        \frac1{\sqrt n}\sum_{i=1}^n\zeta(X_i,Y_i,\hat Y_i)\rb
    \Bigg]
    +o_p(1).
\end{align}
\end{proposition}

\begin{proof}
Set $\Delta_n=\hat s_{\theta^\star}^\all-s$. Expanding \eqref{eq:same_section_same_score_moment} at $\theta^\star$ gives
\begin{align*}
    &\sqrt n\left[
        \P_n\nabla\ell_{\theta^\star}
        +
        \frac{N}{n+N}M\{\P_N\hat s_{\theta^\star}^\all-\P_n\hat s_{\theta^\star}^\all\}
    \right] \\
    &\quad =
    \G_n\nabla\ell_{\theta^\star}
    +
    \frac{N}{n+N} M\left\{\sqrt{\frac nN}\hat\G_Ns-\hat\G_ns\right\}
    +
    \frac{N}{n+N}M\sqrt n(\P_N-\P)\Delta_n
    -
    \frac{N}{n+N}M R_{n,s}^\all
    +o_p(1).
\end{align*}
The unlabeled sample is independent of $\Delta_n$ conditional on the fitted score. Therefore,
\[
    \E\left[
        \left\|\sqrt n(\P_N-\P)\Delta_n\right\|^2
    \right]
    =
    \frac nN
    \|\Delta_n\|_{L_2(\P)}^2
    =o_p(1).
\]
Substituting \eqref{eq:same_section_leakage_expansion} gives the expansion of the moment at $\theta^\star$. A Taylor expansion of the moment in $\theta$ and nonsingularity of $H_{\theta^\star}$ yield \eqref{eq:same_section_same_expansion}.
\end{proof}

Proposition~\ref{prop:same_sample_extra_term} separates the regular and flexible regimes. If $\hat s_{\theta^\star}^\all$ lies with probability tending to one in a fixed $P$-Donsker class and $\|\hat s_{\theta^\star}^\all-s\|_{L_2(P)}\to_p0$, then stochastic equicontinuity gives $R_{n,s}^\all=o_p(1)$. Hence same-sample fitting and cross-fitting are first-order equivalent in regular parametric or Donsker settings. For flexible learners, however, $R_{n,s}^\all$ need not vanish, which may introduce extra bias and variance. Cross-fitting removes this term by construction, because each labeled observation is evaluated by a score learner trained without that observation.

Below, we provide a concrete example to show the failure of the same-sample fitting strategy when a flexible score learner is used.  
\begin{example}[A simple example for the failure of same sample score estimates]
\label{ex:tree-leakage}
Consider scalar mean estimation with loss
\[
    \ell_\theta(X,Y)=\frac{1}{2}(Y-\theta)^2 .
\]
Let 
\(\hat Y,\epsilon,X\) be independent $N(0,1)$, and $Y=\hat Y+\epsilon$. 
Then \(\theta^\star=\E Y=0\), and the oracle imputed score is
\[
    s_\theta(X,\hat Y)
    =
    \E\{\nabla_\theta \ell_{\theta}(X,Y)\mid X,\hat Y\}
    =
    \theta-\hat Y.
\]

Let \(\mathcal D_n=\{(X_i,Y_i,\hat Y_i)\}_{i=1}^n\) denote the labeled
sample. Given \(\mathcal D_n\), define the same-sample
interpolating score learner
\[
    \hat s_n(X,\hat Y)
    =
    \begin{cases}
        \theta - Y_i,
        &\text{if } X=X_i \text{ for some } i\in\{1,\ldots,n\},\\
        \theta -\hat Y + T_n(X,\hat Y),
        &\text{otherwise.}
    \end{cases}
\]
Since \(X\) has a continuous distribution, \(X_1,\ldots,X_n\) are distinct
almost surely, so the above definition is unambiguous almost surely. Here, the learner $s_n(X, \hat Y)$ interpolates the training samples to represent the common overfitting phenomenon in modern machine learning. On test samples, we assume it learns the optimal score with some error $T_n(X,\hat Y)$. If $\E[\|T_n(X,\hat Y)\|^2]\rightarrow 0$, then the fitted score is perfectly out-of-sample consistent:
\begin{equation}
\label{eqn: out of sample consistency}
    \E\left[
        \lb\hat s_n(X,\hat Y)-s(X,\hat Y)\rb^2
    \right]
    \rightarrow0
    \qquad\text{a.s.}
\end{equation}
Thus, Theorem \ref{thm:main} implies that the RePPI estimator achieves the optimal efficiency. On the contrary, the PPI estimator \eqref{eqn: PPI} (or the CC estimator \eqref{eqn: CC estimator} from \cite{chen2000unified}) using the same sample fitting strategy with score $\hat s_n$ leads to the estimator 
\begin{equation}
    \hat \theta^{\mathrm{Same}} = \frac{1}{n}\sum_{i=1}^n Y_i - \hat M \lb \frac{1}{n}\sum_{i=1}^n Y_i - \frac{1}{n+N}\lb\sum_{i=1}^{n}Y_i+\sum_{i=n+1}^{n+N} (\hat Y_i+T_n(X_i,\hat Y_i)))\rb\rb,
\end{equation}
\end{example}
where $\hat M = \widehat\Cov(\theta - Y, \hat s_n(X, \hat Y)) / \widehat\Var(\hat s_n(X, \hat Y))$. There are several problems with this approach: 
\begin{enumerate}
    \item The excessive term $R_{n,s}^\all=-\frac1{\sqrt n}\sum_{i=1}^n\epsilon_i-\sqrt n\,\E\{T_n(X,\hat Y)\}$ in \eqref{eq:same_section_leakage_expansion} can bring additional error in asymptotics. For example, if $T_n(X, \hat Y) = \frac{1}{\sqrt{n}}$, the model is biased, then the L2 consistency condition \eqref{eqn: out of sample consistency} still holds, $M=q$, and we have 
    \begin{equation}
        \hat \theta^{\mathrm{Same}} = \frac{1}{n+N}\lb\sum_{i=1}^{n}Y_i+\sum_{i=n+1}^{n+N} (\hat Y_i+T_n(X_i,\hat Y_i)))\rb,
    \end{equation}
    and thus $\sqrt{n} (\hat \theta^{\mathrm{Same}} - \theta^\star ) \sim \N(\frac{1}{1+r}, (1+\frac{r}{1+r})\frac{r}{1+r})$, it is biased in the root-n regime. 
    \item Even if we assume the asymptotic distribution is still correct (the same as Theorem \ref{thm:main}). Without sample splitting, it is impossible to consistently estimate the asymptotic variance \eqref{eq:Sigma_REPPI}, which involves the covariance of $Y$ and $\hat s_n(X,\hat Y)$ and thus has a different value on labeled data and the true population. Therefore, the final confidence intervals can be invalid. 
\end{enumerate}

Example~\ref{ex:tree-leakage} is deliberately simple, but it captures the practical failure mode: a highly adaptive learner can behave almost correctly on new observations while memorizing training data from labeled observations. Cross-fitting rules out this own-observation contribution by construction. Hence, the cross-fitting procedure in RePPI is not to improve over regular parametric CC fitting; it is to make flexible recalibration of the optimal imputed-loss gradient compatible with the fixed-score variance theory under weak out-of-sample convergence conditions.

In addition, we illustrate this idea using a straightforward simulation example of the mean estimation,  where the target is $\theta^\star=\E(Y)=0$ and
\[
    \hat Y\sim \mathrm{Unif}[0,1],\qquad
    Y=f(\hat Y)+\epsilon,\qquad
    f(x)=\sin(2\pi x),
    \qquad
    \epsilon\sim \N(0,1).
\]
The optimal score is $s^\star(\hat Y)=\theta - f(\hat Y)$.  We use two different learners to learn $m$:
\begin{enumerate}
    \item Parametric learner: A ridge regression using degree-5 polynomial features of $\hat Y$. 
    \item Flexible learner: A decision tree model using $\hat Y$. 
\end{enumerate}
Since the flexible learner has a much larger model capacity, it is very likely to overfit to the noise and perform differently on the training data (labeled data) and testing data (unlabeled data), which makes the inference procedure invalid. 

After we get an estimate $\hat f$, we can plug in and obtain the estimate for the optimal score $\hat s = \theta - \hat f$. 
For a fitted function $\hat f$, the estimator is
\[
    \hat\theta=\frac{1}{n}\sum_{i=1}^nY_i-\hat M\lb\frac{1}{n}\sum_{i=1}^n\hat f(\hat Y_i) - \frac{1}{n+N}\sum_{i=1}^{n+N}\hat f(\hat Y_i)\rb,
    \qquad
    \hat M=\frac{\widehat\Cov(Y,\hat f(\hat Y_i))}{\widehat\Var(\hat f(\hat Y_i))}.
\]
The no-split version uses in-sample fitted values $\hat m(\hat Y_i)$ on the labeled data; the cross-fitted version uses out-of-fold fitted values.  Table~\ref{tab:cross-fitting-simulation} reports the experiment result with $n=500$, $N=5000$, $\alpha = 0.95$, and 200 Monte Carlo repetitions. Here, the empirical s.d. is the standard deviation of $\hat \theta$ over 200 repetitions, and it represents the true uncertainty. The reported standard error is the usual fixed-score standard error that is used to construct confidence intervals
\[
    \left[\frac{1}{n}\lb\widehat\Var(Y)-\frac{1}{1+r}\frac{\Cov(Y_i,\hat f(Y_i))}{\Var(\hat f(Y_i))}\rb\right]^{1/2}.
\]
As we can see, when we use a decision tree model and the same sample fitting strategy, as the model complexity increases, the reported standard error significantly underestimates the true uncertainty, and the confidence intervals are invalid. This is because the report's standard error is computed using labeled data, which can not capture the out-of-sample uncertainty when the labeled data are used as training samples. On the contrary, the cross-fitting method can always measure the uncertainty correctly and provide a valid confidence interval regardless of the model complexity. For the low-dimensional parametric model, both methods can yield a valid confidence interval, which explains why the importance of cross-fitting is not emphasized in existing literature. 

\begin{table}[t]
\centering
\begin{tabular}{llccc}
\hline
Learner & Method & Empirical s.d. & Reported s.e. & Coverage \\
\hline
Polynomial Regression & Same-sample & 0.0449 & 0.0491 & 0.955 \\
Polynomial Regression & Cross-fit & 0.0450 & 0.0495 & 0.955 \\
\hline
Decision Tree (max depth=5) & Same-sample  & 0.0470 & 0.0422 & 0.930 \\
Decision Tree (max depth=5) & Cross-fit & 0.0481 & 0.0494 & 0.950 \\
\hline
Decision Tree (max depth=10) & Same-sample  & 0.0507 & 0.0318 & 0.790 \\
Decision Tree (max depth=10) & Cross-fit & 0.0580 & 0.0569 & 0.950 \\
\hline
Decision Tree (max depth=None) & Same-sample  & 0.0546 & 0.0165 & 0.445 \\
Decision Tree (max depth=None) & Cross-fit & 0.0614 & 0.0620 & 0.940 \\
\hline
\end{tabular}
\caption{A simulation illustrating the role of cross-fitting.  With a regular parametric model, no-split and cross-fitted estimation are first-order equivalent.  With a decision tree, the same-sample fitted values may interpolate labeled noise, leading to an incorrect inference procedure. As model complexity increases, the confidence interval's coverage decreases as overfitting becomes more severe. Cross-fitting restores approximately valid fixed-score inference.}
\label{tab:cross-fitting-simulation}
\end{table}
}
\section{Proofs from Section \ref{sec: REPPI}}
\subsection{Proof of Theorem \ref{thm: efficient PPI}}
\label{sec: proof REPPI}
We first state the regularity conditions required in Theorem \ref{thm: efficient PPI}. 
\begin{assumption}
\label{asm: regularity}
The loss function $\ell_\theta(X, Y )$ and the \rectifier $g_\theta(X, \hat Y )$  satisfy the following conditions:
\begin{enumerate}
    \item $\ell_\theta(X, Y)$ and $g_\theta(X, \hat Y)$ are convex in $\theta$ and differentiable at $\theta^\star$;
    \item $\ell_\theta(X, Y)$ and $g_\theta(X, \hat Y)$ are locally Lipschitz around $\theta^\star$, i.e., there exist a neighborhood of $\theta^\star$ such that $\ell_\theta(X, Y)$ is $M_1(x,y)$ Lipschitz and $g_\theta(X, \hat Y)$ is $M_2(x,y)$ Lipschitz in $\theta$, and $\E[M_1(X,Y)+M_2(X, \hat Y)] < \infty$;
    \item The population losses $L(\theta) = \E[\ell_\theta(X, Y)]$ have positive definite Hessian matrices at $\theta^\star$;
    \item The covariance matrices $\Cov(\nabla \ell_{\theta^\star}(X,Y))$ and $\Cov(\nabla g_{\theta^\star}(X,\hat Y))$ are both semipositive definite. 
\end{enumerate}
\end{assumption}
Assumption \ref{asm: regularity} is a standard assumption to prove asymptotic normality, and a similar assumption has been imposed in \cite{angelopoulos2023ppi++}. Below we provide a self-contained proof of Theorem~\ref{thm: efficient PPI}, though it can be derived as a corollary of \cite{robins1994estimation} with slight modification.

\begin{proof}[of Theorem \ref{thm: efficient PPI}]
    First, we prove that $\hat{\theta}_g^\ppi$ is consistent. Denote the loss function in \eqref{eqn: PPI} by
$$
\mathcal{L}_g^\ppi(\theta) := \frac{1}{n}\sum_{i=1}^{n}\ell_{\theta}(X_i, Y_i) - \lb \frac{1}{n}\sum_{i=1}^{n}g_{\theta}(X_i, \hat Y_i) - \frac{1}{N}\sum_{i=n+1}^{n+N}g_{\theta}(X_i, \hat Y_i)\rb.
$$ Since $(X_i, \hat{Y}_i)$ has the same distribution in the labeled and unlabeled data, we know $\E[\mathcal{L}_g^\ppi(\theta)] = L(\theta).$
  Using a standard covering argument (e.g., see Theorem 9.2 in \cite{keener2010theoretical}), Assumption~\ref{asm: regularity} implies that there exist $\epsilon>0$ such that 
    $$
    \sup_{\theta: \|\theta - \theta^\star\|\leq \epsilon} |\mathcal{L}_g^\ppi(\theta) - L(\theta)| \xrightarrow{p} 0.
    $$
    Since $\theta^\star$ is unique, for all $\epsilon > 0$, we know there exists a $\delta > 0$ such that $L(\theta) - L(\theta^\star) \geq \delta$ for all $\theta$ on the $\epsilon$-shell $\{\theta \mid \|\theta - \theta^\star\| = \epsilon\}$ around $\theta^\star$. With this we can write
    \begin{align*}
    &\inf_{\|\theta - \theta^\star\| = \epsilon} \mathcal{L}_g^\ppi(\theta) - \mathcal{L}_g^\ppi(\theta^\star)\\
    & = \inf_{\|\theta - \theta^\star\| = \epsilon} ((\mathcal{L}_g^\ppi(\theta) - L(\theta)) + (L(\theta) - L(\theta^\star)) + (L(\theta^\star) - \mathcal{L}_g^\ppi(\theta^\star)))\\
    & \ge \delta - 2\sup_{\theta: \|\theta - \theta^\star\|\leq \epsilon} |\mathcal{L}_g^\ppi(\theta) - L(\theta)|\\
 & \geq \delta - o_P(1).
    \end{align*}
Then, for any $\theta$ such that $\|\theta - \theta^\star\| \geq \epsilon$, let $\theta_1 = \theta^\star + \frac{\theta - \theta^\star}{\|\theta - \theta^\star\|} \epsilon$. The convexity of $\mathcal{L}_g^\ppi(\theta)$ implies

$$
\mathcal{L}_g^\ppi(\theta) - \mathcal{L}_g^\ppi(\theta^\star) \geq \frac{\|\theta - \theta^\star\|}{\|\theta_1 - \theta^\star\|}\lb\mathcal{L}_g^\ppi(\theta_1) - \mathcal{L}_g^\ppi(\theta^\star)\rb \geq \delta - o_P(1).
$$
Therefore, no $\theta$ with $\|\theta - \theta^\star\| \geq \epsilon$ can minimize $\mathcal{L}_g^\ppi$, which implies 
\[
\P(\|\hat{\theta}_g^\ppi - \theta^\star\| < \epsilon) \to 1. 
\]
Since this holds for any $\eps > 0$, we can conclude that $\hat{\theta}_g^\ppi$ is consistent: $\hat{\theta}_g^\ppi\xrightarrow{p} \theta^\star$.

    Now we turn to the asymptotic normality. For any function $h$, we use the following shorthand notation
    \begin{equation*}
        \begin{aligned}
            &\mathbb{E}_n h := \frac{1}{n} \sum_{i=1}^n h(X_i, Y_i), \quad \G_n h = \sqrt{n} (\mathbb{E}_n h - \mathbb{E}[h(X, Y)]);\\
            &\hat\E_N h := \frac{1}{N} \sum_{i=n+1}^{N+n} h(X_i,\hat Y_i), \quad \hat \G_N h := \sqrt{N} (\hat\E_N h - \mathbb{E}[h(X, \hat Y)]);\\
            &\hat\E_n h := \frac{1}{n} \sum_{i=1}^n h(X_i,\hat Y_i), \quad \hat \G_n h := \sqrt{n} (\hat\E_n h - \mathbb{E}[h(X, \hat Y)]). 
        \end{aligned}
    \end{equation*}

By Lemma 19.31 in \cite{van2000asymptotic}, Assumption \ref{asm: regularity} implies that for every (possibly random) sequence \(h_n = O_P(1)\), 
\[
\G_n \left[\sqrt{n} \left(\ell_{\theta^\star + \frac{h_n}{\sqrt{n}}} - \ell_{\theta^\star} \right) - h_n^\top \nabla \ell_{\theta^\star} \right] \xrightarrow{P} 0;
\]
\[
\hat \G_N \left[\sqrt{n} \left(g_{\theta^\star + \frac{h_n}{\sqrt{n}}} - g_{\theta^\star} \right) - h_n^\top \nabla g_{\theta^\star} \right] \xrightarrow{P} 0;
\]
\[
\hat \G_n \left[\sqrt{n} \left(g_{\theta^\star + \frac{h_n}{\sqrt{n}}} - g_{\theta^\star} \right) - h_n^\top \nabla g_{\theta^\star} \right] \xrightarrow{P} 0.
\]
Applying a second-order Taylor expansion and using the fact that $\nabla L(\theta^\star) = 0$, we obtain that
\begin{equation}
\label{eqn: expansion ell}
    \begin{aligned}
        n \mathbb{E}_n \left(\ell_{\theta^\star + \frac{h_n}{\sqrt{n}}} - \ell_{\theta^\star} \right) =& n \left(L\left(\theta^\star + \frac{h_n}{\sqrt{n}} \right) - L(\theta^\star) \right) +  h_n^\top \G_n \nabla \ell_{\theta^\star}+ o_P(1) \\
        =& \frac{1}{2} h_n^\top H_{\theta^\star} h_n + h_n^\top \G_n \nabla \ell_{\theta^\star} + o_P(1).
    \end{aligned}
\end{equation}
Similarly, we have 
\begin{equation}
\label{eqn: expansion g}
    \begin{aligned}
        n \hat \E_N \left(g_{\theta^\star + \frac{h_n}{\sqrt{n}}} - g_{\theta^\star} \right) =& n \left(G\left(\theta^\star + \frac{h_n}{\sqrt{n}} \right) - G(\theta^\star) \right) + \sqrt\frac{n}{N} h_n^\top \hat\G_N \nabla g_{\theta^\star} + o_P(1),\\
        n \hat \E_n \left(g_{\theta^\star + \frac{h_n}{\sqrt{n}}} - g_{\theta^\star} \right) =& n \left(G\left(\theta^\star + \frac{h_n}{\sqrt{n}} \right) - G(\theta^\star) \right) +  h_n^\top \hat\G_n \nabla g_{\theta^\star} + o_P(1).
    \end{aligned}
\end{equation}
Combining \eqref{eqn: expansion ell} and \eqref{eqn: expansion g}, we obtain 
\begin{align*}
&n \left( \mathcal{L}_g^\ppi\left(\theta^\star+\frac{h_n}{\sqrt{n}}\right) - \mathcal{L}_g^\ppi(\theta^\star) \right)\\
& = \frac{1}{2} h_n^\top H_{\theta^\star} h_n + h_n^\top \left( \G_n \nabla \ell_{\theta^\star} + \sqrt{\frac{n}{N}} \hat \G_N \nabla g_{\theta^\star} - \hat \G_n \nabla g_{\theta^\star} \right)  + o_P(1).
\end{align*}
We now consider two particular choices of the sequence \( h_n \). First, consider
\[
h_n^\star = \sqrt{n} (\hat \theta_g^\ppi - \theta^\star).
\]
By Corollary 5.53 of \cite{van2000asymptotic} and Assumption \ref{asm: regularity}, the consistency \( \hat \theta_g^\ppi \xrightarrow{p} \theta^\star \) implies that \( h_n^\star = O_P(1) \). This yields
\begin{equation*}
    \begin{aligned}
n \left(\mathcal{L}_g^\ppi(\hat \theta_g^\ppi) - \mathcal{L}_g^\ppi(\theta^\star) \right)
= &\frac{1}{2} h_n^{\star\top} H_{\theta^\star} h_n^\star\\ +& h_n^{\star\top} \left( \G_n \nabla \ell_{\theta^\star} +\sqrt{\frac{n}{N}} \hat \G_N \nabla g_{\theta^\star} - \hat \G_n \nabla g_{\theta^\star} \right) + o_P(1).        
    \end{aligned}
\end{equation*}

Similarly, we can choose 
\[
h_n = -H_{\theta^\star}^{-1} \left( \G_n \nabla \ell_{\theta^\star} + \sqrt{\frac{n}{N}} \hat \G_N \nabla g_{\theta^\star} - \hat \G_n \nabla g_{\theta^\star} \right),
\]
which is \( O_P(1) \) by the central limit theorem, and this gives
\[
n \left( \mathcal{L}_g^\ppi\left(\theta^\star+\frac{h_n}{\sqrt{n}}\right) - \mathcal{L}_g^\ppi(\theta^\star) \right) = -\frac{1}{2} h_n^\top H_{\theta^\star} h_n + o_P(1).
\]
Since $\hat \theta_g^\ppi$ is defined as the minimizer of $\mathcal{L}_g^\ppi$, we know that 

\begin{equation*}
    \begin{aligned}
        &n \left(\mathcal{L}_g^\ppi(\hat \theta_g^\ppi) - \mathcal{L}_g^\ppi(\theta^\star) \right) \leq n \left( \mathcal{L}_g^\ppi(\theta^\star+\frac{h_n}{\sqrt{n}}) - \mathcal{L}_g^\ppi(\theta^\star) \right)\\
        \Rightarrow\quad &  \frac{1}{2} h_n^{\star\top} H_{\theta^\star} h_n^\star - h_n^{\star\top}H_{\theta^\star} h_n\leq -\frac{1}{2} h_n^\top H_{\theta^\star} h_n + o_P(1)\\
        \Rightarrow\quad &  \frac{1}{2} (h_n^\star - h_n)^\top H_{\theta^\star} (h_n^\star - h_n) = o_P(1)\\
        \Rightarrow\quad & h_n^\star = h_n + o_P(1)\\
        \Rightarrow\quad &\sqrt{n} (\hat \theta_g^\ppi - \theta^\star) = -H_{\theta^\star}^{-1} \left( \G_n \nabla \ell_{\theta^\star} + \sqrt{\frac{n}{N}} \hat \G_N \nabla g_{\theta^\star} - \hat \G_n \nabla g_{\theta^\star} \right) + o_P(1).
    \end{aligned}
\end{equation*}
By the central limit theorem, the right-hand side is asymptotically normal:
\begin{equation*}
    \begin{aligned}
        &\G_n \nabla \ell_{\theta^\star} + \sqrt{\frac{n}{N}} \hat \G_N \nabla g_{\theta^\star} - \hat \G_n \nabla g_{\theta^\star} \\
        =&\sqrt{\frac{n}{N}} \frac{1}{\sqrt{N}} \sum_{i=n+1}^{n+N} \lb\nabla g_{\theta^\star}(X_i, \hat Y_i) -\E[\nabla g_{\theta^\star}(X, \hat Y)]\rb \\
        &+\frac{1}{\sqrt{n}} \sum_{i=1}^n \lb\nabla \ell_{\theta^\star}(X_i, Y_i) - \nabla g_{\theta^\star}(X_i, \hat Y_i) - \E[\nabla \ell_{\theta^\star}(X, Y) - \nabla g_{\theta^\star}(X, \hat Y)]\rb \\ 
        \xrightarrow{d}& ~\mathcal{N}(0, r \Cov(\nabla g_{\theta^\star}(X, \hat Y)) + \Cov(\nabla \ell_{\theta^\star}(X, Y) - \nabla g_{\theta^\star}(X, \hat Y))).
    \end{aligned}
\end{equation*}
Thus
\begin{equation*}
    \sqrt{n} (\hat \theta_g^\ppi - \theta^\star) = -H_{\theta^\star}^{-1} \left( \G_n \nabla \ell_{\theta^\star} + \sqrt{\frac{n}{N}} \hat \G_N \nabla g_{\theta^\star} - \hat \G_n \nabla g_{\theta^\star} \right) + o_P(1)
\xrightarrow{d} \mathcal{N}(0, \Sigma_g^\ppi),
\end{equation*}
where
\begin{equation}
    \label{eqn: sigma g ppi appendix}
    \Sigma_g^\ppi = H_{\theta^\star}^{-1}\lb r \Cov(\nabla g_{\theta^\star}(X, \hat Y)) + \Cov(\nabla \ell_{\theta^\star}(X, Y) - \nabla g_{\theta^\star}(X, \hat Y))\rb H_{\theta^\star}^{-1}.
\end{equation}

From \eqref{eqn: sigma g ppi appendix}, we can identify the optimal $g$, i.e. the one that minimizes the asymptotic variance. Notice that 
\begin{equation*}
    \begin{aligned}
        \Sigma_g^\ppi =H_{\theta^\star}^{-1}\Big(&  (1+r) \Cov(\nabla g_{\theta^\star}(X, \hat Y)) + \Cov(\nabla \ell_{\theta^\star}(X, Y)) \\
        &-2 \Cov(\nabla \ell_{\theta^\star}(X, Y), \nabla g_{\theta^\star}(X, \hat Y))\Big) H_{\theta^\star}^{-1} \\ 
        =H_{\theta^\star}^{-1}\Big(&  (1+r) \Cov\left(\nabla g_{\theta^\star}(X, \hat Y) - \frac{1}{1+r}\nabla \ell_{\theta^\star}(X, Y)\right) \\&+ \frac{r}{1+r}\Cov(\nabla \ell_{\theta^\star}(X, Y))\Big) H_{\theta^\star}^{-1} 
    \end{aligned}
\end{equation*}
Since the covariance matrix is positive semi-definite, the minimum is achieved when 
\[
\nabla g_{\theta^\star} = \frac{1}{1+r}\E[\nabla\ell_{\theta^\star}(X, Y)|X,\hat Y].\]
That is, for any other function $g'$, we have 
$
\Sigma_{g'}^\ppi \succeq \Sigma_{g}^\ppi. 
$
Under the optimal $g$, we have 
\begin{equation*}
    \begin{aligned}
        \Sigma_{g}^\ppi =& H_{\theta^\star}^{-1}\lb \frac{1}{1+r} \E\left[\Cov\left(\nabla \ell_{\theta^\star}(X, Y)|X,\hat Y\right)\right] + \frac{r}{1+r}\Cov(\nabla \ell_{\theta^\star}(X, Y))\rb H_{\theta^\star}^{-1} \\ 
        = &H_{\theta^\star}^{-1}\lb \E\left[\Cov\left(\nabla \ell_{\theta^\star}(X, Y)|X,\hat Y\right)\right] + \frac{r}{1+r}\Cov\lb\E\left[\nabla \ell_{\theta^\star}(X, Y)|X,\hat Y\right]\rb\rb H_{\theta^\star}^{-1}\\
        = &  H_{\theta^\star}^{-1}\lb \Cov(\nabla \ell_{\theta^\star}(X, Y)) - \frac{1}{1+r}\Cov(\E[\nabla \ell_{\theta^\star}(X, Y)|X, \hat Y])\rb H_{\theta^\star}^{-1}.
    \end{aligned}
\end{equation*}
\end{proof}
\subsection{Proof of Theorem \ref{thm:main}}
\label{sec: proof linear}

We introduce an additional assumption on the gradient of the loss function $\ell_\theta$. 
\begin{assumption}
    \label{asm: smooth gradient}
    The gradient $\nabla \ell_\theta(X, \hat Y)$ is differentiable and locally Lipschitz around $\theta^\star$, i.e., there exist a neighborhood of $\theta^\star$ such that $\nabla \ell_\theta(X, \hat Y)$ is $M_3(X, \hat Y)$ Lipschitz in $\theta$, and $\E[M_3(X, \hat Y)]<\infty$. 
\end{assumption}

We start with an intermediate result that serves as a building block for Theorem \ref{thm:main}. 
\begin{theorem}
\label{thm: power tune RePPI}
    Let the target $\theta^\star$ defined in \eqref{eq:thetastar} be unique and assume $n/N\rightarrow r$. Assume that $\hat s(X, \hat Y)$ is a random function such that $\E[\|\hat s(X, \hat Y) - s(X, \hat Y)\|^2] \xrightarrow{p} 0$ as $n\rightarrow \infty$, for some fixed function $s$. Let $\hat g_\theta(X,\hat Y) = \frac{1}{1+n/N}\theta^\top \hat M \hat s(X, \hat Y)$, where $\hat M$ is computed by \eqref{eq:Mhat}. Then, under Assumptions \ref{asm: regularity} and \ref{asm: smooth gradient}, $\hat \theta_{\hat g}^\ppi \xrightarrow{p} \theta^\star$ and $\sqrt{n}(\hat \theta_{\hat g}^\ppi - \theta^\star) \xrightarrow{d} \N(0, \Sigma_s^\REPPI)$, where $\Sigma_s^\REPPI$ is defined in Theorem \ref{thm:main}.
\end{theorem}
\begin{proof}
    For notational convenience, we write $\nabla \ell_\theta, \hat{s}, s$ for $\nabla \ell_\theta(X, Y), \hat{s}(X, \hat{Y}), s(X, \hat{Y})$. By the law of large number, the estimated weight matrix $\hat M$ converges to the population optimal matrix $M$:
    \begin{equation*}
        \hat M = \widehat \Cov(\nabla \ell_{\hat \theta_0}, \hat s) \widehat \Cov(\hat s)^{-1} \xrightarrow{p} \Cov(\nabla\ell_{\theta^\star}, s)\Cov(s)^{-1} = M.
    \end{equation*}
    The consistency of $\hat{\theta}_{\hat g}^\ppi$ then follows from the same argument as in the proof of Theorem \ref{thm: efficient PPI}.
    
    Now we turn to the proof of asymptotic normality. Since $\hat \theta_{\hat g}^{\ppi}$ is a consistent estimator of $\theta^\star$, by Assumption \ref{asm: smooth gradient} and Lemma 19.24 of \cite{van2000asymptotic}, we have 
    \begin{equation}
        \begin{aligned}
        \label{eqn: ell consistent}
            &\hat \G_n \left[\lb \nabla \ell_{\hat \theta_{\hat g}^{\ppi}} - \nabla \ell_{\theta^\star} \rb \right] \xrightarrow{p} 0,  
        \end{aligned}
    \end{equation}
    Moreover, since $\E\|\hat s - s\|^2 \xrightarrow{p} 0$,  Lemma 19.24 of \cite{van2000asymptotic} implies that 
\begin{equation}
        \begin{aligned}
        \label{eqn: g consistent}
            \hat \G_N \left[\lb \hat s - s \rb \right] \xrightarrow{p} 0, \,\, \hat \G_n \left[\lb \hat s - s \rb \right] \xrightarrow{p} 0.
        \end{aligned}
    \end{equation}
By the convexity of $\ell$, $\hat \theta_{\hat g}^{\ppi}$ solves the estimating equation 
    \begin{equation*}
        \Phi_{\hat g}^{\initial}(\theta) := \frac{1}{n}\sum_{i=1}^{n}\nabla\ell_{\theta}(X_i, Y_i) + \frac{1}{1+n/N}\hat M \lb \frac{1}{N}\sum_{i=n+1}^{n+N}\hat s(X_i, \hat Y_i)- \frac{1}{n}\sum_{i=1}^{n} \hat s(X_i, \hat Y_i)\rb = 0.
    \end{equation*}
    Define $\Psi(\theta) := \E[\nabla\ell_{\theta}(X, Y)] = \nabla L(\theta)$, then by the convexity of the loss function $\ell$, $\Psi(\theta^\star) = 0$. By Assumption \ref{asm: smooth gradient}, $\Psi(\theta)$ is differentiable, we can apply Lemma 2.12 in \cite{van2000asymptotic} and obtain that 
    $$
    \sqrt{n}H_{\theta^\star}(\hat \theta_{\hat g}^{\ppi}-\theta^\star) + \sqrt{n}\cdot o_P(\|\hat \theta_{\hat g}^{\ppi}-\theta^\star\|) = \sqrt{n}(\Psi(\hat \theta_{\hat g}^{\ppi}) - \Psi(\theta^\star)) + o_P(1).
    $$    
    Notice that $\Phi_{\hat g}^{\initial}(\hat \theta_{\hat g}^{\ppi})=0$, and therefore
    \begin{equation*}
        \begin{aligned}
            &\sqrt{n}(\Psi(\hat \theta_{\hat g}^{\ppi}) -  \Psi(\theta^\star)) \\=&  \sqrt{n}(\Psi(\hat \theta_{\hat g}^{\ppi}) - \Phi_{\hat g}^{\initial}(\hat \theta_{\hat g}^{\ppi}))\\
            =& -  \lb\G_n \nabla \ell_{\hat \theta_{\hat g}^{\ppi}} + \frac{1}{1+n/N}\hat M \lb\sqrt{\frac{n}{N}} \hat \G_N \hat s - \hat \G_n \hat s \rb\rb \\
            = & -  \lb\G_n \nabla \ell_{\theta^{\star}} + \left(\frac{1}{1+r}M+o_P(1)\right)\lb\sqrt{\frac{n}{N}} \hat \G_N  s - \hat \G_n  s + o_P(1) \rb\rb \\ 
            = & - \lb\G_n \nabla \ell_{\theta^{\star}} + \frac{1}{1+r} M\lb\sqrt{\frac{n}{N}} \hat \G_N  s - \hat \G_n  s \rb\rb +  o_P(1) 
        \end{aligned}
    \end{equation*}
    By the central limit theorem,  the right-hand side is asymptotically normal:
\begin{equation*}
    \begin{aligned}
        &\G_n \nabla \ell_{\theta^\star} + \frac{1}{1+r} M\lb\sqrt{\frac{n}{N}} \hat \G_N s - \hat \G_n s\rb \\
        =&\sqrt{\frac{n}{N}} \frac{1}{\sqrt{N}} \sum_{i=n+1}^{n+N} \frac{1}{1+r} M\lb s(X_i, \hat Y_i) -\E[s(X, \hat Y)]\rb \\
        &+\frac{1}{\sqrt{n}} \sum_{i=1}^n \lb\nabla \ell_{\theta^\star}(X_i, Y_i) - \frac{1}{1+r} M s(X_i, \hat Y_i) - \E[\nabla \ell_{\theta^\star}(X, Y) - \frac{1}{1+r} M s(X, \hat Y)]\rb \\ 
        \xrightarrow{d}& \mathcal{N}(0, \Lambda) = O_P(1),
    \end{aligned}
\end{equation*}
where the variance 
\begin{equation}
\label{eqn: variance power tuned}
    \begin{aligned}
        &\Lambda = \frac{r}{(1+r)^2} \Cov(Ms) + \Cov\left(\nabla \ell_{\theta^\star} - \frac{1}{1+r} Ms\right)\\
        = & \Cov(\nabla \ell_{\theta^\star}) - \frac{1}{1+r} \Cov(\nabla \ell_{\theta^\star}, s)\Cov(s)^{-1}\Cov(s, \nabla \ell_{\theta^\star}).
    \end{aligned}
\end{equation}
    Therefore, 
    $$
    \sqrt{n}\|\hat \theta_{\hat g}^{\ppi}-\theta^\star\| \leq \|H_{\theta^\star}^{-1}\|\sqrt{n}\|H_{\theta^\star}(\hat \theta_{\hat g}^{\ppi}-\theta^\star)\| = O_P(1) + o_P(\sqrt{n}\|\hat \theta_{\hat g}^{\ppi}-\theta^\star\|).
    $$
    This implies that $\sqrt{n}\|\hat \theta_{\hat g}^{\ppi}-\theta^\star\| = O_P(1)$, and we conclude that 
    \begin{equation}
    \label{eqn: linear expansion}
        \begin{aligned}
            &\sqrt{n}(\hat \theta_{\hat g}^{\ppi}-\theta^\star)  = \sqrt{n}H_{\theta^\star}^{-1}(\Psi(\hat \theta_{\hat g}^{\ppi}) - \Psi(\theta^\star)) + o_P(1) \\
            =&- H_{\theta^\star}^{-1}  \lb\G_n \nabla \ell_{\theta^{\star}} + \frac{1}{1+r}M\lb\sqrt{\frac{n}{N}} \hat \G_N  s - \hat \G_n  s\rb\rb+o_P(1) \xrightarrow{d} \N(0, \Sigma_s^\REPPI),
        \end{aligned}
    \end{equation}
    as desired.
\end{proof} 
Now we are ready to prove Theorem \ref{thm:main}.
\begin{proof}[of Theorem \ref{thm:main}]
   By Theorem \ref{thm: power tune RePPI}, each of the estimates $\hat \theta^1, \hat \theta^2, \hat \theta^3$ is consistent, and therefore
    $$
    \hat \theta^\cf = \frac{|\mathcal{D}_1|\hat \theta^1+ |\mathcal{D}_2|\hat \theta^2+ |\mathcal{D}_3|\hat \theta^3}{n} \xrightarrow{p}\theta^\star. 
    $$
    Moreover, applying the asymptotic linear expansion \eqref{eqn: linear expansion} in the proof of Theorem \ref{thm: power tune RePPI}, we know that 
    \begin{equation*}
    \begin{aligned}
        &\sqrt{|\mathcal{D}_k|}(\hat \theta^k -\theta^\star) = -  H_{\theta^\star}^{-1}  \lb\G_n^k \nabla \ell_{\theta^{\star}} + \frac{1}{1+r}M\lb\sqrt{\frac{|\mathcal{D}_k|}{N}} \hat \G_N  s - \hat \G_n^k  s\rb\rb+o_P(1) \\
        = & -H_{\theta^\star}^{-1}\sqrt{\frac{|\mathcal{D}_k|}{N}} \frac{1}{\sqrt{N}} \sum_{i=n+1}^{n+N} \frac{1}{1+r} M\lb s(X_i, \hat Y_i) -\E[s(X, \hat Y)]\rb \\
        &-H_{\theta^\star}^{-1}\frac{1}{\sqrt{|\mathcal{D}_k|}} \sum_{i\in \mathcal{D}_k } \lb\nabla \ell_{\theta^\star}(X_i, Y_i) - \frac{1}{1+r}  M s(X_i, \hat Y_i) - \E[\nabla \ell_{\theta^\star}(X, Y) - \frac{1}{1+r} M s(X, \hat Y)]\rb \\
        &+ o_P(1),\\
    \end{aligned}
    \end{equation*}
    where the notation $\G_n^k, \hat \G_n^k$ is the analogy of $\G_n, \hat \G_n$ defined on the respective fold. Therefore, combining all three linear expansions together, we obtain that 
    \begin{equation*}
    \begin{aligned}
        &\sqrt{n}(\hat \theta^\cf -\theta^\star)  = \sum_{k = 1, 2, 3} \sqrt{\frac{|\mathcal{D}_k|}{n}} \sqrt{|\mathcal{D}_k|} (\hat \theta^k -\theta^\star)\\
        = & -H_{\theta^\star}^{-1}\sqrt{\frac{n}{N}} \frac{1}{\sqrt{N}} \sum_{i=n+1}^{n+N} \frac{1}{1+r} M\lb s(X_i, \hat Y_i) -\E[s(X, \hat Y)]\rb \\
        &-H_{\theta^\star}^{-1}\frac{1}{\sqrt{n}} \sum_{i=1}^n \lb\nabla \ell_{\theta^\star}(X_i, Y_i) - \frac{1}{1+r} M s(X_i, \hat Y_i) - \E[\nabla \ell_{\theta^\star}(X, Y) - \frac{1}{1+r} M s(X, \hat Y)]\rb \\
        &+ o_P(1)\\ 
        \xrightarrow{d}&  \N\lb 0, H_{\theta^\star}^{-1}\lb \Cov(\nabla \ell_{\theta^\star}) - \frac{1}{1+r} \Cov(\nabla \ell_{\theta^\star}, s)\Cov(s)^{-1}\Cov(s, \nabla \ell_{\theta^\star})\rb H_{\theta^\star}^{-1}\rb,
    \end{aligned}
    \end{equation*}
    where the last equation is obtained from \eqref{eqn: variance power tuned}, and this concludes the proof. 
\end{proof}

\section{Proofs from Section \ref{sec: linear model}}
\label{sec: proof lm}
This section provides a detailed computation of the asymptotic variance for the three examples from Section \ref{sec: linear model}. The asymptotic variances of the XY-only and PPI estimators are given by $\Sigma_g^{\ppi}$ in Theorem~\ref{thm: efficient PPI} with $g = 0$ and $g = \nabla \ell_\theta$, respectively. The asymptotic variance of the PPI++ estimator is derived in \cite{angelopoulos2023ppi++} (Theorem 1 and Proposition 2) and we state it below for completeness. Here, we focus on the PPI++ estimator with the optimal power tuning. We summarize these results in Lemma \ref{lem: PPI++ var}. 

\begin{lemma}
\label{lem: PPI++ var}
In the same setting as in Theorem \ref{thm: efficient PPI}, 
\begin{align*}
\tr(\Sigma^\cls) =& \tr(H_{\theta^\star}^{-1}\Cov(\nabla\ell_{\theta^\star}(X, Y)) H_{\theta^\star}^{-1} );\\
\tr(\Sigma^\ppi) =&\tr(H_{\theta^\star}^{-1}(r\Cov(\nabla\ell_{\theta^\star}(X, \hat Y))+\Cov(\nabla\ell_{\theta^\star}(X, Y) - \nabla\ell_{\theta^\star}(X, \hat Y)))H_{\theta^\star}^{-1});\\
\tr(\Sigma^\ppiplus) =& \tr(H_{\theta^\star}^{-1}\Cov(\nabla\ell_{\theta^\star}(X, Y)) H_{\theta^\star}^{-1} )\\
&-\frac{1}{1+r} \frac{\mathrm{Tr} \left( H_{\theta^\star}^{-1}  \mathrm{Cov}(\nabla \ell_{\theta^\star}(X, Y), \nabla \ell_{\theta^\star}(X, \hat{Y})) 
 H_{\theta^\star}^{-1} \right)^2}
{ \mathrm{Tr}\left(H_{\theta^\star}^{-1} \, \mathrm{Cov}(\nabla \ell_{\theta^\star}(X, \hat{Y})) \, H_{\theta^\star}^{-1}\right)}.
\end{align*}
\end{lemma}
Moreover, we provide the following lemma useful in the computation of closed-form expressions.
\begin{lemma}
\label{lem: gaussian moment}
    If $X\sim \N(0, \Sigma_X)\in \R^d$, and $\theta\in \R^d$ is a fixed vector, then $$\tr(\Sigma_X^{-1}\Cov(XX^\top \theta)\Sigma_X^{-1}) =\|\theta\|^2 + \tr(\Sigma_X^{-1})\cdot \theta^\top \Sigma_X \theta.$$
\end{lemma}
\begin{proof}
    We write the eigendecomposition $\Sigma_X = QDQ^\top$, where $Q$ is an orthonormal matrix and $D = \mathrm{diag}(\sigma_1^2, \cdots, \sigma_d^2)$ is a diagonal matrix. Then we can write $X = QZ$, where $Z\sim \N(0, D)$ is a normal vector with independent coordinates. Notice that 
    \begin{align*}
        \Cov(XX^\top \theta)& = \E[XX^\top \theta\theta^\top XX^\top]  - \E[XX^\top]\theta\theta^\top \E[XX^\top] \\
        &= \E[XX^\top \theta\theta^\top XX^\top] - \Sigma_X \theta\theta^\top\Sigma_X,
    \end{align*}
    and 
    \begin{align*}
        \tr(\Sigma_X^{-1}\E[XX^\top \theta\theta^\top XX^\top]\Sigma_X^{-1}) = &\tr(\theta\theta^\top\E[XX^\top \Sigma_X^{-2}XX^\top]) \\
        =& \tr(\theta\theta^\top Q\E[ZZ^\top D^{-2} ZZ^\top] Q^\top).
    \end{align*}
    Let $S = \E[ZZ^\top D^{-2} ZZ^\top] = \E[(Z^\top D^{-2} Z) ZZ^\top]$, then for $i\neq j$,
    $$
    S_{i, j} = \E\left[\lb\sum_{k=1}^d \frac{z_k^2}{\sigma_k^4}\rb z_iz_j\right] = 0,
    $$
    and
    $$
    S_{i, i} = \E\left[\lb\sum_{k=1}^d \frac{z_k^2}{\sigma_k^4}\rb z_i^2\right] = \frac{\E[z_i^4]}{\sigma_i^4} + \sum_{k\neq i}\frac{\E[z_k^2 z_i^2]}{\sigma_k^4} = 3 + \sum_{k\neq i}\frac{\sigma_i^2}{\sigma_k^2} = 2 + \tr(D^{-1})\sigma_i^2.
    $$
    Putting everything together, we have $S = 2I + \tr(D^{-1})D$, and 
    \begin{align*}
&\tr(\Sigma_X^{-1}\E[XX^\top \theta\theta^\top XX^\top]\Sigma_X^{-1}) = \tr(\theta\theta^\top Q(2I + \tr(D^{-1})D) Q^\top)\\
& = 2\tr(\theta \theta^\top) + \tr(D^{-1})\tr(\theta\theta^\top \Sigma_X) = 2\|\theta\|^2 + \tr(\Sigma_X^{-1})\cdot \theta^\top \Sigma_X \theta.
    \end{align*}
    Thus we can conclude that 
    $$
    \tr(\Sigma_X^{-1}\Cov(XX^\top \theta)\Sigma_X^{-1}) = \|\theta\|^2 + \tr(\Sigma_X^{-1})\cdot \theta^\top \Sigma_X \theta.     $$
\end{proof}
\subsection{Modality Mismatch}
\label{sec: proof independent}
\begin{proof}[of Proposition \ref{prop: lm independent}]
In this setting, we have
$$
\hat{Y} = W^\top \gamma, \quad \theta^\star = \theta, \quad H_{\theta^\star} = \E[XX^\top] = \Sigma_X.
$$
Then,
\begin{align*}      &\tr(H_{\theta^\star}^{-1}\Cov(\nabla\ell_{\theta^\star}(X, Y))H_{\theta^\star}^{-1}) \\
= &\tr(\Sigma_X^{-1}\Cov(X(X^\top \theta - Y))\Sigma_X^{-1})\\ =& \tr(\Sigma_X^{-1}\Cov(X(W^\top \gamma + \epsilon))\Sigma_X^{-1}) \\=& \tr(\Sigma_X^{-1}\Cov(XW^\top \gamma)\Sigma_X^{-1}) + \tr(\Sigma_X^{-1}\Cov(X\epsilon)\Sigma_X^{-1}) \\=& \tr(\Sigma_X^{-1})(\sigma^2 + \gamma^\top \Sigma_W \gamma);\\
\end{align*}
\begin{align*}
    &\tr(H_{\theta^\star}^{-1}\Cov(\nabla\ell_{\theta^\star}(X, \hat Y))H_{\theta^\star}^{-1})\\
    =& \tr(\Sigma_X^{-1}\Cov(X(X^\top \theta - W^\top \gamma))\Sigma_X^{-1}) \\=& \tr(\Sigma_X^{-1}\Cov(XX^\top \theta)\Sigma_X^{-1}) + \tr(\Sigma_X^{-1}\Cov(XW^\top \gamma)\Sigma_X^{-1})\\
        =& \|\theta\|^2+\tr(\Sigma_X^{-1})(\theta^\top \Sigma_X\theta+\gamma^\top \Sigma_W\gamma);\\
\end{align*}
        \begin{align*} 
        &\tr(H_{\theta^\star}^{-1}\Cov(\nabla\ell_{\theta^\star}(X, Y), \nabla\ell_{\theta^\star}(X, \hat Y))H_{\theta^\star}^{-1})\\
        =& \tr(\Sigma_X^{-1}\Cov(X(X^\top\theta - Y), X(X^\top\theta - \hat Y))\Sigma_X^{-1}) \\= & \tr(\Sigma_X^{-1}\Cov(-X(W^\top\gamma+\epsilon), X(X^\top\theta - W^\top\gamma))\Sigma_X^{-1})\\
        =& \tr(\Sigma_X^{-1}\Cov(XW^\top\gamma)\Sigma_X^{-1}) = \tr(\Sigma_X^{-1})\gamma^\top \Sigma_W\gamma,
\end{align*}
where the second expression is derived using Lemma \ref{lem: gaussian moment}. Using Lemma \ref{lem: PPI++ var}, we can now derive the asymptotic variance of the XY-only, PPI, and PPI++ estimators:

\begin{align*}
  \tr(\Sigma^\cls) = &\tr(H_{\theta^\star}^{-1}\Cov(\nabla\ell_{\theta^\star}(X, Y))H_{\theta^\star}^{-1}) = (\sigma^2 + \gamma^\top \Sigma_W \gamma)\tr(\Sigma_X^{-1});\\
      \tr(\Sigma^\ppi) =& \tr(H_{\theta^\star}^{-1}(r\Cov(\nabla\ell_{\theta^\star}(X, \hat Y))+\Cov(\nabla\ell_{\theta^\star}(X, Y) - \nabla\ell_{\theta^\star}(X, \hat Y)))H_{\theta^\star}^{-1})\\
    =& \tr\bigg(H_{\theta^\star}^{-1}\Big((1+r)\Cov(\nabla\ell_{\theta^\star}(X, \hat Y))\\
    & \quad +\Cov(\nabla\ell_{\theta^\star}(X, Y))-2\Cov(\nabla\ell_{\theta^\star}(X, Y), \nabla\ell_{\theta^\star}(X, \hat Y))\Big)H_{\theta^\star}^{-1}\bigg)\\
    = & (1+r) \|\theta\|^2 + (\sigma^2+ (1+r) \theta^\top \Sigma_X\theta + r\gamma^\top \Sigma_W\gamma) \tr(\Sigma_X^{-1});\\
      \tr(\Sigma^\ppiplus) =& \tr(H_{\theta^\star}^{-1}\Cov(\nabla\ell_{\theta^\star}(X, Y)) H_{\theta^\star}^{-1} ) \\&- \frac{1}{1+r}\frac{\mathrm{Tr}\left(H_{\theta^\star}^{-1}\mathrm{Cov}(\nabla \ell_{\theta^\star}(X, Y), \nabla \ell_{\theta^\star}(X, \hat{Y}))  H_{\theta^\star}^{-1} \right)^2}
{\mathrm{Tr}\left(H_{\theta^\star}^{-1} \, \mathrm{Cov}(\nabla \ell_{\theta^\star}(X, \hat{Y})) \, H_{\theta^\star}^{-1}\right)} \\
=&(\sigma^2 + \gamma^\top \Sigma_W \gamma)\tr(\Sigma_X^{-1}) - \frac{1}{1+r} \frac{(\tr(\Sigma_X^{-1})\cdot \gamma^\top \Sigma_W\gamma )^2}{\|\theta\|^2+\tr(\Sigma_X^{-1})(\theta^\top \Sigma_X\theta+\gamma^\top \Sigma_W\gamma)}\\
= &\left(\sigma^2 + \left(\frac{r}{1+r} + \frac{1}{1+r}\frac{1}{1+\frac{\tr(\Sigma_X^{-1})\gamma^\top \Sigma_W\gamma}{\|\theta\|^2 +\tr(\Sigma_X^{-1})\theta^\top \Sigma_X\theta }}\right)\gamma^\top \Sigma_W \gamma\right)\tr(\Sigma_X^{-1}).
\end{align*}
For our recalibrated PPI estimator, using Theorem  \ref{thm: efficient PPI}, we can obtain that 
\begin{equation*}
    \begin{aligned}
        \Sigma^\cls - \Sigma^\REPPI  =& \frac{1}{1+r} \Cov(\E[H_{\theta^\star}^{-1}\nabla\ell_{\theta^\star}(X, Y)|X, \hat Y]) \\
        =& \frac{1}{1+r} \Cov(\Sigma_X^{-1}\E[X(X^\top\theta - Y)|X, \hat Y]) \\
        =& \frac{1}{1+r} \Cov(\Sigma_X^{-1}(X(X^\top\theta - X^\top\theta - W^\top\gamma))) = \frac{1}{1+r} \gamma^\top \Sigma_W \gamma\cdot \Sigma_X^{-1},
    \end{aligned}
\end{equation*}
which implies 
$$
\tr(\Sigma^\REPPI) =  \left(\sigma^2 + \frac{r}{1+r} \gamma^\top \Sigma_W\gamma\right )\tr(\Sigma_X^{-1}),
$$
concluding the proof.  
\end{proof}

\subsection{Distribution Shift}
\begin{proof}[of Proposition \ref{prop: lm shift}]
In this setting, we have
$$
\hat{Y} = X^\top\td{\theta} + W^\top \td{\gamma}, \quad \theta^\star = \theta, \quad H_{\theta^\star} = \E[XX^\top] = \Sigma_X.
$$
Following the same calculation as in the proof of Proposition \ref{prop: lm independent}, we have 
\begin{align*}        \tr(H_{\theta^\star}^{-1}\Cov(\nabla\ell_{\theta^\star}(X, Y))H_{\theta^\star}^{-1}) =& (\sigma^2 + \gamma^\top \Sigma_W \gamma)\tr(\Sigma_X^{-1});\\
\end{align*}
        \begin{align*} 
&\tr(H_{\theta^\star}^{-1}\Cov(\nabla\ell_{\theta^\star}(X, \hat Y))H_{\theta^\star}^{-1})\\
=& \tr(\Sigma_X^{-1}\Cov(X(X^\top \theta -X^\top \tilde\theta -  W^\top \td{\gamma}))\Sigma_X^{-1}) \\=& \tr(\Sigma_X^{-1}\Cov(XX^\top (\theta- \tilde \theta))\Sigma_X^{-1}) + \tr(\Sigma_X^{-1}\Cov(XW^\top \td{\gamma})\Sigma_X^{-1})\\
=& \|\theta-\tilde\theta\|^2+\tr(\Sigma_X^{-1})((\theta - \tilde\theta)^\top \Sigma_X(\theta - \tilde\theta)+\td{\gamma}^\top \Sigma_W\td{\gamma});\\    
\end{align*}
        \begin{align*} 
&\tr(H_{\theta^\star}^{-1}\Cov(\nabla\ell_{\theta^\star}(X, Y), \nabla\ell_{\theta^\star}(X, \hat Y))H_{\theta^\star}^{-1})\\
=& \tr(\Sigma_X^{-1}\Cov(X(X^\top\theta - Y), X(X^\top\theta - \hat Y))\Sigma_X^{-1}) \\= & \tr(\Sigma_X^{-1}\Cov(-X(W^\top\gamma+\epsilon), X(X^\top\theta -X^\top\tilde\theta- W^\top\td{\gamma}))\Sigma_X^{-1})\\
=& \tr(\Sigma_X^{-1}\Cov(XW^\top\gamma, XW^\top\td{\gamma})\Sigma_X^{-1})\\
=& \tr(\Sigma_X^{-1})\gamma^\top \Sigma_W\td{\gamma}.
\end{align*}
Similarly to Proposition \ref{prop: lm independent},
\begin{align*}
\tr(\Sigma^\cls) =& \tr(H_{\theta^\star}^{-1}\Cov(\nabla\ell_{\theta^\star}(X, Y))H_{\theta^\star}^{-1}) = (\sigma^2 + \gamma^\top \Sigma_W \gamma)\tr(\Sigma_X^{-1});\\
\end{align*}
        \begin{align*}
&\tr(\Sigma^\ppi) = \tr(H_{\theta^\star}^{-1}(r\Cov(\nabla\ell_{\theta^\star}(X, \hat Y))+\Cov(\nabla\ell_{\theta^\star}(X, Y) - \nabla\ell_{\theta^\star}(X, \hat Y)))H_{\theta^\star}^{-1})\\
=& \tr\bigg(H_{\theta^\star}^{-1}\Big((1+r)\Cov(\nabla\ell_{\theta^\star}(X, \hat Y))\\
& \quad +\Cov(\nabla\ell_{\theta^\star}(X, Y))-2\Cov(\nabla\ell_{\theta^\star}(X, Y), \nabla\ell_{\theta^\star}(X, \hat Y))\Big)H_{\theta^\star}^{-1}\bigg)\\
= & (1+r) \|\theta-\tilde \theta\|^2 \\
& \quad + \lb\sigma^2+ (1+r) (\theta -\tilde\theta)^\top \Sigma_X(\theta -\tilde\theta) + r\tilde \gamma^\top \Sigma_W \tilde\gamma + (\gamma- \td{\gamma})^\top \Sigma_W(\gamma - \td{\gamma})\rb \tr(\Sigma_X^{-1});\\
\end{align*}
        \begin{align*}
&\tr(\Sigma^\ppiplus) \\\
=& \tr(H_{\theta^\star}^{-1}\Cov(\nabla\ell_{\theta^\star}(X, Y)) H_{\theta^\star}^{-1} )- \frac{1}{1+r}\frac{\mathrm{Tr}\left(H_{\theta^\star}^{-1}\mathrm{Cov}(\nabla \ell_{\theta^\star}(X, Y), \nabla \ell_{\theta^\star}(X, \hat{Y}))  H_{\theta^\star}^{-1} \right)^2}
{\mathrm{Tr}\left(H_{\theta^\star}^{-1} \, \mathrm{Cov}(\nabla \ell_{\theta^\star}(X, \hat{Y})) \, H_{\theta^\star}^{-1}\right)} \\
= &\left(\sigma^2 + \gamma^\top \Sigma_W \gamma - \frac{1}{1+r}\frac{(\gamma^\top \Sigma_W \td{\gamma})^2}{\td{\gamma}^\top \Sigma_W \td{\gamma} + (\theta -\tilde\theta)^\top \Sigma_X(\theta -\tilde\theta) + \|\theta -\tilde\theta\|^2 / \tr(\Sigma_X^{-1})}\right)\tr(\Sigma_X^{-1}).
\end{align*}
For our recalibrated PPI estimator, using Theorem  \ref{thm: efficient PPI}, again we can obtain that 
\begin{align*}
    \Sigma^\cls - \Sigma^\REPPI  =& \frac{1}{1+r} \Cov(\E[H_{\theta^\star}^{-1}\nabla\ell_{\theta^\star}(X, Y)|X, \hat Y ])\\
        = &\frac{1}{1+r} \Cov(\Sigma_X^{-1}\E[X(X^\top\theta - Y)|X, X^\top\tilde\theta+W^\top\td{\gamma}]) \\
        = &\frac{1}{1+r} \Cov(\Sigma_X^{-1}\E[X(W^\top\gamma + \eps)|X, W^\top\td{\gamma}]) \\
        = &\frac{1}{1+r} \Cov(\Sigma_X^{-1}X\E[W^\top\gamma| W^\top\td{\gamma}]) \\        
 = &\frac{1}{1+r} \frac{(\gamma^\top \Sigma_W \td{\gamma})^2}{\td{\gamma}^\top \Sigma_W \td{\gamma}}\cdot \Sigma_X^{-1},
\end{align*}
where the last line uses the fact that $(W^\top \gamma, W^\top \td{\gamma})$ are bivariate Gaussian.
This implies 
$$
\tr(\Sigma^\REPPI) =  \left(\sigma^2 + \gamma^\top \Sigma_W\gamma - \frac{1}{1+r} \frac{(\gamma^\top \Sigma_W \td{\gamma})^2}{\td{\gamma}^\top \Sigma_W \td{\gamma}}\right )\tr(\Sigma_X^{-1}).
$$
\end{proof}

\subsection{Discrete Predictions}
\begin{proof}[of Proposition \ref{prop: binary}]
    Using Lemma \ref{lem: PPI++ var} and Theorem  \ref{thm: efficient PPI}, we know that for mean estimation, 
\begin{align*}
    &\Sigma^\cls = \Var(Y), \\
    &\Sigma^\ppi = \Var(Y-\hat Y) + r\Var(\hat Y), \\
    &\Sigma^\ppiplus = \Var(Y) - \frac{1}{1+r}\frac{\Cov(Y, \hat Y)^2}{\Var(\hat{Y})}, \\
    &\Sigma^\REPPI = \Var(Y) - \frac{1}{1+r}\Var(\E[Y|\hat Y]). 
\end{align*}
Under the Gaussian mixture model,
\begin{equation*}
    \begin{aligned}
        &\Var(Y) = \sigma^2 + \Var(\mu_Z) = \sigma^2 + \frac{1}{3}\sum_{i=1}^3 \mu_i^2 - \lb\frac{1}{3}\sum_{i=1}^3 \mu_i\rb^2 \\
        &= \sigma^2 + \frac{(\mu_1 - \mu_2)^2 + (\mu_2 - \mu_3)^2 + (\mu_3 - \mu_1)^2}{9}, \\
        &\Var(\hat Y) = \Var(Z) = \frac{2}{3}, \\ 
        & \Cov(Y, \hat Y) = \frac{\mu_1+ 2\mu_2+3\mu_3}{3} - 2\frac{\mu_1+\mu_2+\mu_3}{3} = \frac{\mu_3-\mu_1}{3},\\
        & \Var(\E[Y|\hat Y]) = \Var(\mu_Z) = \frac{(\mu_1 - \mu_2)^2 + (\mu_2 - \mu_3)^2 + (\mu_3 - \mu_1)^2}{9}.
    \end{aligned}
\end{equation*}
Therefore, the asymptotic variances are
\begin{align*}
    &\Sigma^\cls = \sigma^2 + \frac{(\mu_1 - \mu_2)^2 + (\mu_2 - \mu_3)^2 + (\mu_3 - \mu_1)^2}{9}, \\
    &\Sigma^\ppi = \Sigma^\cls + \frac{2(1+r)}{3} - \frac{2(\mu_3-\mu_1)}{3}, \\
    &\Sigma^\ppiplus = \Sigma^\cls - \frac{1}{1+r}\frac{(\mu_3-\mu_1)^2}{6}, \\
    &\Sigma^\REPPI = \Sigma^\cls - \frac{1}{1+r}\frac{(\mu_1 - \mu_2)^2 + (\mu_2 - \mu_3)^2 + (\mu_3 - \mu_1)^2}{9}. 
\end{align*}
This in turn implies that $\Sigma^\ppiplus - \Sigma^\REPPI =  \frac{1}{1+r}\frac{(2\mu_2 - \mu_3- \mu_1)^2}{18} \geq 0$.
\end{proof}

\section{Simulation studies}
\label{sec:additional-simulations}
\subsection{Simulation in Stylized models}
\subsubsection{Modality Mismatch}
\label{sec:sim_modality}
\begin{figure}[t]
    \centering
    \includegraphics[width=0.49\linewidth]{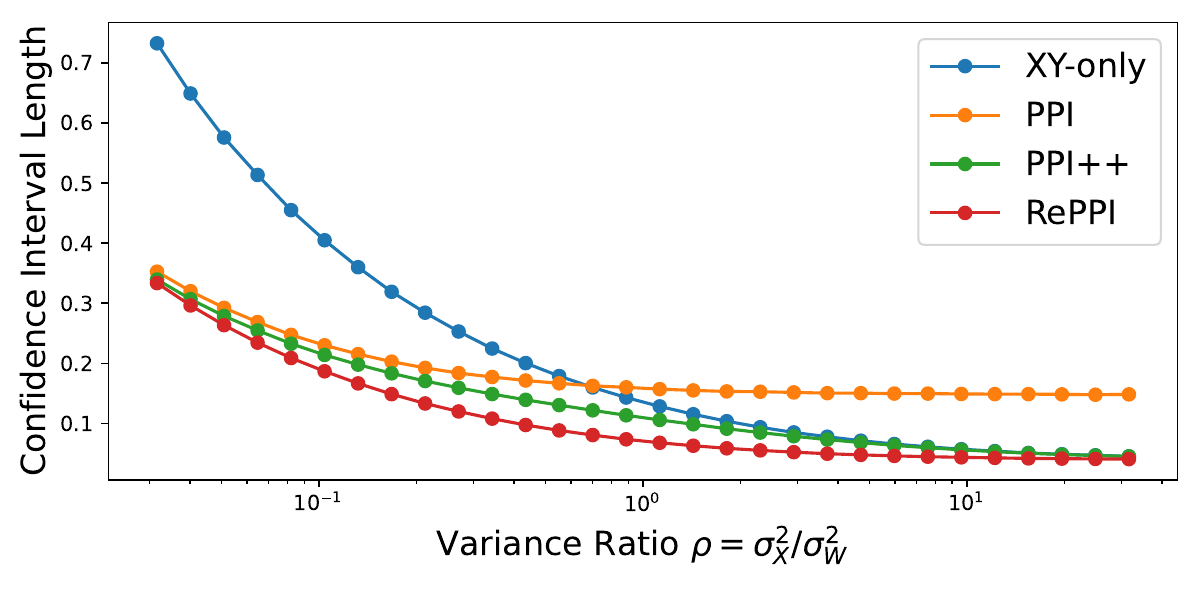}
    \includegraphics[width=0.49\linewidth]{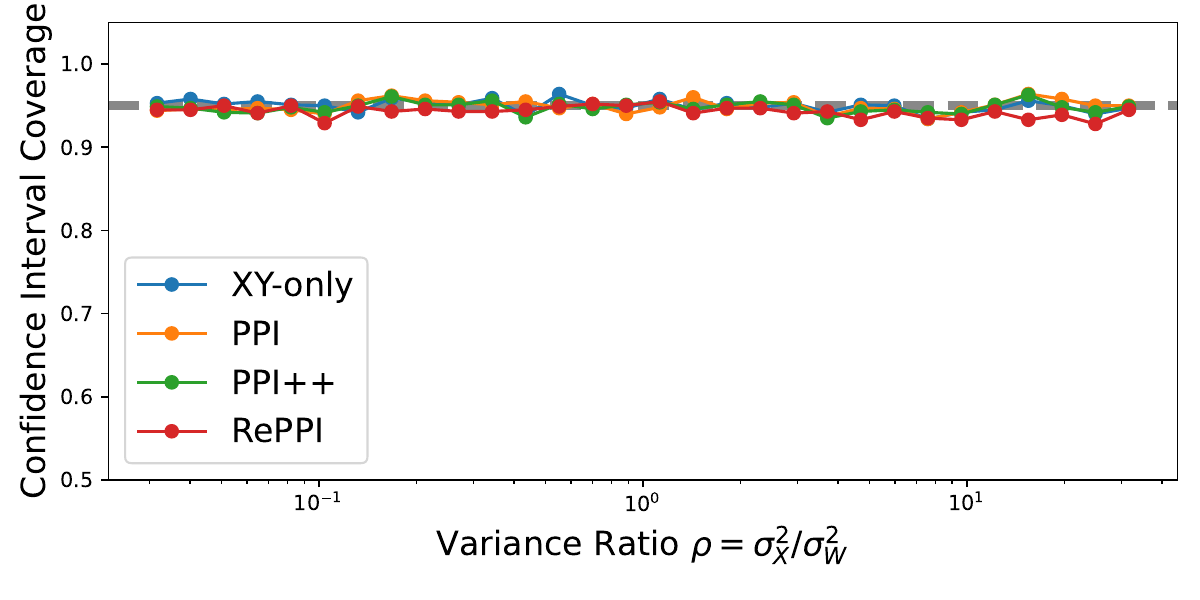}
    \caption{Average length of confidence intervals (left) and coverage (right) in modality mismatch simulation. The horizontal axis represents the variance ratio $\rho = \sigma_X^2 / \sigma_W^2$.}
    \label{fig:lm}
\end{figure}
We run a simple simulation to compare the asymptotic variances numerically. We choose $d=5$, $n=1000$, $N=9000$, $\sigma^2 = 1$, $\Sigma_X = \sigma_X^2 I_d, \Sigma_W = \sigma_W^2 I_d$ and $\sigma_W^2+\sigma_X^2 = 10$. In each trial, we generate $\theta$ and $\gamma$ randomly from the unit sphere $\mathbb{S}^{d-1}$. We vary the variance ratio $\rho = \sigma_X^2 / \sigma_W^2 \in [10^{-1.5}, 10^{1.5}]$. For each value of $\rho$ and each of $1000$ trials, we
compute $95\%$ confidence intervals for each method. We report the average interval lengths and coverage rates in Figure \ref{fig:lm}.

From Figure \ref{fig:lm} we find that, as expected, all methods achieve the desired coverage approximately and RePPI uniformly improves upon PPI and PPI++. As discussed above, the gain of RePPI over PPI++ is inverse U-shaped and maximized in the middle, when $\rho=\sigma_X^2/\sigma_W^2$ is close to 1, in which case both $X$ and $W$ are informative in explaining $Y$. 
\subsubsection{Distribution Shift}
\label{sec:sim_shift}

Again, we compare the estimators numerically using a stylized simulation study in Figure \ref{fig:lm shift}. We choose $d=5$, $n=1000$, $N=9000$, $\sigma^2 = 1$, $\Sigma_X = 10 I_d, \Sigma_W = 10 I_d$ and fixed the shift on $\gamma$ to be $\|\gamma-\tilde \gamma\| = 1$. In each trial, we generate $\theta$, $\gamma$, and the shift direction $\frac{\theta-\tilde\theta}{\|\theta-\tilde\theta\|}, \gamma-\tilde\gamma$ randomly from the unit sphere $\mathbb{S}^{d-1}$. 
As in Section \ref{sec:modality_mismatch}, we report the average $95\%$ confidence interval lengths and coverage rates for different values of $\|\theta - \td{\theta}\|$ over 1000 trials. As indicated by Proposition~\ref{prop: lm shift}, the asymptotic variance of RePPI does not vary with the bias of $\td{\theta}$, and the PPI estimator is worse than the XY-only estimator as $\|\gamma-\tilde\gamma\|$ is large.
\begin{figure}[t]
    \centering
    \includegraphics[width=0.49\linewidth]{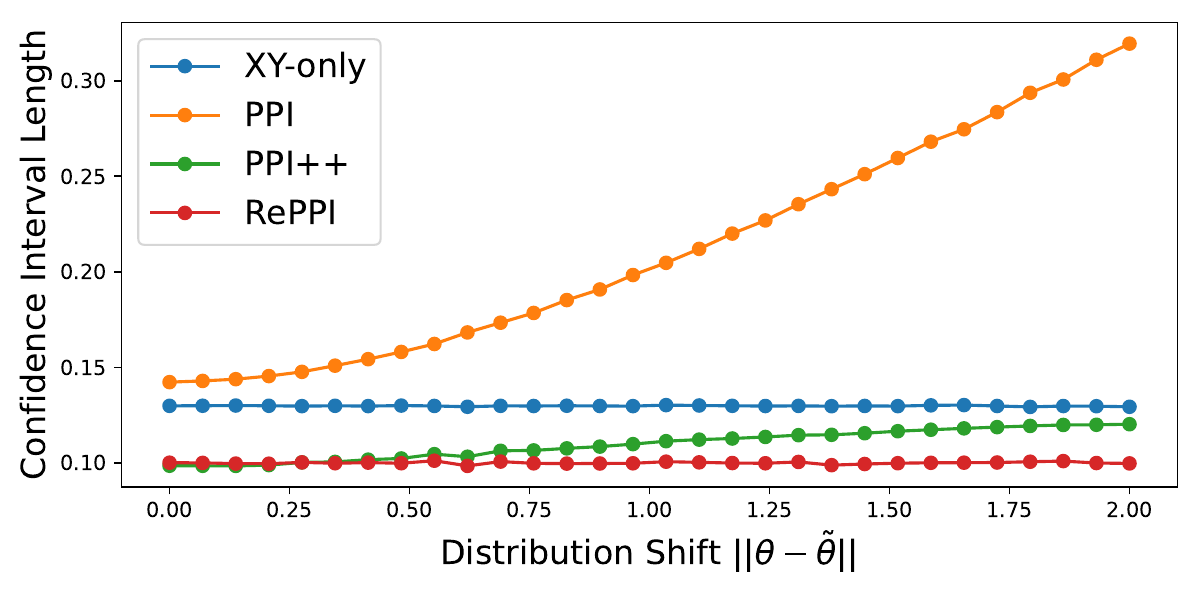}
    \includegraphics[width=0.49\linewidth]{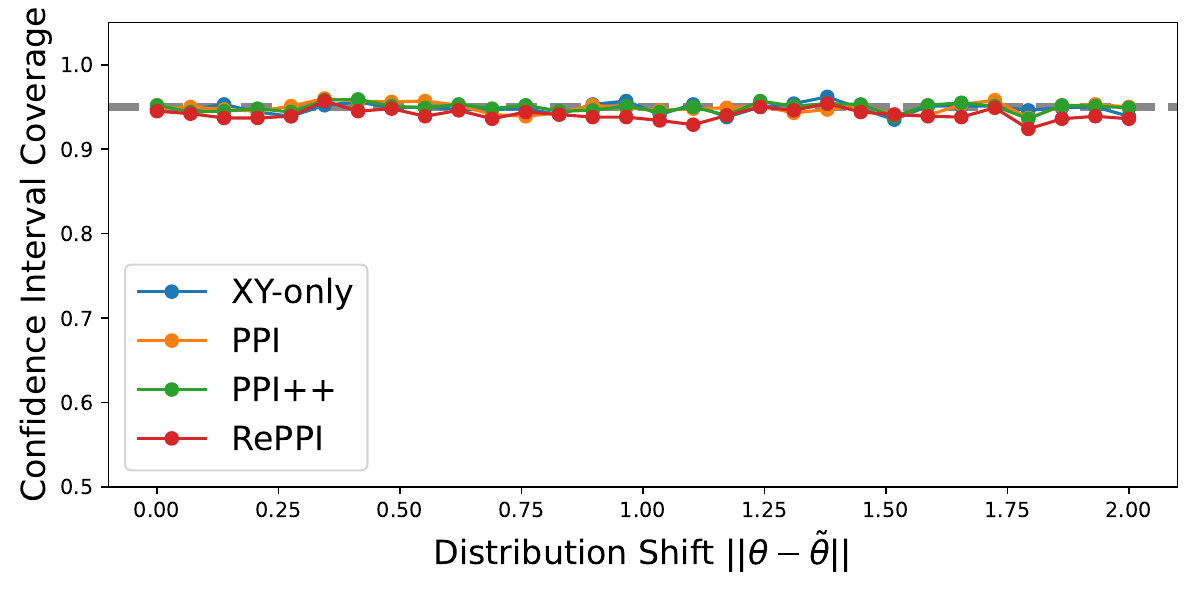}
    \caption{Average length of confidence intervals (left) and coverage (right) in distribution shift simulation. The horizontal axis represents the bias $\|\theta - \td{\theta}\|$.}
    \label{fig:lm shift}
\end{figure}
\subsubsection{Discrete Predictions}
\label{sec:sim_discrete}
\begin{figure}[ht]
    \centering
    \includegraphics[width=0.49\linewidth]{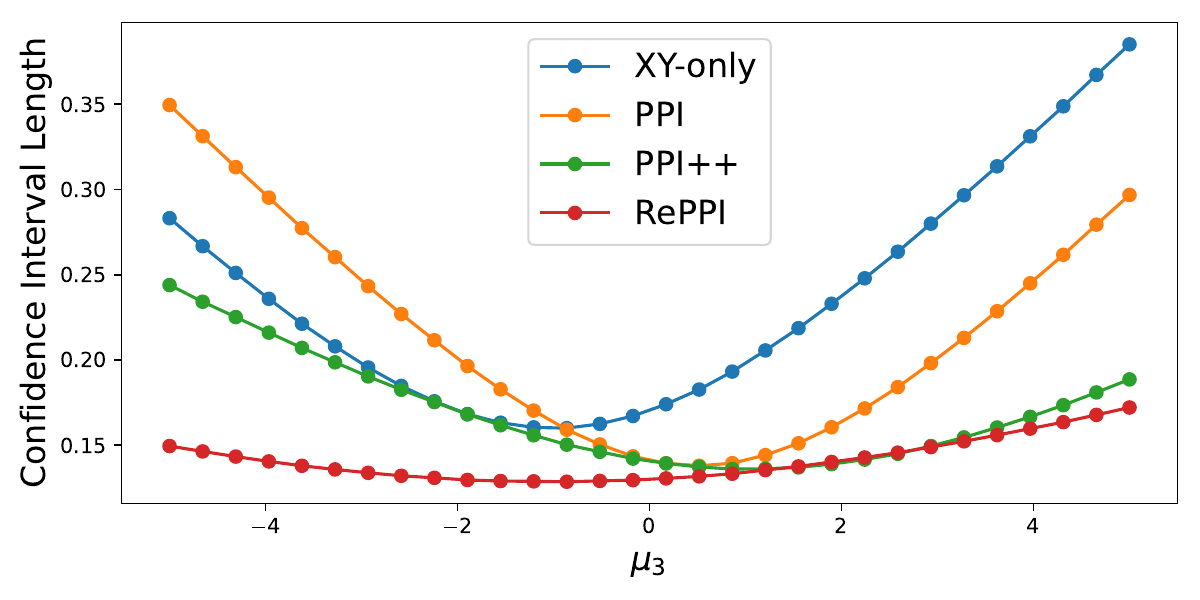}
    \includegraphics[width=0.49\linewidth]{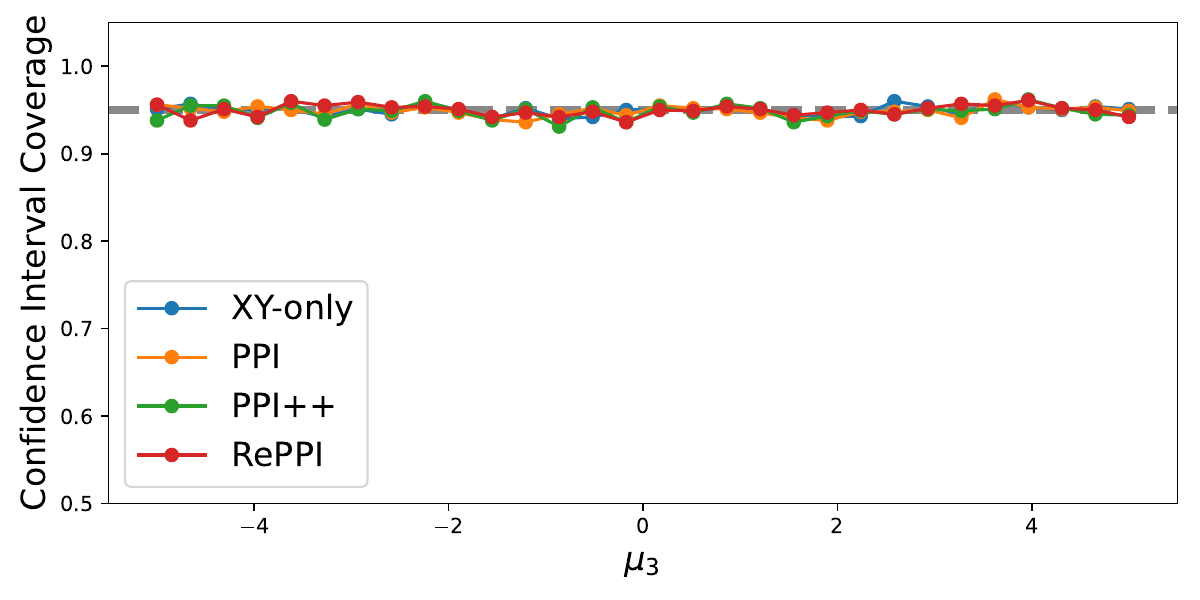}
    \caption{Average length of confidence intervals (left) and coverage (right) in simulation with discrete predictions. The horizontal axis represents $\mu_3$.}
    \label{fig:lm binary}
\end{figure}
We evaluate the performance of the estimators via a stylized simulation study where $\mu_1 = -2$, $\mu_2 = 0$, $\sigma^2 = 1$, and $\mu_3$ is varied. We set $n=1000$, $N=9000$, and repeat the simulation for 1000 trials. The average $95\%$ confidence interval length and coverage rates are reported in Figure \ref{fig:lm binary}. As suggested by Proposition \ref{prop: binary}, the efficiency gap is zero when $\mu_3 - \mu_2 = \mu_2 - \mu_1$ and increases as the means become more nonlinear in the predicted label.

{

\subsection{Performance of RePPI under Different Parameters}
In this section, we provide additional simulation studies to compare the performance of RePPI and XY-only, PPI, and PPI++ under various choices of labeled data ratio, score fitting method, and target estimand. 

We will use the following data-generating processes.  Let
\[
    Z\sim \N(0,I_5),\qquad X=Z+\xi,\qquad A=Z+\nu,
\]
where $\xi,\nu\sim \N(0,I_5)$ are independent of $Z$. The auxiliary prediction is generated from
the nonlinear signal
\[
    s(A)=A_1+\sin(A_2)+A_3A_4 .
\]
The variable $A$ should be interpreted as auxiliary information that is accessible to the prediction
model but not directly used in the definition of the inferential target. The prediction
$\hat Y$ is therefore informative, but deliberately not calibrated for the target score.

We consider three nonlinear target specifications. For mean estimation, we set
\[
    \hat Y = s(A)+\eta,\qquad
    Y=1+\hat Y+\sin(\hat Y)+\varepsilon,
\]
where $\eta,\varepsilon\sim \N(0,1)$, and the target is $\theta^\star=E(Y)$. For linear regression,
we set
\[
    \hat Y=s(A)+\eta,\qquad
    Y=X^\top\gamma+\sin(\hat Y)+\hat Y\tanh(X_2)+\varepsilon,
    \qquad
    \gamma=(1,-1,1,0,0)^\top .
\]
The target is the first coordinate of the population least-squares projection of $Y$ on $X$.
For logistic regression, we set
\[
    \hat Y=\sigma\{s(A)\},
\]
and generate
\[
    Y\mid X,A
    \sim
    \mathrm{Bernoulli}\left[
    \sigma\left\{
        X^\top\beta+(2\hat Y-1)X_1+(2\hat Y-1)\tanh(X_2)
    \right\}
    \right],
    \qquad
    \beta=(1,-1,1,0,0)^\top ,
\]
where $\sigma(t)=1/(1+e^{-t})$. The target is the first coordinate of the population logistic
projection. In each case, the population truth used to evaluate coverage is approximated using an
independent Monte Carlo reference sample of size $200{,}000$.

We compare four estimators: the XY-only estimator, standard PPI, PPI++ with scalar power tuning, and RePPI. We report the average confidence interval length and empirical coverage. All confidence intervals are nominal $95\%$ intervals, and each configuration is repeated over
100 Monte Carlo replications. We will use three types of models to learn the imputed score for RePPI, including linear model (Linear), random forest (RF), and neural network (NN). 

We first study the nonlinear linear-regression target while fixing $N=10{,}000$ and varying
$n\in\{500,1000,1500,2000,2500,3000\}$. Figure~\ref{fig:sim-sample-size} reports the average
confidence interval length and coverage as a function of the label ratio $n/N$. RePPI consistently
has the shortest intervals over the full grid, and all methods have approximately the target level of coverage. The PPI++ method has no improvement compared to the XY-only method due to the nonlinear structure of the signal in $\hat Y$, which can be learned 

\begin{figure}[t]
    \centering
    \includegraphics[width=.95\linewidth]{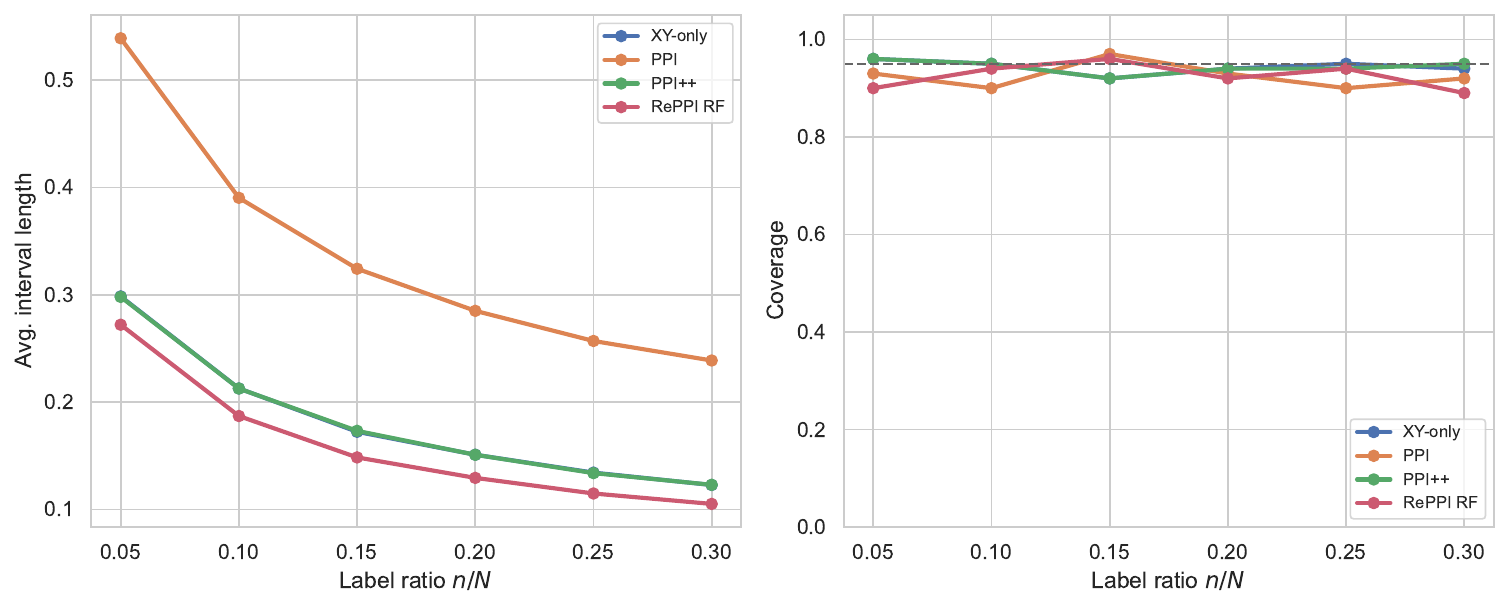}
    \caption{
    Average confidence interval length and coverage for the nonlinear
    linear-regression target with fixed $N=10{,}000$ and varying labeled sample size $n$.
    }
    \label{fig:sim-sample-size}
\end{figure}

We next fix $n=2000$ and $N=10{,}000$ under the same nonlinear linear-regression target and compare three learners for estimating the RePPI imputed score: Linear, RF, and NN. The results are shown in Figure~\ref{fig:sim-learners}. The linear recalibration gives only a small improvement over PPI++ because the optimal score is nonlinear in $(X,\hat Y)$. The flexible learners are more effective: the average interval length is $0.109$ for the neural-network recalibration and $0.128$ for the random-forest recalibration, compared with approximately $0.150$ for the labeled-only estimator, PPI++, and linear recalibration. The empirical coverages all remain valid for different models. 

\begin{figure}[t]
    \centering
    \includegraphics[width=.85\linewidth]{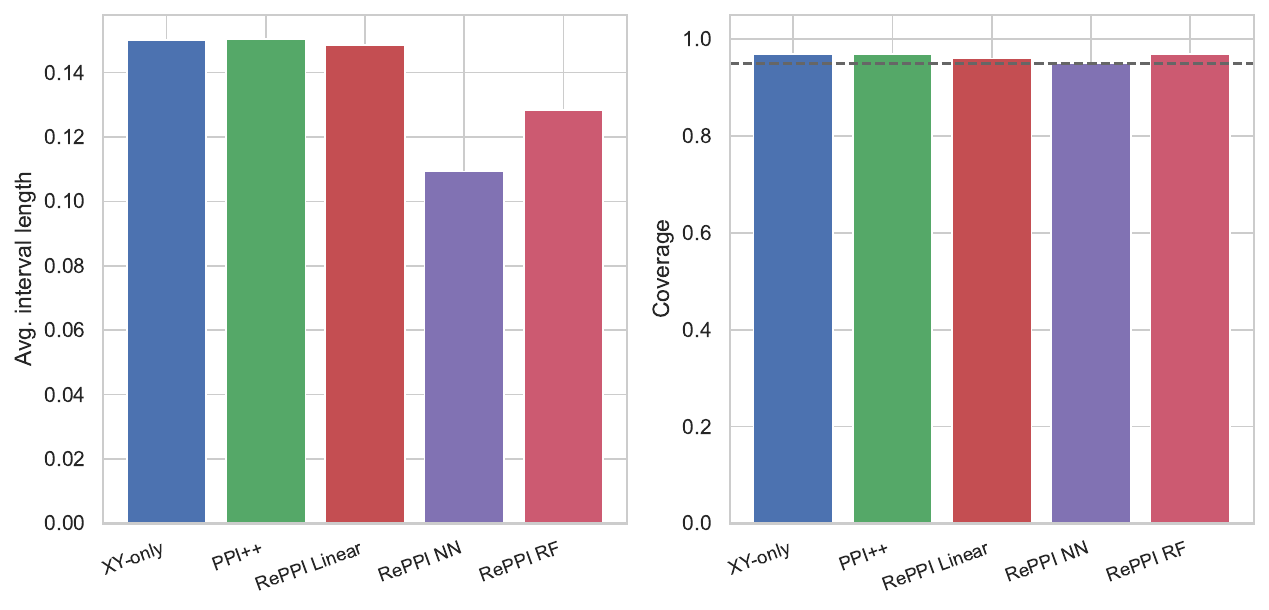}
    \caption{
    Comparison of different learners for estimating the RePPI imputed score
    under the nonlinear linear-regression target with $n=2000$ and $N=10{,}000$.}
    \label{fig:sim-learners}
\end{figure}

We then compare the methods across different estimands in mean estimation, linear regression, and logistic regression, again with $n=2000$ and $N=10{,}000$. Figure~\ref{fig:sim-estimands} reports average confidence
interval length and coverage. RePPI improves over the labeled-only estimator in all three target specifications, and preserves the target coverage level.

\begin{figure}[t]
    \centering
    \includegraphics[width=.95\linewidth]{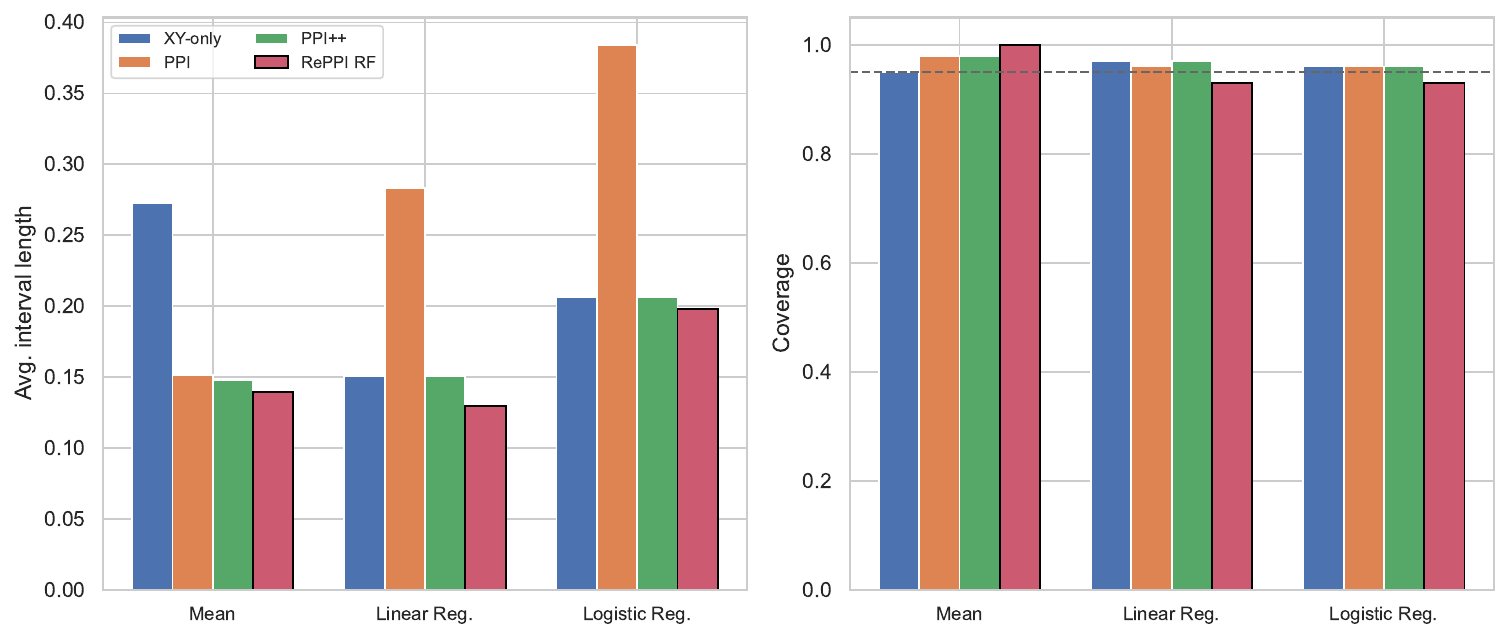}
    \caption{
    Comparison across target specifications. The left panel reports average
    confidence interval length, and the right panel reports empirical coverage.
    }
    \label{fig:sim-estimands}
\end{figure}
\subsection{Comparison with Other Methods}
\label{sec:other-comparison}
In this section, we conduct an additional simulation study for linear regression to compare RePPI with a broader set of recently proposed prediction-assisted inference methods, including:

\begin{itemize}
    \item The post-prediction adaptive inference (PSPA) method of \citet{miao2025assumption}.
    \item The prediction de-correlated inference (PDC) method of \citet{gan2024prediction}.
    \item The CC estimator of \cite{gronsbell2024another}.
    \item The safe and efficient estimator of \cite{xu2025unified}.
\end{itemize}
We consider a linear regression setting, where the target is defined as  $\theta^\star=\argmin_{\theta}\mathbb E\{(Y-X^\top\theta)^2\}$. We generate the data through the following model: let \(X=(1,\widetilde X^\top)^\top\), \(\widetilde X\sim \N(0,I_2)\), $W \sim \N(0,1)$, and
\[
    \hat Y = W, \quad Y=X^\top\theta^\star+W+W^2-1+\varepsilon,
\]
where the true parameter is set to be $\theta^\star = (1,1,1)^\top$. 
For each replication, we generate \(n=1000\) labeled observations \((X_i,Y_i,\hat Y_i)\) and \(N=5000\)
unlabeled observations \((X_i,\hat Y_i)\). The prediction surrogate \(\hat Y\) is informative but the optimal score $s^\star (X,\hat Y) = \E[X(X^\top \theta^\star  - Y)|X,\hat Y] = -X(\hat Y+\hat Y^2-1)$ is nonlinear; therefore, recalibration is necessary to achieve efficiency. We conduct 500 replications in total to report the average 95\% confidence interval and coverage for the regression coefficient of $\tilde X_1$ among all the methods mentioned above, and the XY-only and PPI++\cite{angelopoulos2023ppi++} methods as benchmarks. The experiment results are reported in Figure \ref{fig:method_comparison}.

\begin{figure}
    \centering
    \includegraphics[width=0.95\linewidth]{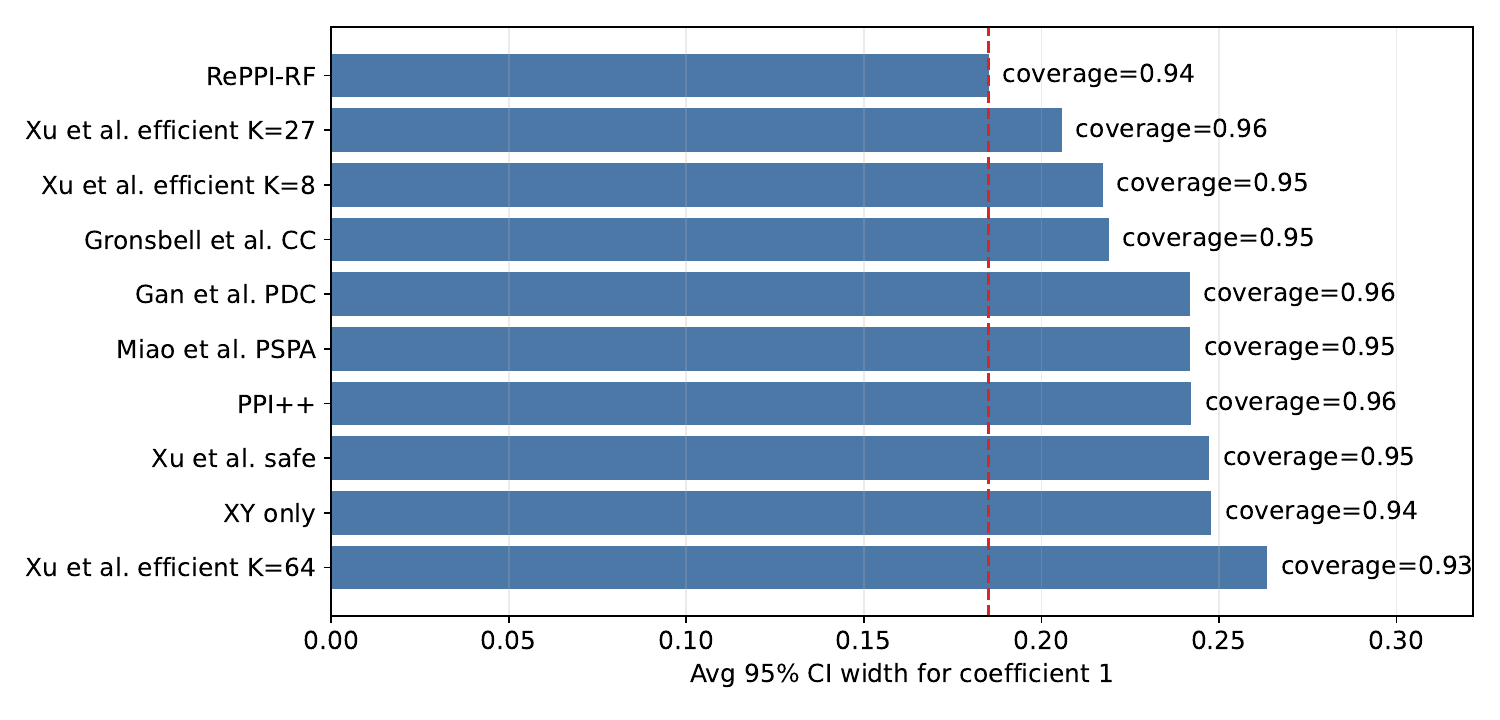}
    \caption{Comparison of our RePPI method with methods in \citet{miao2025assumption,gan2024prediction,gronsbell2024another,xu2025unified}. We use a random forest model to fit the optimal score for RePPI. For \citet{xu2025unified}, we choose the basis function to be the tensor product of natural cubic splines as the authors suggested, and we choose the number of basis $K=8,27,64$, which corresponds to $2,3,4$ basis on each dimension. We repeat the experiments for 500 times and report the average confidence interval length and coverage.}
    \label{fig:method_comparison}
\end{figure}
From Figure \ref{fig:method_comparison}, we find that our RePPI methods have the shortest confidence interval among all the methods. This is because \cite{gronsbell2024another, gan2024prediction, miao2025assumption} all use the original loss function, and therefore have suboptimal efficiency since the optimal imputed loss is nonlinear. \cite{xu2025unified} achieves better efficiency since it uses a basis expansion method to approximate the optimal score. However, we find that among $K=8,27,64$, the $K=27$ achieves the best performance, and $K=64$ is significantly worse than the XY-only estimator. This is because \cite{xu2025unified} explicitly requires that the minimum eigenvalue of the covariance matrix of basis functions be bounded. But empirically, we observe that in $K=64$, the condition number of the empirical covariance matrix for basis functions is very large, leading to poor numerical stability. Therefore, although theoretically \cite{xu2025unified} can achieve the semiparametric efficiency, their computational strategy may perform badly due to the computational instability and approximation efficiency of basis expansion. Our RePPI method doesn't restrict the model class used to estimate the optimal imputed score, and therefore can perform better in more general problems.

}
\section{Additional Applications}
\subsection{US Census Data}
We consider the relationship between age and wage rates, measured by the coefficient of age in the regression of log-income on age based on US Census
data, following \cite{angelopoulos2023prediction, angelopoulos2023ppi++}. We fit a prediction model of log-income with XGBoost using 14 other covariates, including education, marital status, citizenship, and race, among others. To demonstrate the effect of distribution shift, we restrict the training data to contain only those individuals who have a college degree or above, but seek inference on the whole population. The data used for inference contains 377,575 observations in total; we vary the fraction of labeled instances from 2\%-10\% and treat the remaining instances as unlabeled. We use linear regression to fit the imputed loss. The length and coverage of the computed confidence intervals are shown in Figure \ref{fig: us census}, and the number of labels required to achieve a given interval length is shown in Table \ref{tab: us census}. We find that all methods achieve the desired coverage, and RePPI consistently outperforms the other methods under different ratios $\frac{n}{N+n}$: it saves over 24\% of the labels required to achieve the same interval length as PPI and PPI++.

\begin{figure}[t]
    \centering
    \includegraphics[width=0.98\linewidth]{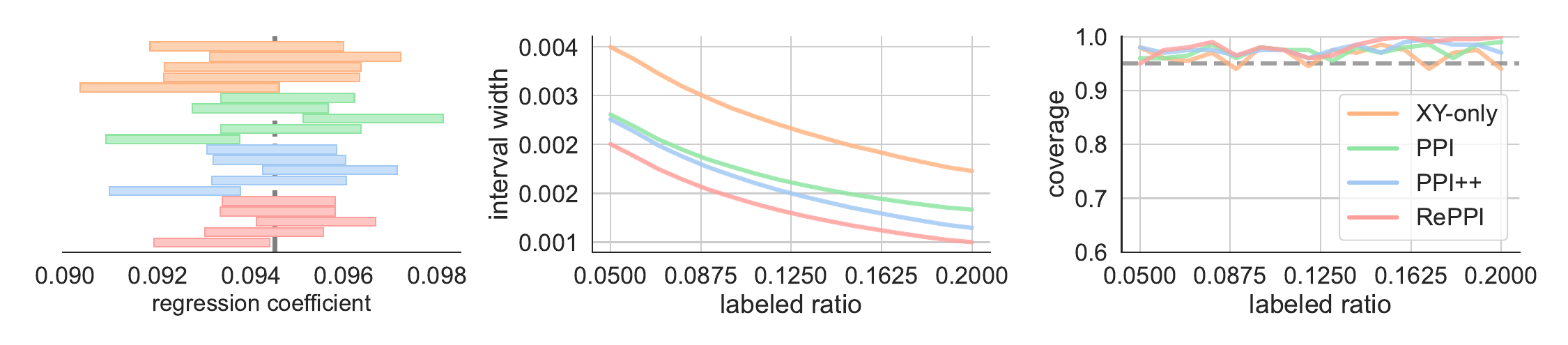}
       \caption{US Census experiments. The left panel shows a representative 95\% confidence interval, the middle panel shows average interval width, and the right panel shows empirical coverage. The horizontal axis represents the ratio of labeled data, $\frac{n}{N+n}$.}
    \label{fig: us census}
\end{figure}

\begin{table}[t]
    \centering
    \caption{Required number of labeled observations to achieve a given interval length in the US census experiment. The last column reports the
label reduction of RePPI relative to PPI++. Here ``NA'' means that the target 
interval length is not reached within the labeled-sample-size range considered.}
    \begin{tabular}{c||ccccc}
         \hline
 Interval Length & XY-only & PPI & PPI++ & RePPI & Reduced Samples (\%) \\
\hline
0.0019 & NA & 18696 & 15847 & 11918 & 24.79\% \\
0.0020 & NA & 16186 & 14062 & 10560 & 24.90\% \\
0.0021 & NA & 14165 & 12610 & 9405 & 25.41\% \\
0.0022 & 24883 & 12607 & 11344 & 8412 & 25.84\% \\
0.0023 & 22432 & 11262 & 10287 & 7594 & 26.18\% \\
\hline
    \end{tabular}
    \label{tab: us census}
\end{table}
\subsection{Politeness of Online Requests}
\label{sec: politeness}
We study the relationship between politeness and the linguistic device of hedging, following \citep{gligoric2024can}. The dataset \citep{danescu2013computational} contains texts from 5512 online requests posted on Stack Exchange and Wikipedia. For each request, the politeness score is obtained from an average score of 5 human evaluators on a scale of 1-25. The politeness prediction is obtained from OpenAI's GPT-4o mini model by prompting it to rate the politeness of the given text on a scale of 1-25 as well. The parameter of interest is the regression coefficient obtained by regressing the politeness score on the binary indicator of hedging in the request.  We again use linear regression to fit the imputed loss. The experimental results are shown in Figure \ref{fig: politeness} and Table \ref{fig: politeness}. We again observe a uniform improvement of RePPI over the other methods. Intuitively, this is because the language model's predictions are not well-calibrated to match the correct distribution of human scores. Therefore, the recalibration improves the estimation accuracy.
\begin{figure}[b]
    \centering
    \includegraphics[width=0.98\linewidth]{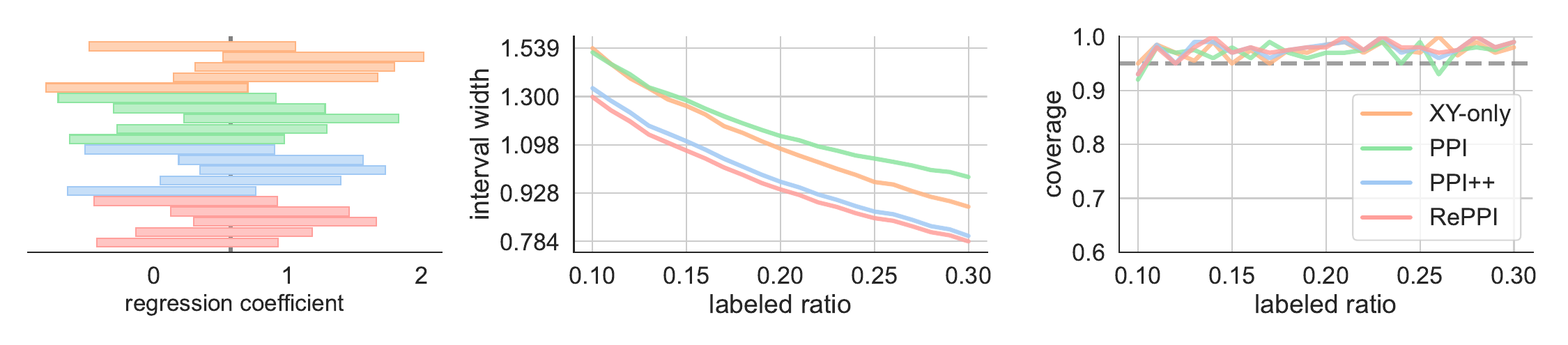}
    \caption{Politeness experiments. The left panel shows a representative 95\% confidence interval, the middle panel shows average interval width, and the right panel shows empirical coverage. The horizontal axis represents the ratio of labeled data, $\frac{n}{N+n}$.}
    \label{fig: politeness}
\end{figure}
\begin{table}[t]
\caption{Required number of labeled observations to achieve a given interval length in the politeness experiment. The last column reports the
label reduction of RePPI relative to PPI++. Here ``NA'' means that the target 
interval length is not reached within the labeled-sample-size range considered.}
    \centering
    \begin{tabular}{c||ccccc}
         \hline
 Interval Length & XY-only & PPI & PPI++ & RePPI & Reduced Samples (\%) \\
\hline
0.8 & NA & NA & 1650 & 1599 & 3.09\% \\
0.9 & 1604 & NA & 1289 & 1211 & 6.08\% \\
1.0 & 1299 & 1591 & 1026 & 960 & 6.42\% \\
1.1 & 1072 & 1193 & 843 & 787 & 6.60\% \\
1.2 & 897 & 954 & 697 & 656 & 5.88\% \\
\hline
    \end{tabular}
    \label{tab: politeness}
\end{table}

\section{Extension}
In this section, we discuss several immediate extensions of our framework.
\subsection{Predicted Covariates}
While we focus on predicted outcomes, it is easy to adapt RePPI to additionally enable prediction models for covariates. Let $X = (X_1, X_2)$, and suppose that the subset of covariates $X_2$ is missing in the unlabeled data. Let $\hat{X}_2 
= f(X_1, W)$ where $f$ is a prediction model for $X_2$. Then, we can use an \rectifier that takes predictions of both $X_2$ and $Y$ as input, $g_\theta(X_1, \hat{X}_2, \hat{Y})$. Following Theorem \ref{thm: efficient PPI}, the optimal \rectifier satisfies 
\[\nabla g_{\theta^\star}(X_1, \hat{X}_2, \hat{Y}) = \E[\nabla \ell_{\theta^\star}(X, Y)\mid X_1, \hat{X}_2, \hat{Y}].\]
Similarly as before, we can adapt RePPI by choosing $g_\theta(X_1, \hat{X}_2, \hat{Y}) = \theta^\top \hat{M}\hat{s}(X_1, \hat{X}_2, \hat{Y})$, where $\hat{s}(X_1, \hat{X}_2, \hat{Y})$ is an estimate of $\E[\nabla \ell_{\theta^\star}(X, Y)\mid X_1, \hat{X}_2, \hat{Y}]$.

More generally, there may be multiple missingness patterns. For instance, there could be four datasets $\{\mathcal{D}_{d_{X_1}d_Y}: d_Y, d_{X_2}\in \{0, 1\}\}$ where $Y$ (resp. $X_1$) is missing iff $d_Y = 0$ (resp. $d_{X_2} = 0$). We can then extend the class of PPI estimators to include all estimators of the form 
\begin{align*}
\argmin_{\theta} &\frac{1}{|\mathcal{D}_{11}|}\sum_{i\in \mathcal{D}_{11}}\ell_{\theta}(X_i, Y_i) -\lb \frac{1}{|\mathcal{D}_{11}|}\sum_{i\in \mathcal{D}_{11}}g^{(10)}_{\theta}(X_i, \hat Y_i) - \frac{1}{|\mathcal{D}_{10}|}\sum_{i\in \mathcal{D}_{10}}g^{(10)}_{\theta}(X_i, \hat Y_i)\rb\\
& -\lb \frac{1}{|\mathcal{D}_{11}|}\sum_{i\in \mathcal{D}_{11}}g^{(01)}_{\theta}(X_{i1}, \hat{X}_{i2}, Y_i) - \frac{1}{|\mathcal{D}_{01}|}\sum_{i\in \mathcal{D}_{01}}g^{(01)}_{\theta}(X_{i1}, \hat{X}_{i2}, Y_i)\rb\\
& -\lb \frac{1}{|\mathcal{D}_{11}|}\sum_{i\in \mathcal{D}_{11}}g^{(00)}_{\theta}(X_{i1}, \hat{X}_{i2}, \hat Y_i) - \frac{1}{|\mathcal{D}_{00}|}\sum_{i\in \mathcal{D}_{00}}g^{(00)}_{\theta}(X_{i1}, \hat{X}_{i2}, \hat Y_i)\rb.
\end{align*}
Note that the expectation of the above loss function is $\E[\ell_\theta(X, Y)]$. Thus, under regularity conditions, the estimator is consistent and asymptotically normal. We leave the investigation of optimal $(g_\theta^{(10)}, g_\theta^{(01)}, g_\theta^{(00)})$ for future research.

\subsection{Distribution Shift}
\label{sec: distribution shift}
In this paper, we mainly focus on the setting where the labeled and unlabeled data are i.i.d. sampled from the same distribution. Here, we consider generalization to the case where the labeled data $\{(X_i, \hat Y_i)\}_{i=1}^n$ comes from a distribution $\mathbb{P}$ and the unlabeled data $\{(X_i, \hat Y_i)\}_{i=n+1}^{n+N}$ comes from a distribution $\mathbb{Q}$, where $\mathbb{P}_{X,\hat Y}\neq \mathbb{Q}_{X,\hat Y}$. Here $\mathbb{P}_{X,\hat Y}, \mathbb{Q}_{X,\hat Y}$ are both marginal distributions on $(X, \hat Y)$, since the label $Y$ is not observed on the unlabeled dataset, and we denote the conditional distribution of $Y$ given $(X, \hat Y)$ as $\mathbb{P}_{Y|X,\hat Y}$. In particular, we assume that the target estimand is defined on the same distribution of the labeled dataset, i.e., 
$$
\theta^\star = \argmin_{\theta\in \Theta}\E_{\mathbb{P}_{X,\hat Y}\times\mathbb{P}_{Y|X,\hat Y}}[\ell_{\theta}(X, Y)].
$$
Under this setting, the distribution shift between the two datasets can be handled by reweighting the unlabeled data. Assume the Radon-Nikodym derivative $w(x,\hat y) = \frac{d \mathbb{P}_{X,\hat Y}}{d \mathbb{Q}_{X,\hat Y}}(x,\hat y)$ is known, then  we can modify the PPI estimator \eqref{eqn: PPI} as 
\begin{equation*}
    \begin{aligned}
\hat{\theta}_g^\ppi = \argmin_{\theta} \frac{1}{n}\sum_{i=1}^{n}\ell_{\theta}(X_i, Y_i) -\lb\frac{1}{n}\sum_{i=1}^{n}g_{\theta}(X_i, \hat Y_i) - \frac{1}{N}\sum_{i=n+1}^{n+N}w(X_i,\hat Y_i)g_{\theta}(X_i, \hat Y_i)\rb,  
\end{aligned}
\end{equation*}
and the results in Theorem \ref{thm: efficient PPI} and Theorem \ref{thm:main} immediately applies under the same proof.

{

In practice, the Radon--Nikodym derivative \(w(x,\hat y)\) is typically unknown and must be estimated from the labeled and unlabeled data. A natural implementation replaces \(w\) in the weighted unlabeled term by a plug-in estimate \(\widehat w\), and fits it with an additional fold in the cross-fitting procedure to keep independence. The difference in the objective function of using $w$ and $\hat w$ is 
$$
\frac{1}{N}\sum_{i=n+1}^{n+N}(w(X_i,\hat Y_i) - \hat w(X_i,\hat Y_i))g_{\theta}(X_i, \hat Y_i)
$$
Hence, the resulting estimator has the same first-order expansion as the known-\(w\) estimator if
\[
    \sqrt n
    \left\|
        \E_Q\!\left[
            \{\widehat w(X,\hat Y)-w(X,\hat Y)\}
            \nabla g_{\theta^\star}(X,\hat Y)
        \right]
    \right\|
    =
    o_p(1), \quad \E_Q[\|\hat w-w\|^2] = o_p(1)
\]
A simple sufficient condition is
\[
    \|\widehat w-w\|_{L_2(Q)}=o_p(n^{-1/2})
\]
when \(\nabla g_{\theta^\star}\) is square-integrable and uniformly bounded in \(L_2(Q)\). However, this is generally impossible with usual parametric and nonparametric models. If the above moment condition is not satisfied, then the plug-in density-ratio error may lead to additional first-order bias and result in variance inflation. Since quantifying the error arising from plug-in density would depend on the specific model choice of model for $\hat w$, we leave a more thorough analysis for future research.  
\subsection{Finite-Sample Variance Comparison for Mean Estimation}

For mean estimation, the efficiency gain from recalibration can be understood exactly through a finite-sample control-variate calculation. Let \(\theta^\star=\mathbb E[Y]\), and let \(h(\hat Y)\) be a recalibrated surrogate. Consider the class of PPI estimators (which includes XY-only, PPI, PPI++, and RePPI)
\[
    \widehat\theta(a,h)
    =
    \bar Y_L
    -
    a\{\bar h_L-\bar h_U\},
\]
where
\[
    \bar Y_L=\frac1n\sum_{i=1}^nY_i,\qquad
    \bar h_L=\frac1n\sum_{i=1}^nh(\hat Y_i),\qquad
    \bar h_U=\frac1N\sum_{i=n+1}^{n+N}h(\hat Y_i).
\]
Writing
\[
    \sigma_Y^2=\operatorname{Var}(Y),\qquad
    \gamma_h=\operatorname{Cov}\{Y,h(\hat Y)\},\qquad
    \Gamma_h=\operatorname{Var}\{h(\hat Y)\},
\]
we have
\[
    \operatorname{Var}\{\widehat\theta(a,h)\}
    =
    \frac{\sigma_Y^2}{n}
    -
    \frac{2a\gamma_h}{n}
    +
    a^2\Gamma_h\left(\frac1n+\frac1N\right).
\]
Similar to the optimal tuning \eqref{eq:Mhat} in the general case, for a fixed \(h\), the optimal coefficient is
\[
    a_h^\star
    =
    \frac{N}{n+N}\frac{\gamma_h}{\Gamma_h},
\]
and the optimized variance is
\[
    \operatorname{Var}\{\widehat\theta(a_h^\star,h)\}
    =
    \frac{\sigma_Y^2}{n}
    \left(
        1-\frac{N}{n+N}R_h^2
    \right),
    \qquad
    R_h^2=
    \frac{\gamma_h^2}{\sigma_Y^2\Gamma_h}.
\]
The XY-only estimator corresponds to \(a=0\), with variance \(\sigma_Y^2/n\). Standard PPI corresponds to \(h(\hat Y)=\hat Y\) and \(a=1\). PPI++ keeps the raw prediction \(h(\hat Y)=\hat Y\) but optimizes the coefficient \(a\), giving
\[
    \operatorname{Var}(\widehat\theta^{\mathrm{PPI++}})
    =
    \frac{\sigma_Y^2}{n}
    \left(
        1-\frac{N}{n+N}R_{\hat Y}^2
    \right),
    \qquad
    R_{\hat Y}^2=
    \frac{\operatorname{Cov}(Y,\hat Y)^2}
    {\operatorname{Var}(Y)\operatorname{Var}(\hat Y)}.
\]
Furthermore, RePPI has the flexibility to first recalibrate the prediction and then applies the same optimal tuning, which corresponds to optimize both $h$ and $a$. With the choice of the oracle function \(h(\hat Y)=\mathbb E[Y\mid \hat Y]\) (same as the oracle score $s^\star$), the variance is optimized and the optimal value is
\[
    \operatorname{Var}(\widehat\theta^\REPPI)
    =
    \frac{\sigma_Y^2}{n}
    \left(
        1-\frac{N}{n+N}R_h^2
    \right) = \frac{\sigma_Y^2}{n}
    \left(
        1-\frac{N}{n+N}\frac{\Var(\E[Y|\hat Y])}{\sigma_Y^2}
    \right).
\]
Therefore, RePPI improves over PPI++ whenever the optimal recalibration $h(\hat Y) = \E[Y|\hat Y]$ is not a simple linear function $aY$. In this finite-sample mean-estimation setting,
\[
    \operatorname{Var}(\widehat\theta^\REPPI)
    \le
    \operatorname{Var}(\widehat\theta^\ppiplus)
    \le
    \min\{
        \operatorname{Var}(\widehat\theta^\ppi),
        \operatorname{Var}(\widehat\theta^\cls)
    \}.
\]
And the gain of RePPI over PPI++ is
\[
    \frac{\sigma_Y^2}{n}\frac{N}{n+N}
    \left(
        \frac{\Var(\E[Y|\hat Y])}{\sigma_Y^2}-\frac{\operatorname{Cov}(Y,\hat Y)^2}
    {\operatorname{Var}(Y)\operatorname{Var}(\hat Y)}
    \right),
\]
which corresponds to the difference between the optimal projection of $Y$ on $\hat{Y}$ in the class of linear functions and all measurable functions.

}
\end{document}